\documentclass[twoside,11pt]{article}
\usepackage{blindtext}
\usepackage{jmlr2e}
\usepackage{natbib}
\usepackage{latexsym}
\usepackage{color}
\usepackage{bbm}
\usepackage{amssymb, graphicx, amsmath}
\usepackage{wrapfig}
\usepackage{footnote}
\usepackage{subfigure}
\usepackage{graphicx}
\usepackage{url}
\usepackage{CJK,graphicx}
\usepackage{algorithm}
\usepackage{multirow}
\usepackage{subfigure}
\usepackage{float}
\usepackage{algorithm, algorithmic}
\usepackage{appendix}
\usepackage{pifont}
\usepackage{amsfonts,amsmath,amsfonts,amssymb}
\usepackage{graphics,epsfig,graphicx,subfigure}
\usepackage[table]{xcolor}
\usepackage{helvet,epsfig,graphicx,subfigure}
\usepackage{textcomp}
\usepackage{color}
\usepackage{makecell}
\usepackage[normalem]{ulem}
\usepackage{tikz}
\usepackage{amsmath}
\usepackage{booktabs,array,xcolor,colortbl,multirow,epsfig,graphicx,subfigure,amsmath,amsfonts,amssymb,bm}
\usepackage[most]{tcolorbox}
\usepackage{soul}

\allowdisplaybreaks

\newtheorem{theo}{Theorem}[section]
\newtheorem{assumption}[theorem]{Assumption}
\newtheorem{prop}{Proposition}
\numberwithin{equation}{section}

\definecolor{bluegray}{cmyk}{1,0,0,0}

\newcommand{\nn}{\nonumber}

\usepackage{lastpage}
\jmlrheading{24}{2023}{1-\pageref{LastPage}}{7/20; Revised 3/23}{6/23}{20-700}{Wenhao Li, Bo Jin, Xiangfeng Wang, Junchi Yan, and Hongyuan Zha}

\ShortHeadings{F2A2: Flexible Fully-decentralized Approximate Actor-critic}{Li, Jin, Wang, Yan, and Zha}
\firstpageno{1}

\begin{document}

\title{F2A2: Flexible Fully-decentralized Approximate Actor-critic for Cooperative Multi-agent Reinforcement Learning}

\author{\name Wenhao~Li \email liwenhao@cuhk.edu.cn \\
      \addr School of Data Science\\
      The Chinese University of Hong Kong, Shenzhen\\
      Shenzhen Institute of Artificial Intelligence and Robotics for Society\\
      Shenzhen 518172, China
      \AND
      \name Bo~Jin \email bjin@tongji.edu.cn \\
      \addr Corresponding author\\
      School of Software Engineering\\
      Shanghai Research Institute for Intelligent Autonomous Systems\\
      Tongji University\\
      Shanghai 201804, China
      \AND
      \name Xiangfeng~Wang \email xfwang@cs.ecnu.edu.cn \\
      \addr Corresponding author\\
      School of Computer Science and Technology\\
      Key Laboratory of Mathematics and Engineering Applications, Ministry of Education\\
      East China Normal University\\
      Shanghai 200062, China
      \AND
      \name Junchi~Yan \email yanjunchi@sjtu.edu.cn \\
      \addr Department of Computer Science and Engineering \\
      Key Laboratory of Artificial Intelligence, Ministry of Education\\
      Shanghai Jiao Tong University \\
      Shanghai 200240, China 
      \AND
      \name Hongyuan~Zha \email zhahy@cuhk.edu.cn \\
      \addr School of Data Science\\
      The Chinese University of Hong Kong, Shenzhen\\
      Shenzhen Institute of Artificial Intelligence and Robotics for Society\\
      Shenzhen 518172, China}

\editor{Ambuj Tewari}

\maketitle

\begin{abstract}
Traditional centralized multi-agent reinforcement learning (MARL) algorithms are sometimes unpractical in complicated applications due to non-interactivity between agents, the curse of dimensionality, and computation complexity. 
Hence, several decentralized MARL algorithms are motivated. 
However, existing decentralized methods only handle the fully cooperative setting where massive information needs to be transmitted in training.
The block coordinate gradient descent scheme they used for successive independent actor and critic steps can simplify the calculation, but it causes serious bias.
This paper proposes a flexible fully decentralized actor-critic MARL framework, which can combine most of the actor-critic methods and handle large-scale general cooperative multi-agent settings. 
A primal-dual hybrid gradient descent type algorithm framework is designed to learn individual agents separately for decentralization.
From the perspective of each agent, policy improvement and value evaluation are jointly optimized, which can stabilize multi-agent policy learning. 
Furthermore, the proposed framework can achieve scalability and stability for the large-scale environment.
This framework also reduces information transmission by the parameter sharing mechanism and novel modeling-other-agents methods based on theory-of-mind and online supervised learning.
Sufficient experiments in cooperative Multi-agent Particle Environment and StarCraft II show that the proposed decentralized MARL instantiation algorithms perform competitively against conventional centralized and decentralized methods.
\end{abstract}


\section{Introduction} \label{sec:intro}

Multi-agent reinforcement learning (MARL,~\citet{Zhang2019MultiAgentRL}) has shown remarkable performance in interactive and complicated cooperative multi-agent environments, e.g. multi-robot controlling \citep{matignon2012coordinated} and multi-player games~\citep{peng2017multiagent}. 
MARL algorithms generally model a cooperative multi-agent learning system as a Markov game~\citep{littman1994markov} (or a stochastic game, \citep{owen1982gametheory}), where the joint actions of multiple agents influence a shared environment. 
In particular, each agent can access the full observation of the environment and takes action according to its current policy. 
These actions together determine the successive states of the environment~\citep{lowe2017multi, foerster2018counterfactual, rashid2018qmix}.

However, 
the global assumption that each agent can fully observe the environment is usually difficult to satisfy in many practical applications, such as intelligent connected vehicle \citep{adler2002cooperative}. 
Hence, a more reasonable solution is to model the problem as a more general formulation, i.e., cooperative partially observable stochastic game (POSG) \citep{hansen2004dynamic}.
Various MARL methods have been proposed, including value-based \citep{jiang2018graph, rashid2018qmix}, actor-critic-based \citep{foerster2018counterfactual, wei2018multiagent, iqbal2019actor}.
However, these methods are designed to solve a fully cooperative POSG with another global assumption, i.e., all agents share a global cost function. 
Therefore, some works try to solve the more general cooperative POSG problem,
where the cost functions of the agents might correspond to different tasks, and are only known to the corresponding agent.
The collective goal of the agents is to minimize the globally summed return~\citep{jiang2018learning, yang2018mean, li2019robust}.
This paper is focused on this general cooperative setting.
	
Existing MARL algorithms for cooperative POSG mostly follow two frameworks (Figure~\ref{fig:ctde-dtde}): Centralized Training Decentralized Execution (CTDE) \citep{oliehoek2008optimal} and Decentralized Training Decentralized Execution (DTDE).
CTDE assumes that a powerful central controller can receive and process all agents' information.
The left of Figure~\ref{fig:ctde-dtde} shows the CTDE framework of actor-critic-based MARL methods, where centralized training explicitly takes into account the observations, policies, and costs of all agents in the learning phase, thereby effectively solving the non-stationary environment problem in multi-agent reinforcement learning. 
However, CTDE suffers from some limitations. 
First, as the number of agents increases, the amount of information for the centralized controller to process will increase exponentially, eventually leading to the \textit{curse of dimensionality}. 
This will bring heavy space-time overhead to the entire system. 
Second, the centralized assumption for training is unrealistic in many real-world scenarios.
Besides, the presence of a centralized controller also degrades the system's capability to resist malicious attacks.

\begin{figure*}[tb!]
    \centering
    \includegraphics[width=\textwidth]{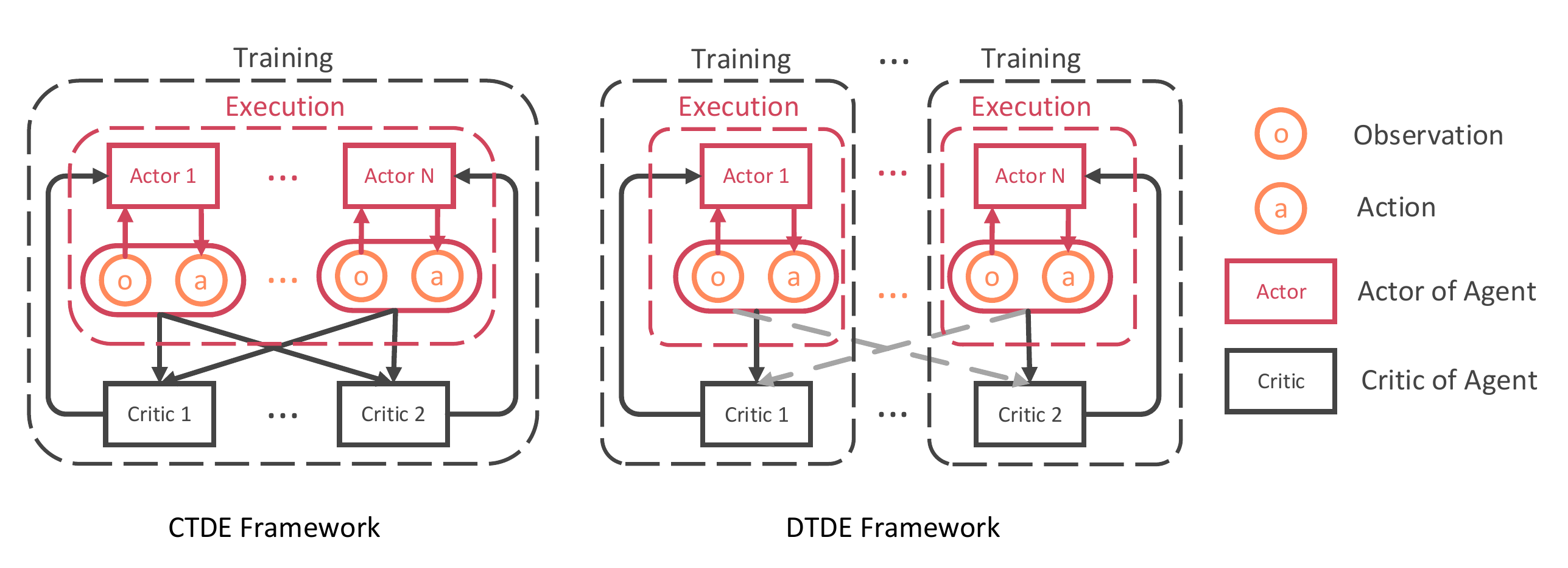}
    \vspace{-10pt}
    \caption{Centralized training decentralized execution (CTDE) and decentralized training decentralized execution (DTDE) actor-critic frameworks. The gray dotted arrow indicates that the relevant information of the agents is no longer scheduled by a centralized controller but is passing between agents by communication.}
    \label{fig:ctde-dtde}
\end{figure*}

Many recent works try to follow the DTDE framework (Figure~\ref{fig:ctde-dtde}) to decentralize cooperative MARL \citep{zhang2018fully,doan2019finite,doan2019convergence,suttle2019multi}.
Unfortunately, existing decentralized cooperative MARL methods either only focus on the policy evaluation stage in fully observable small-scale multi-agent problems or only could handle the fully cooperative POSG setting where massive information needs to be transmitted in decentralized training. 
Therefore, our motivation is to design a more general DTDE framework for general cooperative POSG problems.
To achieve a fully decentralized MARL algorithm, the main challenge need to be tackled is: 
How to effectively utilize partial observation of each agent to make the global decision in a fully decentralized scheme?
In other words, it needs to drive the performance of decentralized MARL equal or close to the centralized one.

In this paper, we propose a Flexible Fully-decentralized Approximate Actor-critic (F2A2) algorithm for DTDE cooperative MARL under a novel and general joint actor-critic framework in Figure~\ref{fig:decent}.
A fully decentralized mechanism is designed based on consensus constraints and primal-dual optimization,
whose benefit is that it can introduce a parameter sharing mechanism to increase the efficiency and adaptability for different settings.
Moreover, to reduce the effect of the information loss caused by decentralized settings, a novel modeling-other-agents (MOA) technique is adopted based on theory-of-mind (TOM)~\citep{rabinowitz2018machine} and online supervised learning, which enables the agent to estimate the information of other agents while making decisions based on local information, and improves the robustness and performance of the F2A2.
Several decentralized versions of typical MARL algorithms are devised in the proposed framework, i.e., F2A2-COMA, F2A2-DDPG, F2A2-TD3, and F2A2-SAC.
Extensive experiments on Cooperative Multi-agent Particle Environment and StarCraft II show that the proposed fully decentralized algorithms can obtain competitive performance against conventional centralized and decentralized methods.
Overall, the main contributions of this paper are:
\begin{itemize}
    \item \textit{Fully decentralized framework}. We reformulate actor-critic for general cooperative POSG problems in an additive joint form. Its separable characteristics lead to a novel fully decentralized actor-critic framework cooperating with separable primal-dual optimization.  This is the first fully decentralized MARL to our knowledge.
    \item \textit{Flexibility}. The proposed novel actor-critic framework is compatible with various actor-critic algorithms, and the general decentralizing mechanism can transform them into fully decentralized versions. Besides, the proposed framework has a flexible parameter-sharing and regularization mechanism, which makes F2A2 suitable for different kinds of settings, including small- and large-scale cooperative scenarios, on- and off-policy training schemes.
    \item \textit{Performance}. A sufficient comparison is made between existing centralized and decentralized algorithms with F2A2. The proposed decentralized solutions have achieved remarkable performance with a more general setting even compared with its centralized version. 
\end{itemize}

\begin{figure*}[tb!]
    \centering
    \includegraphics[width=0.8\textwidth]{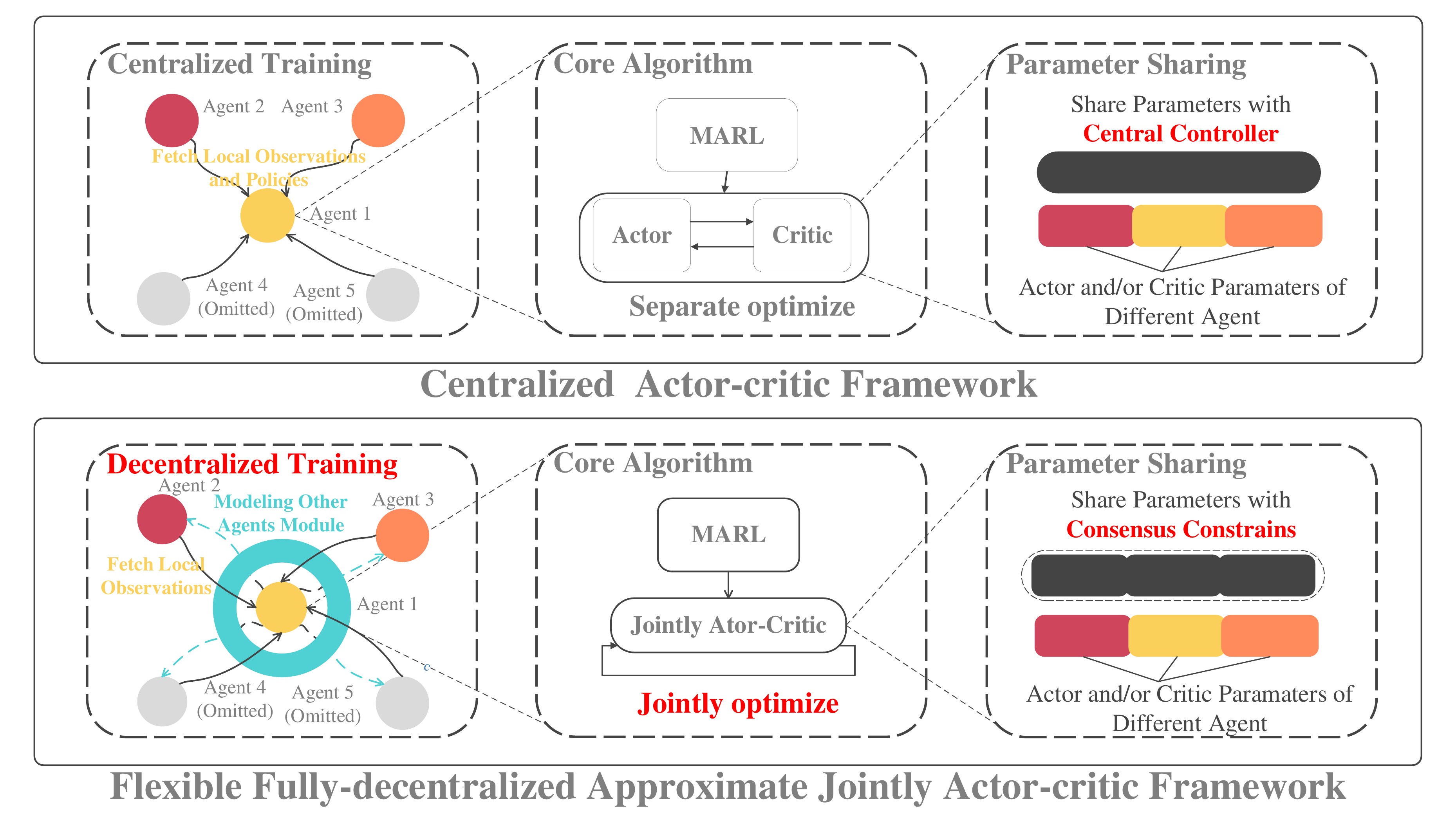}
    \vspace{-10pt}
    \caption{The F2A2 architecture. 
    Compared with the upper traditional centralized training methods, F2A2 solves general cooperative POSG by a novel additive joint policy optimization objective and decentralizes via a separable primal-dual algorithm. 
    The flexibility and superiority of the joint actor-critic framework allow it to combine any on/off-policy actor-critic algorithms with the advantages of the trust region. F2A2 contains a novel agent modeling scheme in decentralized training to improve the communication efficiency. Parameter sharing mechanism and consensus constraints enhance generality and scalability.}
    \label{fig:decent}
\end{figure*}

The following is the roadmap of this paper.
Section \ref{sec:related} provides a brief but complete introduction to related works in centralized MARL and decentralized MARL and provide the background material on cooperative POSG and actor-critic algorithm in Section \ref{sec:background}. 
Section \ref{sec:model} describes the proposed model. 
The fully decentralized actor-critic MARL algorithms and related analysis from various aspects are presented in Section \ref{sec:algorithm}.
Extensive experiments are presented in Section \ref{sec:experiment}, and we conclude this paper in Section \ref{sec:conclusion}.

\section{Related Work}\label{sec:related}
The cooperative multi-agent problem exists widely in real world, such as multi-robot control~\citep{matignon2012coordinated}, multi-player games~\citep{peng2017multiagent} and etc.
In this section, existing cooperative multi-agent reinforcement learning methods is divided into \textit{centralized} and \textit{decentralized} methods based on their training mechanism.
In this paper, the decentralized algorithm means it satisfies decentralization in both \textit{training} and \textit{execution} stages.

\subsection{Centralized Cooperative MARL.}
Most of the recent centralized cooperative MARL algorithms are follow the centralized training decentralized execution (CTDE) mechanism.
The rise of the CTDE framework mainly has two-fold reasons.
At early stages, the cooperative multi-agent problem is often modeled as a sizeable single-agent problem via joint action learning (JAL)~\citep{claus1998dynamics}, by learning a centralized policy based on global state or joint observation and then executing the joint action.
It has to assume that the agents can access the global state, or there is a communication channel 
to integrate the individual information of all agents, no matter for training or execution. 
Even if the above problems can be solved by technical means, when the number of agents grows, the joint action space's size will increase exponentially, leading to the centralized policy learning infeasible.
By contrast, independent learning (IL)~\citep{tan1993multi} can learn decentralized policies, but it results in nonstationarity as each agent treats the other agents as part of its environment.
To solve the infeasible policy learning and the non-stationary problem, \citet{oliehoek2008optimal} introduces the CTDE mechanism.
Specifically, in the training stage, the agents can freely share individual observations, parameters, and/or gradients, similar to JAL.
However, in the execution stage, the policies are decentralized carried out with local observations, just similar in IL.

The algorithms that follow the CTDE mechanism can be divided into two categories: value-based methods and actor-critic-based methods.
For value-based methods, some works \citep{rashid2018qmix,son2019qtran} obtain a decomposable joint Q-value function through centralized training and only use the local Q-value function for each agent in the execution stage.
The graph convolutional method~\citep{jiang2018graph} introduces the graph neural networks (GNN)~\citep{scarselli2008graph} to \textcolor{black}{factorize the joint Q-value function without decomposable assumption}.
\textcolor{black}{\citet{Yang2018MeanFM} and \citet{Wang2020BreakingTC} introduce the mean-field approximation to mitigate the curse of dimentionality of joint Q-value function.}
The actor-critic-based algorithms are more suitable for the CTDE training mechanism than the value-based methods because of their unique structure composed of the actor and the critic.
Actor-critic-based multi-agent reinforcement learning algorithms generally train a central-critic (based on individual observations and policies of all agents) centrally and use the actor, which is only based on the local observation during decentralized execution for each agent. 
COMA \citep{foerster2018counterfactual} learns a shared centralized counterfactual baseline for all agents, which addresses the credit assignment problem in cooperative multi-agent learning.
\citet{wei2018multiagent} extends the Soft Q-Learning~\citep{haarnoja2017reinforcement} to multi-agent to solve relative over-generalization problem.
A line of works \citep{lowe2017multi, li2019robust, reddy2019risk} extend the DDPG\citep{lillicrap2015continuous} algorithm to multi-agent scenarios. 
Unlike previous methods, MADDPG~\citep{lowe2017multi} centrally trains a central critic for each agent. 
M3DDPG~\citep{li2019robust} introduces the minimax method to maximize the expected returns when the other agents have the lowest returns.
RC-MADDPG~\citep{reddy2019risk}, in addition to maximizing the expected return of the agent, additionally introduces the variation of return (VOR) as an optimization goal, making the policy performance of the agent more stable.
\textcolor{black}{
\citet{Ryu2018MultiAgentAW} introduced an additional generative cooperative policy network, which is used to generate the more diverse samples via maximizing the expected return of other agents, to encourage the exploration based on \citet{lowe2017multi}.}
MAAC~\citep{iqbal2019actor}, which is also based on MADDPG, introduces the attention mechanism~\citep{bahdanau2014neural} to model centralized non-shared critic in actor-critic framework, enabling the agent to self-identify information quality.

In recent years there has been another main line of work to solve the cooperative POSG problem under CTDE framework, i.e. \textit{learning to communicate}~\citep{foerster2016learning, sukhbaatar2016learning, peng2017multiagent, jiang2018learning, sheng2020learning} and \textcolor{black}{\citep{Chu2020MultiagentRL}}.
The cooperation between agents is accomplished by passing information to each other, and the information is generated by a shared message generator which is training centrally.


\subsection{Decentralized Cooperative MARL.}
For the current mainstream CTDE framework, the existence of an unavoidable central controller during training makes the framework have many performance limitations, such as single points of failure, high communication requirements, massive computing burden, and limited flexibility and scalability.
In order to solve the problems caused by the CTDE framework, there are also some works on fully decentralized frameworks in recent years.
From the perspective of optimization goals, these methods are divided into two categories: \textit{policy evaluation} and \textit{optimal policy learning}.
For methods that belong to policy evaluation, their purpose is to learn the optimal value function corresponding to a fixed multi-agent policy (the fixed policy does not have to be optimal). 
For the methods that belong to the optimal policy learning, the goal is the same as all the above-mentioned centralized methods, to learn an optimal control policy for all agents. 
Therefore, \textit{the policy evaluation method can be regarded as a component of the optimal policy learning method.}

For policy evaluation methods, the objective of all agents is to jointly minimize the mean square projected Bellman error (MSPBE).
\citet{macua2015distributed} and \citet{stankovic2016multi} are fully decentralized multi-agent extensions of gradient temporal difference (GTD-2) \citep{sutton2009fast} and linear temporal-difference with gradient correction (TDC) \citep{sutton2009fast}.
\citet{lee2018primal} also develops a fully decentralized multi-agent extension of GTD-2 and using the ordinary differential equation (ODE) method to establish the asymptotic convergence.
\citet{wai2018multi} proposes a double averaging scheme that combines the dynamic consensus~\citep{qu2017harnessing} and the SAG algorithm~\citep{schmidt2017minimizing} to solve the distributed saddle-point version of MSPBE and achieving the linear convergence rate.
More recently, standard TD learning~\citep{tesauro1995temporal}, instead of gradient-TD, has been generalized to this MARL setting, with particular focuses on finite-sample analyses~\citep{doan2019finite}.


For optimal policy learning methods,
\textcolor{black}{
some early value-based multi-agent reinforcement learning algorithms include Team Q-learning~\citep{Littman2001ValuefunctionRL}, Distributed Q-learning~\citep{Lauer2000AnAF}, FMQ~\citep{Kapetanakis2002ReinforcementLO}, Hyper-Q learning~\citep{Tesauro2003ExtendingQT}, OAB~\citep{Wang2002ReinforcementLT}, Hysteretic Q-learning~\citep{Matignon2007HystereticQ} and Lenient Q-learning~\citep{Panait2008TheoreticalAO}.
All these works are based on tabular Q-learning.
\citet{Littman2001ValuefunctionRL,Lauer2000AnAF,Kapetanakis2002ReinforcementLO,Wang2002ReinforcementLT,Matignon2007HystereticQ,Panait2008TheoreticalAO} only consider the fully cooperative POSG setting and works \citet{Lauer2000AnAF,Matignon2007HystereticQ,Panait2008TheoreticalAO} have the limitation that they only work in deterministic environments.
}
In recent years,
\citet{kar2013cal} combines the idea of consensus and innovation to the standard Q-learning algorithm, proposes the $\mathcal{QD}$-learning algorithm.
\textcolor{black}{
\citet{Arslan2017DecentralizedQF} proposed a decentralized Q-learning algorithm for fully observable stochastic games with weakly acyclic and verified its effectiveness in some tabular cases.
The algorithm is fully decentralized in that each decision-maker has access to only its local information, the global state information, and the local cost function.
\citet{Zhang2018FiniteSampleAF} proposed two decentralized Q-learning algorithms which extended the fitted-Q iteration algorithm for single-agent RL into a multi-agent tabular-based scenario with both the cooperative and the competitive settings.
Each agent maintains a local estimation of the global average value function by exchanges local information over a communication network to keep the decentralization.
\citet{Yongacoglu2019LearningTF} proposed a variant of Q-learning algorithm for fully observable and fully cooperative stochastic games, which is conduct policy evaluation and policy improvement at two-timescale, and verified its team optimality in some tabular cases under the same decentralized settings as in \citet{Arslan2017DecentralizedQF}.
\citet{Cassano2019TeamPL} also proposed a decentralized tabular Q-learning algorithm, in which each agent maintains a local estimation of the global average value function for fully observable stochastic games under the same communicable settings as in \citet{Zhang2018FiniteSampleAF,zhang2018fully}.
\citet{Zeng2020FiniteTimeAO} proposed a similar MARL algorithm with the value function consensus constraint as \citet{Zhang2018FiniteSampleAF,zhang2018fully}, expect that it is a Q-learning algorithm and is to solve the multi-task fully observable stochastic game.
}

Compared with the value-based MARL methods, the actor-critic-based algorithms are also more applied to decentralized works.
Some works follow a local actor and a consensus critic update scheme.
\citet{cassano2018multi} and \citet{suttle2019multi} propose a fully decentralized policy gradient and actor-critic method to solve fully cooperative POSG problems separately. 
\citet{zhang2018fully} derives two multi-agent actor-critic algorithms for policy optimization, which employs a distributed evaluation strategy by combining diffusion learning \citep{sayed2014adaptive} and TD, to solve the fully observable stochastic games.
This method introduces each agent's local estimation of the counterfactual baseline and uses a consensus constraint to make the local estimation as accurate as possible compared to the centralized counterfactual baseline.
Although the above methods are decentralized for fully cooperative POSG (except for \citet{zhang2018fully}, which is proposed for general cooperative problem), they still assume that all agents share value functions. 
The 
massive information exchange makes the above method unable to extend to large-scale multi-agent scenarios. 
\citet{zhang2019distributed} is not the same as \citet{zhang2018fully} to learn the local estimate of the centralized optimal critic and then learn the optimal policy independently.
Instead, it learns the local estimate of the optimal centralized policy and then learns the optimal critic independently.
Later in~\citet{zhang2018networked}, the same idea as~\citet{zhang2018fully} is extended to the continuous control setting.
\citet{zhang2018networked} develop an on-policy actor-critic algorithm, using the recent development of the expected policy gradient method~\citep{whitesonexpected}.
\citet{qu2019value} is also working under the problem setting of~\citet{zhang2018fully}, which extends~\citet{dai2018sbeed} to the multi-agent scenario.
\textcolor{black}{
To solve a MARL problem with safety constraints, \citet{lu2021decentralized} formulate the problem as a distributed constrained Markov decision process with networked
agents and proposes a decentralized policy gradient method based on~\citet{zhang2018fully}.
\citet{qu2022scalable} extends~\citet{zhang2018fully} to large-scale multi-agent scenarios by introducing truncated Q-functions and~\citet{hu2021decentralized} further introduces reward machines on the basis of~\citet{qu2022scalable} to solve complex cooperation problems.
}

\paragraph{Remark 1.} Although these actor-critic-based methods are decentralized, they still assume that all agents \textcolor{black}{share value functions or policies}. 
The 
massive information exchange makes them unable to extend to large-scale multi-agent scenarios.
Meanwhile, the above actor-critic-based methods optimize the actor and critic \textit{seperately} via block coordinate gradient descent (BCGD) type methods according to the their related problem formulations.
This formulation is composed of successive independent actor and critic steps and ignores the influence between each other, which can introduce bias.
The biggest difference between these methods and the proposed algorithm framework is that a primal-dual hybrid gradient (PDHG) type method is introduced to \textit{jointly} optimize the actor and critic, thereby avoiding error accumulation and the unstable problem in practical deployment of methods above.
As will be shown, the proposed method not only solves the more general cooperative POSG problem, but also avoid direct sharing of policies between agents via a novel MOA module.

\section{Preliminaries}\label{sec:background}

{\color{black}{
\subsection{Markovian and Non-Markovian Policies and Values}
Under a few conditions~\citep{puterman2014markov}, for any given MDP, there exists an optimal policy which is deterministic and Markovian (that maps agent's immediate or current observation of the environment to actions).
In contrast, when dealing with partially observed MDPs (POMDPs), we might be interested in the class of non-Markovian and history-dependent policies (that map the state-action trajectory history to distributions over actions).
In this paper, we do not employ non-Markovian policies for several reasons.
    
Firstly, tackling non-Markovian and history dependent policies can be computationally undicidable~\citep{madani1999undecidability} for the infinite horizon, or PSPACE-Complete~\citep{papadimitriou1987complexity} in the finite horizon POMDPs.
When extended to the POSG, this problem will become more serious.
To avoid undesirable computational burdens, we focus on Makovian policies.

Secondly, indeed, many prior works study Markovian policies for POMDPs~\citep{baxter2001infinite,littman1994memoryless,baxter2000reinforcement,azizzadenesheli2016reinforcement} where, in general, the optimal policies are stochastic~\citep{littman1994memoryless,singh1994learning,montufar2015geometry}.
Acknowledging the computation complexity of Markovian policies~\citep{vlassis2012computational}, this line of work highlights the broad interest, importance, and applicability of Markovian policies.

Thirdly, one could alternatively consider history-dependent policies, which is a richer function class than Markovian policies.
However, utilizing history-dependent policies requires turning a given POMDP problem to a potentially non-stationary MDP with states as concatenations of the historical data.
As such, a thorough theoretical treatment of history-dependent policies for POMDPs would likely require leveraging theoretical analyses for non-stationary MDPs.
When extended to the POSG, the theoretical analysis of non-stationarity of POSG is still in a very preliminary stage. 
So much discussion in this area is beyond the scope of this paper, and we will delve into this in future work.

Fourthly, and most importantly, the expressiveness of general policy classes is mainly entangled with the definition of the observation.
Under some regularity conditions, for any given class of history-dependent policies on a POMDP, there exists a class of Markovian policies on a new POMDP~\citep{littman1994memoryless,castro2009equivalence,hausknecht2015deep,azizzadenesheli2018policy} such that:
(1) the observations of the new POMDP are the (typically) discrete concatenations of the historical data in the former POMDP; and 
(2) the two policy classes are equivalent, i.e., the repsective policies result in the same behavior (e.g., action sequence).
In other words, instead of making the policies non-Markovian and history-dependent, we can keep them Markovian and instead enrich the observation space.
Similarly, for any limited-memory policy class (depending on a fixed window of history instead of the whole history, e.g., the policies in MDPs of order more than one), there is an enrichment of the observation that results in an equivalent class of Markovian policies.
This viewpoint is known as the emergentism approach~\citep{sep-properties-emergent}.

Similar to the policy class, in this paper, we also focus on the Markovian value function.
Recent work~\citep{mao2020information} in cooperative POSG has shown that, with appropriate enrichment of the observation space, the optimal state-value function ($V$) or state-action-value function ($Q$) obtained under the newly constructed POSG can approximate the optimal one obtained under the original POSG.
However, \citet{mao2020information} adopts a value-based approach. 
\ul{Fortunately, after enriching the observation space, we can naturally extend the policy gradient theorem under MDP to POMDP} based on~\citet{mao2020information} as the basis for the subsequent theoretical derivation of this paper\footnote{\textcolor{black}{For the specific derivation process, please refer to the Appendix.~C}}.

The above reasoning implies that, by considering the class of Markovian policies and value functions in episodic POMDPs, we often will not restrict the generality of the results in this paper and similar conclusions can also be obtained by extending to the POSG.
Therefore, despite focusing on Markovian policies and value functions, our results also hold for both classes of limited-memory as well as history-dependent policies and value functions through representing histories as observations.
So in this paper, both $o$ and $\boldsymbol{O}$ refer to the \underline{enriched} observation spaces unless specified.
}}

{\color{black}{
\subsection{Problematic Instances of POMDPs}
Certain aberrant instances of worst-case scenarios in POMDP can potentially render the proposed F2A2 computationally infeasible. 
Specifically, in the realm of POMDPs, the identification of the system dynamics is hindered by pathological instances. 
Such cases are characterized by observations that offer no valuable information for discerning the system's behavior. 
The arduousness of the situation was proven by~\citet{krishnamurthy2016pac}, who demonstrated that, in the worst-case scenario, the discovery of a near-optimal policy for a POMDP necessitates an exponential number of samples proportional to the horizon. 
This result is formally expressed through the following proposition, where $\boldsymbol{\mathcal{A}}$ represents the joint action space with a cardinality of $|\boldsymbol{\mathcal{A}}|=A$.

\begin{prop}[\citealt{krishnamurthy2016pac,jin2020sample}]
    There exists a class of $2$-states $T$-horizon POMDPs whose observations reveal no information about the underlying states up to the end, such that any algorithm requires at least $A^{\Omega(T)}$ samples to learn an $\mathcal{O}(1)$-optimal policy with a probability of $1 / 2$ or higher.
\end{prop}

A running example can be found in~\citet[$\S$3.1 Hard instances of POMDPs]{liu2022partially}. 
For POMDPs with more than two states, the proposition presented earlier can be easily extended to situations where two mixtures of latent states with disjoint support exist, causing observations to fail to differentiate between the two mixtures. 
A mixture of states is represented by a probability vector $\nu \in \Delta_\mathcal{S}$, where $\nu_1$ and $\nu_2$ are considered to have disjoint support if $\operatorname{supp}\left(\nu_1\right) \cap \operatorname{supp}\left(\nu_2\right)=\varnothing$.

To avoid the aforementioned pathological instances, a simple approach is to assume that any two latent state mixtures $\nu_1$ and $\nu_2$ with disjoint support yield distinct distributions over observations, meaning that $\mathbb{O}_t \nu_1 \neq \mathbb{O}_t \nu_2$ for all $t \in[T]$, where $\mathbb{O}_h$ is the $O \times S$ emission matrix at step $h$. 
Here, $\boldsymbol{\mathcal{S}}$ and $\boldsymbol{\mathcal{O}}$ denote the state and joint observation spaces, respectively, with cardinalities of $|\boldsymbol{\mathcal{S}}|=S$ and $|\boldsymbol{\mathcal{O}}|=O$. 
It can be shown using a linear algebraic argument that this condition is equivalent to the rank of the emission matrix $\mathbb{O}_t$ being $S$.

\begin{prop}[\citealt{jin2020sample}]
    The emission matrix $\mathbb{O}_h$ is rank $S$ if and only if the induced distributions over observations are distinct for any two mixtures of latent states with disjoint support.
\end{prop}

Building upon the proposition discussed earlier, \citet{jin2020sample} suggested a more robust assumption to guarantee that observations possess adequate information to differentiate any two-state mixtures, given a sufficiently large number of samples.

\begin{assumption}[$\alpha$-weakly revealing condition, \citealt{liu2022partially}]\label{ass:weakly}
    There exists $\alpha>0$, such that $\min _t \sigma_S\left(\mathbb{O}_t\right) \geq \alpha$.
\end{assumption}

\citet{liu2022partially} refer to Assumption~\ref{ass:weakly} as the ``weakly'' revealing condition, in contrast to the rich observation or block MDP setup in the literature. 
However, it should be noted that Assumption~\ref{ass:weakly} implicitly assumes that $S \leq O$, as $\mathbb{O}_t$ is a matrix of size $O \times S$. 
Thus, it only holds in the \textit{undercomplete} setting. 
To ensure the generality of F2A2, we adopt the \textit{overcomplete} version assumption proposed by~\citep{liu2022partially} where $S>O$.

Observations in a single step are insufficient to distinguish between any two mixtures of latent states due to information-theoretic limitations.
Instead, we must examine the distribution of observations for a sequence of $m$ consecutive steps corresponding to the enriched observation space in F2A2. 
It is worth noting that the number of possible joint observable sequences of length $m$, $\left(\boldsymbol{o}_1, \boldsymbol{a}1, \ldots, \boldsymbol{a}{m-1}, \boldsymbol{o}_m\right)$, is $O^m A^{m-1}$, which is greater than $S$ when $m \geq \Omega(\log S)$. 
To formalize this, we define the m-step emission-action matrices $\left\{\mathbb{M}_t \in \mathbb{R}^{\left(A^{m-1} O^m\right) \times S}\right\}_{t \in[T-m+1]}$, and introduce the following assumption.

\begin{assumption}[$m$-step $\alpha$-weakly revealing condition, \citealt{liu2022partially}]\label{ass:weakly-m}
    There exists $m \in \mathbb{N}, \alpha>0$ such that $\min _{t \in[T-m+1]} \sigma_S\left(\mathbb{M}_t\right) \geq \alpha$.
\end{assumption}

Assumption~\ref{ass:weakly-m} guarantees that the observable sequence in the next $m$ consecutive steps contains enough information to distinguish any two mixtures of states when a sufficiently large number of enriched observations is available. 
This assumption encompasses Assumption~\ref{ass:weakly} as a specific case with $m=1$.
}}

\subsection{Cooperative Partially Observable Stochastic Games}

\textit{Partially observable stochastic game} (POSG)~\citep{hansen2004dynamic} is denoted as a tuple based on Markov Game as follows:
$$
\langle \mathcal{X}, \mathcal{S}, \left\{ \mathcal{A}^i \right\}_{i=1}^{n}, \left\{ \mathcal{O}^i \right\}_{i=1}^{n}, \mathcal{P}, \mathcal{E}, \left\{ \mathcal{C}^i \right\}_{i=1}^{n} \rangle,
$$
where $n$ is the total number of agents, $\mathcal{X}$ represents the agent space, $\mathcal{S}$ is a finite set of states,
	$\mathcal{A}^i$ is a finite action set of agent $i$, $\boldsymbol{\mathcal{A}}=\mathcal{A}^1\times\mathcal{A}^2\times\cdots\times\mathcal{A}^n$ is the finite set of joint actions,
	$\mathcal{P}(s^{\prime}|s, \boldsymbol{a})$ is the Markovian state transition probability function,
$\mathcal{O}^i$ is a finite observation set of agent $i$, $\boldsymbol{\mathcal{O}}=\mathcal{O}^1\times\mathcal{O}^2\times\cdots\times\mathcal{O}^n$ is the finite set of joint observations,
$\mathcal{E}(\boldsymbol{o}|s)$ is the Markovian observation emission probability function,
and $\mathcal{C}^i: \mathcal{S}\times\boldsymbol{\mathcal{A}}\times\mathcal{S} \rightarrow {\mathbb{R}}$ is the cost function of agent $i$.

The game in POSG unfolds over a finite or infinite sequence of stages (or timesteps), where the number of stages is called \textit{horizon}. 
This paper considers the episodic \textcolor{black}{infinite} horizon problem.
The objective for each agent is to minimize the expected \textcolor{black}{discounted} cumulative cost received during the game. 
For a cooperative POSG, the definition in \citep{song2019arena} is quoted,
\begin{equation}
\forall x \in \mathcal{X}, \forall x^{\prime} \in \mathcal{X} \backslash\{x\}, \forall \pi_{x} \in \Pi_{x}, \forall \pi_{x^{\prime}} \in \Pi_{x^{\prime}}, \frac{\partial \mathcal{C}^{x^{\prime}}}{\partial \mathcal{C}^{x}} \geqslant 0, \nonumber
\end{equation}
where $x$ and $x^\prime$ are a pair of agents in agent space $\mathcal{X}$;
$\pi_{x}$ and $\pi_{x^\prime}$ are the corresponding policies in the policy space $\Pi_{x}$ and $\Pi_{x^\prime}$ separately.
Intuitively, this definition means that there is no conflict of interest for any pair of agents.
The most common example of cooperative POSG is the fully cooperative POSG (also called decentralized partially observable Markov decision process, Dec-POMDP), that all the agents share the same global cost at each stage, and $\mathcal{C}^1 = \mathcal{C}^2 = \cdots = \mathcal{C}^n$.

This paper aims to solve the general cooperative POSG. 
Each agent completes a common task based on their local observations, cost, and learning process.
Without loss of generality, the optimization goal of the general cooperative POSG problem is defined as follows
\begin{equation}\label{eq:optimization-goal-general cooperative POSG}
\min_{\Psi} \sum_{i=1}^{n} \sum_{t=0}^{\textcolor{black}{\infty}} \mathbf{E}_{\textcolor{black}{s_0 \sim p_0}, \boldsymbol{o} \sim \mathcal{E}, a \sim \boldsymbol{\pi}_{\Psi}} \left[ \textcolor{black}{\gamma^{t}}c_{t+1}^{i} \right],
\end{equation}
where $\Psi:=\{\psi^i\}_{i=1}^{n}$ denotes parameters of the approximated policy \textcolor{black}{$\pi^i_{\psi^i}:\mathcal{O}^i\rightarrow\mathcal{A}^i$} of all agents and $\boldsymbol{\pi}_{\Psi}:=\prod_{i=1}^{n} \pi^i_{\psi^i}$ represents the factorizable joint policy of all agents.
\textcolor{black}{
$\gamma$ is the discount factor.
Note $p_{0}$ is the distribution of the initial state $s_0$.}
\textcolor{black}{$c_{t+1}^{i}$ represents the reward received by the agent $i$ at timestep $t+1$ after executes action $a_t^i$ in local observation $o_t^i$.}

{\color{black}{

\subsection{Single-agent Actor-critic-type Algorithms for POMDP}

Single-agent actor-critic methods optimize actor $\pi(\cdot|o;\psi)$ directly by minimizing the expected discounted accumulated cost:
\begin{equation}\label{eq:actor-critic-0}
    \sum_{t=0}^{\infty} \mathbf{E}_{s_0 \sim p_{0}, o \sim \mathcal{E}, a \sim \pi_{\psi}} \left[ \gamma^t c_{t+1}\right],
\end{equation}
where $\psi$ is the parameters of the approximated policy, $p_{0}$ is the distribution of the initial state $s_0$.
However, this may lead to a high estimation variance with policy gradient methods. 
Instead, a critic $V_{\phi}(o)$ is introduced to estimate the expected accumulate cost.
The optimization formulation of the vanilla actor-critic algorithm~\citep{sutton2000policy} can be reformalized into a bi-level problem as follows~\citep{Yang2018ConvergentRL,Yang2019ProvablyGC,Hong2020ATF}:
\begin{equation}
    \begin{aligned}
        &\psi^{*} = \arg\min_{\psi}  J_{\text{actor}}(\psi, \phi^{*}(\psi)) := \mathbf{E}_{s_0 \sim p_{0}, o_0 \sim \mathcal{E}} \left[V^{\pi}_{\phi^{*}}(o_0) \right],\\
        &\text{where } \phi^{*}(\psi) = \arg\min_{\phi} J_{\text{critic}}(\psi, \phi) := \mathbf{E}_{s \sim d_{\psi}, o \sim \mathcal{E}} \left[ \left( V^{\pi}_{\phi}(o) - V^{\pi}_{tg}(o) \right)^2 \right],\label{eq:actor-critic}
    \end{aligned}
\end{equation}
where $V^{\pi}_{tg}(o):=\mathbf{E}_{a\sim\pi_{\psi},s'\sim\mathcal{P}, o^{'}\sim\cal{E}}\left(\mathcal{C}_{s, a}^{s'}+\gamma V^{\pi}_{\phi}\left(o'\right) \right)$ and $\mathcal{C}_{s, a}^{s'}$ is equivalent to $\mathcal{C}^i(s,a,s')$ defined in previous subsection.
$d_{\psi}$ is the distribution of the state-occupancy measure of policy $\pi_{\psi}$.

Note the second term of the right side of the Eq.~\eqref{eq:actor-critic} is the \textit{bellman error} and the $V_{tg}(\cdot)$ stands for the temporal difference target in it.
It can be seen that the first term of the right side of the Eq.~\eqref{eq:actor-critic} is equivalent to Eq.~\eqref{eq:actor-critic-0}.
The critic $V_{\phi}(\cdot)$ is introduced to estimate the expected accumulated cost to reduce the variance.
Moreover, correspondingly the second term, bellman error term, is used to fit the introduced critic $V_{tg}(\cdot)$.
Traditional actor-critic algorithms~\citep{lillicrap2015continuous,schulman2015trust,schulman2017proximal,haarnoja2018soft, fujimoto2018addressing}, minimizes Eq.~(\ref{eq:actor-critic}) where updates the actor and the critic parameters alternatively as follows,
\begin{subequations}\label{eq:algorithm-actor-critic}
    \begin{align}
        \min_{\psi} J_{\text{actor}}(\psi, \phi) &:= \mathbf{E}_{s_0 \sim p_{0}, o_0 \sim \mathcal{E}} \left[V^{\pi}_{\phi}(o_0) \right],\label{eq:actor}\\
        \min_{\phi} J_{\text{critic}}(\psi, \phi) &:= \mathbf{E}_{s \sim d_{\psi}, o \sim \mathcal{E}} \left[ \left( V^{\pi}_{\phi}(o) - V^{\pi}_{tg}(o) \right)^2 \right].\label{eq:critic}    
    \end{align}
\end{subequations}The algorithm scheme \eqref{eq:algorithm-actor-critic} can be considered as applying the block coordinate gradient descent (BCGD) type  algorithm on (\ref{eq:actor-critic}).

It is worth noting that when solving (\ref{eq:critic}), the bootstrapping method is generally used~\citep{sutton2018reinforcement}.
Bootstrapping methods are not in fact instances of true gradient descent~\citep{barnard1993temporal}.
They take into account the effect of changing the parameter $\phi$ on the estimated value, but ignore its effect on the target.
They include only a part of the gradient and, accordingly, are called \underline{semi-gradient methods}.
Although semi-gradient (bootstrapping) methods do not converge as robustly as gradient methods, they do converge reliably in important cases such as the linear case~\citep{sutton2018reinforcement}.
Moreover, they offer important advantages that make them ofter clearly preferred.
One reason for this is that they typically enable significantly faster learning.
Another is that they enable learning to be continual and online, without waiting for the end of an episode.
This enables them to be used on continuing problems and provides computational advantages.

}}

\section{Joint Actor-Critic MARL Framework}
\label{sec:model}
Actor-critic type algorithms are popular for MARL, which combines the advantages of both policy gradient and value-based methods, often being more tractable and efficient in either high-dimensional or continuous action space. Therefore, many existing MARL methods~\citep{foerster2018counterfactual,wei2018multiagent,iqbal2019actor} are based on the actor-critic framework.

For general cooperative POSG, this paper first draw on the work of \citet{Yang2018ConvergentRL}, \citet{Yang2019ProvablyGC} and \citet{Hong2020ATF} to look at the actor-critic algorithm from a bi-level perspective and extend it to multi-agent scenarios,
\begin{equation}
\min_{\{\psi^i\}}\sum_{i=1}^n\mathbf{E} \left[V^{\boldsymbol{\pi},i}_{\phi^{i,\star}}(\boldsymbol{o_0}) \right]\ \ s.t.\ \{\phi^{i,\star}\}=\arg\min_{\{\phi^i\}}\sum_{i=1}^n \mathbf{E} \left[ \left( V^{\boldsymbol{\pi},i}_{\phi^i}(\boldsymbol{o}) - V^{\boldsymbol{\pi},i}_{tg}(\boldsymbol{o}) \right)^2 \right],\label{eq:multi-actor-critic-bi}
\end{equation}
where the expectations are taken on $s_0 \sim d_0, s \sim d_{\Psi}, \boldsymbol{o}_0,\boldsymbol{o} \sim \mathcal{E}$; $\{ \phi^i \}$ and $\{ \psi^i \}$ denotes the critic parameters and the actor parameters of all agents respectively, which are alternatively optimized; 
$V^{\boldsymbol{\pi},i}_{tg}(\boldsymbol{o})$ is defined as
$$
\mathbf{E}_{a \sim \boldsymbol{\pi}_{\Psi}, s'\sim\mathcal{P}, \boldsymbol{o}^{'}\sim\cal{E}}\left(\mathcal{C}_{s, \boldsymbol{a}}^{s',i}+ \gamma V^{\boldsymbol{\pi},i}_{\phi^i}\left(\boldsymbol{o}'\right) \right).
$$
Then, the above bi-level form can be simplified to a single-level problem with a penalty term
and obtain the standard optimization formulation for multi-agent cases:
\begin{equation}
    \begin{aligned}
        \min_{\{{\psi}^{i}\}, \{{\phi}^{i}\}}\  \left\{\mathcal{J}_{\text{ac}}(\{{\psi}^{i}\}, \{{\phi}^{i}\}) := 
        \alpha_1\sum_{i=1}^n \mathbf{E}\left[V^{\boldsymbol{\pi},i}_{\phi^i}(\boldsymbol{o_0}) \right] + \alpha_2\sum_{i=1}^n \mathbf{E} \left[ \left( V^{\boldsymbol{\pi},i}_{\phi^i}(\boldsymbol{o}) - V^{\boldsymbol{\pi},i}_{tg}(\boldsymbol{o}) \right)^2 \right]\right\}.\label{eq:multi-actor-critic}
    \end{aligned}
\end{equation}
For each agent $i$, the critic $V^{\boldsymbol{\pi},i}_{\phi^i}(\boldsymbol{o})$ is determined by all $\boldsymbol{\pi} = \prod_{i=1}^n\pi^i_{\psi^i}$ and the specific $\phi^i$ based on joint observation $\boldsymbol{o}$; 
$\pi^i$ represents the policy of each agent $i$.
It is worth noting that each agent's critic $V^{\boldsymbol{\pi},i}_{\phi^i}(\boldsymbol{o})$ defined in Equation~\eqref{eq:multi-actor-critic} are only based on their critic parameters $\phi^i$ but on all agents' actor parameters $\{\psi^i\}$.
The reason is that in the general cooperative POSG problem, each agent has an independent cost function $\mathcal{C}^i$ mentioned in Section 3.1 but this cost function is based on the policy of all agents.
Further, the objective function for each agent $i$ in actor and critic parameter updating phases can be defined as follows respectively:
\begin{equation}\label{eq:multi-actor-critic-agent-i}
\begin{aligned}
J_{\text{actor}}^i(\{ \psi^i\}, \phi^i) &:= \mathbf{E}\left[V^{\boldsymbol{\pi},i}_{\phi^i}(\boldsymbol{o_0}) \right],\\
J_{\text{critic}}^i(\{ \psi^i\}, \phi^i) &:= \mathbf{E} \left[ \left( V^{\boldsymbol{\pi},i}_{\phi^i}(\boldsymbol{o}) - V^{\boldsymbol{\pi},i}_{tg}(\boldsymbol{o}) \right)^2 \right].
\end{aligned}
\end{equation}
Therefore, the optimization problem (\ref{eq:multi-actor-critic}) can be reformulated into
\begin{equation}
\label{eq:framework-v2}
\min_{ {\mathbf{w}} } {\cal{J}}({\mathbf{w}}):= \overbrace{\alpha_1 \underbrace{\sum_{i=1}^n J_{\text{actor}}^i(\{ \psi^i\}, \phi^i) }_{{\cal{J}}_{\text{actor}} ({\mathbf{w}})} + \alpha_2 \underbrace{\sum_{i=1}^n  J_{\text{critic}}^i(\{ \psi^i\}, \phi^i)}_{{\cal{J}}_{\text{critic}} ({\mathbf{w}})} }^{\mathcal{J}_{\text{ac}}(\mathbf{w})} + {\cal{R}} ( {\mathbf{w}} ), 
\end{equation}
where ${\mathbf{w}}:= (\{ \psi^i\}, \{\phi^i\})$ denotes all the parameters;
$\alpha_1$ and $\alpha_2$ denote penalty factors in actor and critic, respectively;
${\cal{R}}(\cdot)$ denotes an \textit{optional} regularizer to prevent over-fitting or for sparsity and etc., and it is assumed with separable structure in the proposed framework.
Some works \citep{farebrother2018generalization,liu2019regularization} have shown that model regularization techniques have a significant impact on the performance of single-agent reinforcement learning models.
The scale of multi-agent reinforcement learning models is generally much larger than that of single-agent models. 
Thus the impact of regularization should not be ignored.
Therefore, the regularization term ${\cal{R}}(\cdot)$ is usually added to the multi-agent actor-critic objective function to make the proposed framework more generalizable.

When solving the general cooperative POSG problem, in order to learn the decentralized policies to achieve global collaboration, the following three types of techniques are mainly used~\citep{heider1944experimental,rasouli2017agreeing,de2019multi}.
(1) \textit{common knowledge} based algorithms~\citep{brafman2003learning,de2019multi} use the common knowledge protocol to achieve global collaboration by establishing common knowledge about other agents' actions or observations, based on the global common knowledge generator;
(2) \textit{explicit communication} based algorithms~\citep{peng2017multiagent,jiang2018learning} achieve certain consensus by explicit exchanging information in the decision phase. But the information exchanged need to be able to be understood by all agents, so the message generation module also needs to be shared between agents;
(3) \textit{implicit communication} algorithms~\citep{foerster2018counterfactual,iqbal2019actor} directly share individual action and/or observation between agents. At the same time, similar to the explicit communication-based algorithms, the model of each agent to process this global information must also be consistent with maintaining the same understanding of the environment for the achievement of global collaboration.

In order to propose an effective framework to be compatible with these collaboration skills, 
and considering that the core modules in these methods can be accessed globally, 
the flexible \textit{parameter sharing} mechanism is introduced to the proposed framework.
In addition, the parameter sharing mechanism also can be used for better algorithm scalability~\citep{yang2018mean}.
\textcolor{black}{Recent work~\citep{Terry2020RevisitingPS,Christianos2021ScalingMR,Grupen2021FairnessFC} also shows that parameter sharing plays a crucial role in improving algorithm performance.}
To achieve highly integration of the parameter sharing mechanism and the general cooperative POSG optimization objective, without loss of generalization, the parameters are reformulated and divided into shared and non-shared parts, denoted as $\mathbf{w}_{in}$ and $\mathbf{w}_{sh}$ respectively(``{\em{in}}" and ``{\em{sh}}" denote ``individual" and ``shared" parameters, respectively). 
Then, the problem (\ref{eq:framework-v2}) can be reformulated into a more general form, i.e., 
\begin{equation}
\min_{\mathbf{w}_{in}, \mathbf{w}_{sh}} \mathcal{J}(\mathbf{w}_{in}, \mathbf{w}_{sh}) =  \alpha_1 {\cal{J}}_{\text{actor}}(\mathbf{w}_{in}, \mathbf{w}_{sh}) + \alpha_2 {\cal{J}}_{\text{critic}}(\mathbf{w}_{in}, \mathbf{w}_{sh}) + \mathcal{R}(\mathbf{w}_{in}, \mathbf{w}_{sh}),\label{eq:framework-v3}
\end{equation}where the general form of ${\cal{J}}_{\text{actor}}, {\cal{J}}_{\text{critic}}$ and $\mathcal{R}(\mathbf{w}_{in}, \mathbf{w}_{sh})$ are as follows:
\begin{eqnarray}
{\cal{J}}_{\text{actor}}(\mathbf{w}_{in}, \mathbf{w}_{sh}) \!\!\!&=&\!\!\! \sum_{i=1}^n J^i_{\text{actor}} \left( \{ \psi_{in}^i\}, \phi_{in}^i, \psi_{sh}, \phi_{sh} \right),\nonumber\\
{\cal{J}}_{\text{critic}}(\mathbf{w}_{in}, \mathbf{w}_{sh}) \!\!\!&=&\!\!\! \sum_{i=1}^n J^i_{\text{critic}} \left( \{ \psi_{in}^i\}, \phi_{in}^i, \psi_{sh}, \phi_{sh} \right),\nonumber\\
\mathcal{R}(\mathbf{w}_{in}, \mathbf{w}_{sh}) \!\!\!&=&\!\!\! \sum_{i=1}^n \underbrace{\left( {{r_{\psi}^{se} ( \psi_{in}^i) + r_{\phi}^{se} (\phi_{in}^i) }}\right)}_{{\mathcal{R}}_{in} (\mathbf{w}_{in})} + \underbrace{r_{\psi}^{sh} (\psi_{sh}) + r_{\phi}^{sh} (\phi_{sh})}_{{\mathcal{R}}_{sh} (\mathbf{w}_{sh})},\nonumber
\end{eqnarray}
with $\mathbf{w}_{in}=(\{ \psi_{in}^i\}, \{ \phi_{in}^i\})$ and $\mathbf{w}_{sh} = (\psi_{sh}, \phi_{sh})$; the regularizer ${\mathcal{R}}(\cdot)$ is separable with $r_{\psi}^{se}, r_{\phi}^{se}, r_{\psi}^{sh}$ and $r_{\phi}^{sh}$ be related regularization functions.
Many existing MARL algorithms are equivalent to solve this general optimization formulation \eqref{eq:framework-v3}.
Some state-of-the-art algorithms, i.e. MADDPG~\citep{lowe2017multi}, COMA~\citep{foerster2018counterfactual}, MAAC~\citep{iqbal2019actor} and etc are summarized in the following Table \ref{table:Summary}, which include all elements $\mathbf{w}_{in}$, $\mathbf{w}_{sh}$, ${\cal{J}}_{\text{actor}}$, ${\cal{J}}_{\text{critic}}$ in framework (\ref{eq:framework-v3}) for them. More detailed derivation can be found in supplemental material.

\begin{table}[tb!]
\centering
    \resizebox{\textwidth}{!}
    {\begin{tabular}{|c|c|c|c|c|c|c|}
	\hline
    \multirow{2}{*}{Algorithm} & \multicolumn{2}{|c|}{$\mathbf{w}_{in}$} & \multicolumn{2}{c|}{$\mathbf{w}_{sh}$}& \multicolumn{2}{c|}{${\cal{J}}_{\text{actor}}$, ${\cal{J}}_{\text{critic}}$} \\
	\cline{2-7}
	 & $\psi_{in}^i$ & $\phi_{in}^i$ & $\psi_{sh}$ & $\phi_{sh}$ & $J^i_{\text{actor}} \left(\{ \{ \psi_{in}^i\}, \phi_{in}^i, \psi_{sh}, \phi_{sh} \right)$ & $J^i_{\text{critic}} \left( \{ \psi_{in}^i\}, \phi_{in}^i, \psi_{sh}, \phi_{sh} \right)$ \\
	\hline
	MADDPG & $\pi^i(o^i)$ & $\varnothing$ & $\varnothing$ & $Q(\boldsymbol{o},\boldsymbol{a})$ & $\mathbb{E}\left[Q^{\boldsymbol{\pi},i}_{\phi_{sh}}(\boldsymbol{o},\boldsymbol{a})\right]$ & $\mathbb{E} \left[ \left( Q^{\boldsymbol{\pi},i}_{\phi_{sh}}(\boldsymbol{o}, \boldsymbol{a}) - Q^{\boldsymbol{\pi},i}_{tg} \right)^2 \right]$ \\
	\hline
	COMA & \makecell{$\pi^i(o^i)$(GRUs\footnote{Gated recurrent units.})} & $\varnothing$ & $\varnothing$ & \makecell{$Q(\boldsymbol{o}, \boldsymbol{a})$ \\ $V(\boldsymbol{o})$(MLP\footnote{Multiple layer perceptron.})} & $\mathbb{E}\left[Q^{\boldsymbol{\pi},i}_{\phi_{sh}}(\boldsymbol{o},\boldsymbol{a})  - \mathcal{B}(\boldsymbol{o}, \boldsymbol{a^{\setminus i}})\right]$ & $\mathbb{E} \left[ \left( Q^{\boldsymbol{\pi},i}_{\phi_{sh}}(\boldsymbol{o}, \boldsymbol{a}) - Q^{\boldsymbol{\pi},i}_{tg} \right)^2 \right]$ \\
	\hline
	MAAC & \makecell{$\pi^i(o^i)$(MLP)} & $\zeta^i, \digamma^i$ & $\varnothing$ & $\left\{ V_h, W_h^{key}, W_h^{que} \right\}_{h=1}^{H}$ & $\mathbb{E}\left[Q^{\boldsymbol{\pi},i}_{\phi_{in}^i, \phi_{sh}}(\boldsymbol{o},\boldsymbol{a}) + \alpha \mathcal{H}(\cdot|\pi^i_{\psi_{in}^i}(o^i)) - \mathcal{B}(\boldsymbol{o}, \boldsymbol{a^{\setminus i}})\right]$ & $\mathbb{E} \left[ \left( Q^{\boldsymbol{\pi},i}_{\phi_{in}^i, \phi_{sh}}(\boldsymbol{o}, \boldsymbol{a}) - Q^{\boldsymbol{\pi},i}_{tg} \right)^2 \right]$ \\
	\hline
	\end{tabular}}
    \caption{Summary of MADDPG, COMA and MAAC algorithms.}
	\label{table:Summary}
\end{table}

Most MARL methods solve the actor-critic framework based on the block coordinate gradient descent (BCGD) type techniques,
whose standard procedure is composed of successive independent actor and critic steps, i.e.,
\begin{equation}\label{eq:iterative-scheme}
    \left\{\begin{array}{rrl}
            {\hbox{Critic-step:}}& \left(\begin{array}{c} \phi_{in}^i \\ \phi_{sh} \end{array} \right) \leftarrow \left(\begin{array}{c} \phi_{in}^i \\ \phi_{sh} \end{array} \right) - \gamma \left(\begin{array}{c} \nabla_{\phi_{in}^i} {\cal{J}}_{\text{critic}}(\mathbf{w}_{in}, \mathbf{w}_{sh})  \\ \nabla_{\phi_{sh}} {\cal{J}}_{\text{critic}}(\mathbf{w}_{in}, \mathbf{w}_{sh}) \end{array} \right);\\
            {\hbox{Actor-step:}}& \left(\begin{array}{c} \psi_{in}^i \\ \psi_{sh} \end{array} \right) \leftarrow \left(\begin{array}{c} \psi_{in}^i \\ \psi_{sh} \end{array} \right) - \gamma \left(\begin{array}{c} \nabla_{\psi_{in}^i} {\cal{J}}_{\text{actor}}(\mathbf{w}_{in}, \mathbf{w}_{sh})  \\ \nabla_{\psi_{sh}} {\cal{J}}_{\text{actor}}(\mathbf{w}_{in}, \mathbf{w}_{sh}) \end{array} \right),
            \end{array}\right. 
\end{equation}
while the actor parameters $(\left\{\psi_{in}^i\right\},\psi_{sh})$ and the critic parameters $(\left\{\phi_{in}^i\right\},\phi_{sh})$ are fixed in critic and actor steps respectively.
By the way, the three sate-of-the-art MARL algorithms mentioned above follow this iterative scheme (\ref{eq:iterative-scheme}), and the details are presented in supplemental material.
In this case, together with the existence of sharing parameters, most typical MARL algorithms are included in the CTDE algorithm framework.
These algorithms have to maintain a globally accessible shared module in order to handle the parameter sharing.
More important, when optimizing actor and critic parts separately, they ignored the influence between each other.
Although this separation optimization scheme simplifies the solution calculation, it also introduces serious bias.

\section{Flexible Fully-decentralized Approximate Actor-critic Framework}
\label{sec:algorithm}

Current decentralized multi-agent actor-critic algorithms~\citep{zhang2018networked,zhang2018fully,zhang2019distributed, suttle2019multi} also use BCGD-type procedure to optimize actor and critic of each agent.
For example, \citet{zhang2018fully} can be regarded as a distributed version of the COMA.
The core idea of \citet{zhang2019distributed} is the opposite of~\citet{zhang2018fully}, and it also uses the same procedure as the former.
In centralized training, the concentrative information collection and process might guarantee the global convergence to a certain degree~\citep{konda2000actor}.
However, in decentralized training, the hysteretic information exchange would cause error accumulation and an unstable problem in many practical deployments.
The standard procedure needs to be modified to a more rational and flexible form to achieve full decentralization, satisfying the demand for the mutual observation, reward assignment, and policy interaction between agents, and synchronize the optimization.
At the same time, the above algorithms generally assume that the local observations and policies of other agents are known when performing decentralized optimization. 
This constraint requires a large amount of inter-agent communication so that the above-mentioned decentralized algorithms cannot be extended to a large-scale multi-agent environment.

Based on the analysis above, the following issues need to be considered to propose a fully decentralized multi-agent actor-critic algorithm.
First, the algorithm can solve error accumulation and instability problems caused by the BCGD-type procedure of the current multi-agent actor-critic algorithms under decentralized training.
Second, it can work with fewer constraints. 
It does not require the precise policies of all other agents, thereby avoiding the high communication costs.
Third, it can be flexibly combined with most actor-critic algorithms and compatible with on-policy and off-policy techniques.

Therefore, the fully decentralized algorithm framework is proposed to solve the general formulation above (\ref{eq:framework-v3}).
Firstly, the consensus variables $\left\{\tilde{\mathbf{w}}_{sh}^i \right\}$ are introduced to help achieve the fully decentralized structure.
Recall the definition of ${\cal{J}}_{\text{actor}}(\mathbf{w}_{in}, {\mathbf{w}}_{sh})$, ${\cal{J}}_{\text{critic}}(\mathbf{w}_{in}, \mathbf{w}_{sh})$ and the regularizer $\mathcal{R}(\mathbf{w}_{in}, \mathbf{w}_{sh})$, the comprehensive formulation can be obtained as follows:
\begin{eqnarray}
\min_{\mathbf{w}_{in}, \left\{\tilde{\mathbf{w}}_{sh}^i \right\}, \mathbf{w}_{sh} } && \alpha_1 \sum_{i=1}^n J^i_{\text{actor}} \left( \mathbf{w}_{in}, \tilde{\mathbf{w}}_{sh}^i \right) + \alpha_2 \sum_{i=1}^n J^i_{\text{critic}} \left( \mathbf{w}_{in}, \tilde{\mathbf{w}}_{sh}^i \right) + \mathcal{R}(\mathbf{w}_{in}, \mathbf{w}_{sh}), \nn \\ 
{\hbox{s.t.}} && \mathbf{w}_{sh} = \tilde{\mathbf{w}}_{sh}^i,\ i = 1,\cdots,n.
\label{eq:framework-v4}
\end{eqnarray}
This can be considered as an equivalent reformulation of the general formulation (\ref{eq:framework-v3}). 
(\ref{eq:framework-v4}) is a typical linear constrained optimization problem while the consensus constraints only relate to the shared parameters for all $n$ agents.
$\left\{\tilde{\mathbf{w}}_{sh}^i\right\}$ are introduced to help communicating the shared parameters and each $\tilde{\mathbf{w}}_{sh}^i$ belongs to agent $i$ respectively.
As a result, the primal-dual hybrid gradient (PDHG) type method (or inexact alternating direction method of multipliers (ADMM) type method) is utilized  to solve (\ref{eq:framework-v4}) because the PDHG-type method naturally has decentralized computing architecture.
The augmented Lagrangian function is defined as follows,
\begin{eqnarray}
\mathcal{L} \left( \mathbf{w}_{in}, \left\{\tilde{\mathbf{w}}_{sh}^i\right\}, \mathbf{w}_{sh}, \left\{\boldsymbol{\lambda}_i \right\} \right) \!\!\!&=&\!\!\! \overbrace{\alpha_1 \sum_{i=1}^n J^i_{\text{actor}} \left( \mathbf{w}_{in}, \tilde{\mathbf{w}}_{sh}^i \right) + \alpha_2 \sum_{i=1}^n J^i_{\text{critic}} \left( \mathbf{w}_{in}, \tilde{\mathbf{w}}_{sh}^i \right)}^{\sum_{i=1}^n J_{ac}^i \left( \mathbf{w}_{in}, \tilde{\mathbf{w}}_{sh}^i \right)} + \mathcal{R}(\mathbf{w}_{in}, \mathbf{w}_{sh})\nonumber\\
&&\qquad - \sum_{i=1}^n \left\langle \boldsymbol{\lambda}_i, \mathbf{w}_{sh} - \tilde{\mathbf{w}}_{sh}^i \right\rangle + \frac{\beta}{2} \sum_{i=1}^n  \left\| \mathbf{w}_{sh} - \tilde{\mathbf{w}}_{sh}^i \right\|^2,
\end{eqnarray}where $\left\{\boldsymbol{\lambda}_i\right\}$ denote the Lagrangian dual variables concerning the consensus linear constraints with a unified penalty parameter $\beta$ (in order to express more clearly, a unified penalty parameter is used, and it can also be modified into separate and different parameters).
Motivated by the popular ADMM, the PDHG algorithm framework can be designed by alternatively calculate the primal variables $\mathbf{w}_{in}, \left\{\tilde{\mathbf{w}}_{sh}^i\right\}, \mathbf{w}_{sh}$ and the dual variables $\left\{\boldsymbol{\lambda}_i\right\}$.
Similar to the definition $\mathbf{w}_{in}=(\{ \psi_{in}^i\}, \{ \phi_{in}^i\})$ and $\mathbf{w}_{sh} = (\psi_{sh}, \phi_{sh})$, we have
$
\tilde{\mathbf{w}}_{sh}^i = (\tilde{\psi}_{sh}^i, \tilde{\phi}_{sh}^i).
$
The augmented Lagrangian $\mathcal{L} \left( \mathbf{w}_{in}, \left\{\tilde{\mathbf{w}}_{sh}^i\right\}, \mathbf{w}_{sh}, \left\{\boldsymbol{\lambda}_i \right\} \right)$ can be denoted equivalently as
\begin{equation}\label{eq:final-obj}
\mathcal{L} \left( \left\{\left(\phi_{in}^i,\tilde{\phi}_{sh}^i  \right)\right\}, \left\{\left( \psi_{in}^i,\tilde{\psi}_{sh}^i \right) \right\},  \left( \psi_{sh},\phi_{sh} \right), \left\{\boldsymbol{\lambda}_i \right\} \right).
\end{equation}
The classical ADMM framework works on this augmented Lagrangian function $\mathcal{L}$, and in each iteration it minimizes variable blocks $\left\{ \left(\phi_{in}^i,\tilde{\phi}_{sh}^i  \right) \right\}$, $\left\{ \left( \psi_{in}^i,\tilde{\psi}_{sh}^i \right) \right\}$ and $\left( \psi_{sh},\phi_{sh} \right)$ based on Gauss-Seidel scheme and further updates the Lagrangian multiplier $\left\{ \boldsymbol{\lambda}_i \right\}$.
Instead of minimizing the Lagrangian function directly, the gradient descent technique is employed to approximately updating primal variable blocks progressively, and as a result, the brief primal-dual hybrid gradient algorithm framework is proposed in the following calculation scheme \citep{Boyd2011,Chambolle2011}
\begin{subequations}\label{eq:algorithm-scheme}
\begin{align}
\left[\!\!\begin{array}{c} \phi_{in}^i\\ \tilde{\phi}_{sh}^i \end{array} \!\! \right] &\leftarrow \left[ \!\! \begin{array}{c} \phi_{in}^i \\ \tilde{\phi}_{sh}^i \end{array}\!\!\right] - \textcolor{black}{\beta_{\phi}}\left[ \!\! \begin{array}{c} \nabla_{\phi_{in}^i} \mathcal{L} \left( \mathbf{w}_{in}, \left\{\tilde{\mathbf{w}}_{sh}^i\right\}, \mathbf{w}_{sh}, \left\{\boldsymbol{\lambda}_i \right\} \right) \\ \nabla_{\tilde{\phi}_{sh}^i} \mathcal{L} \left( \mathbf{w}_{in}, \left\{\tilde{\mathbf{w}}_{sh}^i\right\}, \mathbf{w}_{sh}, \left\{\boldsymbol{\lambda}_i \right\} \right)
\end{array} \!\! \right],\quad i=1,\cdots,n; \label{eq:algorithm-scheme-phi} \\
\left[\!\!\begin{array}{c} \psi_{in}^i\\ \tilde{\psi}_{sh}^i \end{array} \!\! \right] &\leftarrow \left[ \!\! \begin{array}{c} \psi_{in}^i\\ \tilde{\psi}_{sh}^i \end{array}\!\!\right] - \textcolor{black}{\beta_{\psi}}\left[ \!\! \begin{array}{c} \nabla_{\psi_{in}^i} \mathcal{L} \left( \mathbf{w}_{in}^i, \left\{\tilde{\mathbf{w}}_{sh}^i\right\}, \mathbf{w}_{sh}, \left\{\boldsymbol{\lambda}_i \right\} \right) \\ \nabla_{\tilde{\psi}_{sh}^i} \mathcal{L} \left( \mathbf{w}_{in}, \left\{\tilde{\mathbf{w}}_{sh}^i\right\}, \mathbf{w}_{sh}, \left\{\boldsymbol{\lambda}_i \right\} \right)
\end{array} \!\! \right],\quad i=1,\cdots,n; \label{eq:algorithm-scheme-psi} \\
\left[\!\!\begin{array}{c} \psi_{sh}\\ \phi_{sh} \end{array} \!\! \right] &\leftarrow \left[ \!\! \begin{array}{c} (1/n)\sum_{i=1}^n \tilde{\psi}_{sh}^i \\ (1/n)\sum_{i=1}^n \tilde{\phi}_{sh}^i \end{array}\!\!\right]; \label{eq:algorithm-scheme-sh} \\
            \boldsymbol{\lambda}_i &\leftarrow \boldsymbol{\lambda}_i - \beta \left( \mathbf{w}_{sh} - \tilde{\mathbf{w}}_{sh}^i \right),\quad i=1,\cdots,n.\label{eq:algorithm-scheme-lambda}
\end{align}
\end{subequations}
The update step (\ref{eq:algorithm-scheme-phi}) calculates the critic parameters $\phi_{in}^i$ together with the splitting shared critic parameters $\tilde{\phi}_{sh}^i$, and (\ref{eq:algorithm-scheme-psi}) updates the actor parameters $\psi_{in}^i$ with the splitting shared actor parameters $\tilde{\psi}_{sh}^i$.
(\ref{eq:algorithm-scheme-sh}) aims to update the shared parameters $(\phi_{sh},\psi_{sh})$ by averaging all the splitting shared critic and actor respectively.
It is evident that each agent $i$ computes its actor and critic parameters and its splitting shared parameters.
The overall shared parameters are updated through (\ref{eq:algorithm-scheme-sh}) and broadcast to all agents.
(\ref{eq:algorithm-scheme}) can be decentralized implemented and full details can be found in Algorithm \ref{algorithm-1}.

\begin{algorithm}
	\caption{{\bf{The F2A2 Algorithm Framework}}.}\label{algorithm-1}
	\begin{algorithmic}[1]
		\STATE Initialize independent parameters $\mathbf{w}_{in}$, shared parameters $\tilde{\mathbf{w}}_{sh}^i$.
		Set consensus parameters $\tilde{\mathbf{w}}_s$, dual parameters $\{\boldsymbol{\lambda}_i\}$ and the unified penalty parameter $\beta$ to zero.
		For each agent $i$, initialize all other policies estimation parameters $\psi_{j,in}^i$ for $j=1,\cdots,i-1,i+1,\cdots,n$.
		\FOR {each episode}
		\FOR {$t=1$ to \textit{the pre-defined max length of the episode}}
		\STATE Each agent observes initial observation $o_i^t$;
		\STATE For agent $i$, select action $a_i^t$ by current policy $\pi_i$;
		\STATE Execute $\boldsymbol{a}^t=(a_1^t, \cdots, a_N^t)$ and get the cost $c_i^{t+1}$ and next observation $o^{t+1}_i$; 
		\STATE Store $(\boldsymbol{o}^t, \boldsymbol{a}^t, \boldsymbol{c}^{t+1}, \boldsymbol{o}^{t+1})$ to replay buffer $\mathcal{D}$;
		\FOR {each agent $i$}
		\STATE Sample a batch tuple $\{(\boldsymbol{o}^k, \boldsymbol{a}^k, \boldsymbol{c}^{k+1}, \boldsymbol{o}^{k+1})\}$ from replay buffer $\mathcal{D}$;
		\STATE For each tuple sample the next action $\boldsymbol{a}^{k+1}$ according to the estimated policies $\tilde{\boldsymbol{\pi}}$ of all other agents;
		\STATE {\bf{Critic-step}}: Update independent critic parameters and consensus critic parameters $\left\{ \phi_{in}^i,\tilde{\phi}_{sh}^i \right\}$:
        $$
        \left[\!\!\begin{array}{c} \phi_{in}^i\\ \tilde{\phi}_{sh}^i \end{array} \!\! \right] \leftarrow \left[ \!\! \begin{array}{c} \phi_{in}^i \\ \tilde{\phi}_{sh}^i \end{array}\!\!\right] - \textcolor{black}{\beta_{\phi}}\left[ \!\! \begin{array}{c} \nabla_{\phi_{in}^i} \mathcal{L} \left( \mathbf{w}_{in}, \left\{\tilde{\mathbf{w}}_{sh}^i\right\}, \mathbf{w}_{sh}, \left\{\boldsymbol{\lambda}_i \right\} \right) \\ \nabla_{\tilde{\phi}_{sh}^i} \mathcal{L} \left( \mathbf{w}_{in}, \left\{\tilde{\mathbf{w}}_{sh}^i\right\}, \mathbf{w}_{sh}, \left\{\boldsymbol{\lambda}_i \right\} \right) \end{array} \!\! \right],
        $$
		and use the formula in Prop.~\ref{the:ddpg} or its variants to calculate the sampled off-policy policy gradient and $\alpha$ is the step-size of gradient descent method.
        \STATE {\bf{Actor-step}}: Update independent actor parameters and consensus actor parameters $\left\{ \psi_{in}^i,\tilde{\psi}_{sh}^i \right\}$ by minimizing the loss:
        $$
        \left[\!\!\begin{array}{c} \psi_{in}^i\\ \tilde{\psi}_{sh}^i \end{array} \!\! \right] \leftarrow \left[ \!\! \begin{array}{c} \psi_{in}^i\\ \tilde{\psi}_{sh}^i \end{array}\!\!\right] - \textcolor{black}{\beta_{\psi}}\left[ \!\! \begin{array}{c} \nabla_{\psi_{in}^i} \mathcal{L} \left( \mathbf{w}_{in}, \left\{\tilde{\mathbf{w}}_{sh}^i\right\}, \mathbf{w}_{sh}, \left\{\boldsymbol{\lambda}_i \right\} \right) \\ \nabla_{\tilde{\psi}_{sh}^i} \mathcal{L} \left( \mathbf{w}_{in}, \left\{\tilde{\mathbf{w}}_{sh}^i\right\}, \mathbf{w}_{sh}, \left\{\boldsymbol{\lambda}_i \right\} \right)
\end{array} \!\! \right],
        $$
		where use the recent obtained $\left\{ \phi_{in}^i,\tilde{\phi}_{sh}^i \right\}$ into critic-step;
		and use the formula in Prop.~\ref{the:ddpg} or its variants to calculate the sampled off-policy policy gradient;
		and $\alpha$ is the step-size of gradient descent method.
		\STATE {\bf{Consensus-step}}: Eq.~\ref{eq:algorithm-scheme} can be implemented in a decentralized manner. Specifically, agent $i$ broadcast the recent obtained consensus parameter $\tilde{\phi}_{sh}^i, \tilde{\psi}_{sh}^i$ and receive all others recent obtained consensus parameter $\tilde{\phi}_{sh}^j, \tilde{\psi}_{sh}^j$ for all $j \neq i$. Then update shared parameter $\mathbf{w}_{sh}$:
	    $$ \left[\!\!\begin{array}{c} \psi_{sh}\\ \phi_{sh} \end{array} \!\! \right] \leftarrow \left[ \!\! \begin{array}{c} (1/n)\sum_{i=1}^n \tilde{\psi}_{sh}^i \\ (1/n)\sum_{i=1}^n \tilde{\phi}_{sh}^i \end{array}\!\!\right],$$
		\STATE and update the dual multiplier parameters $\left\{ \boldsymbol{\lambda}_i \right\}$:
		$
		\boldsymbol{\lambda}_i \leftarrow \boldsymbol{\lambda}_i - \beta \left( \mathbf{w}_{sh} - \tilde{\mathbf{w}}_{sh}^i \right)
		$
		\IF{$\text{t}\;\textbf{mod}\;\textit{the update frequency of the policy estimation model} == 0$} 
		    \STATE Sample a batch tuple $\{(\boldsymbol{o}^k, \boldsymbol{a}^k, \boldsymbol{c}^{k+1}, \boldsymbol{o}^{k+1})\}$ from replay buffer $\mathcal{D}$;
		    \STATE Update the estimated policies $\tilde{\boldsymbol{\pi}}$ of all others by supervised learning.
		\ENDIF
		\ENDFOR
		\ENDFOR
		\ENDFOR
	\end{algorithmic}
	\normalsize
\end{algorithm}
Noting that if no shared parameters is used in formulation (\ref{eq:framework-v3}) and (\ref{eq:framework-v4}), then the whole problem is degenerated to a simplified version, i.e.,
\begin{equation}
\min_{\mathbf{w}_{in}} \mathcal{J}(\mathbf{w}_{in}) =  \alpha_1 {\cal{J}}_{\text{actor}}(\mathbf{w}_{in}) + \alpha_2 {\cal{J}}_{\text{critic}}(\mathbf{w}_{in}) + \mathcal{R}_{in} (\mathbf{w}_{in}).\label{eq:framework-v4-simple}
\end{equation}
The scheme (\ref{eq:algorithm-scheme}) with be simplified into
\begin{subequations}\label{eq:algorithm-scheme-simple}
\begin{align}
\phi_{in}^i &\leftarrow \phi_{in}^i - \nabla_{\phi_{in}^i} \mathcal{J}(\mathbf{w}_{in}),\quad i=1,\cdots,n; \label{eq:algorithm-scheme-phi-simple} \\
\psi_{in}^i &\leftarrow \psi_{in}^i - \nabla_{\psi_{in}^i} \mathcal{J}(\mathbf{w}_{in}),\quad i=1,\cdots,n. \label{eq:algorithm-scheme-psi-simple}
\end{align}
\end{subequations}
(\ref{eq:algorithm-scheme-simple}) can be considered to minimize (\ref{eq:framework-v4-simple}) concerning  the critic parameters $\left\{\phi_{in}^i\right\}$ and actor parameters $\left\{ \psi_{in}^i \right\}$ alternatively.
This basic scheme is different from traditional multi-agent actor-critic algorithms (e.g., \ref{eq:algorithm-actor-critic}), while the BCGD-type scheme is employed on the jointly MARL framework (\ref{eq:framework-v4-simple}).
%

\textcolor{black}{
\paragraph{Remark 2.} Although some previous value-based multi-agent reinforcement learning algorithms~\citep{Lauer2000AnAF,Matignon2007HystereticQ,Panait2008TheoreticalAO,Arslan2017DecentralizedQF} do not pass any messages between agents during the learning process and belong to pure decentralization, it is still reasonable to name the proposed framework as ``fully decentralized" similar with \citet{zhang2018fully}.
These messages transmit between agents in the learning procedure are not uniformly collected and distributed by a centralized controller, but each agent sends and receives them individually.
\paragraph{Remark 3.} There are two key formulations in this paper, i.e., Equation~\ref{eq:multi-actor-critic-bi} and \ref{eq:multi-actor-critic}.
As discussed above, problem (\ref{eq:multi-actor-critic-bi}) is the core problem that actor-critic type methods aim to solve, and problem~(\ref{eq:multi-actor-critic}) can be considered as an approximation version of problem~(\ref{eq:multi-actor-critic-bi}).
Problem~(\ref{eq:multi-actor-critic}) has the separable structure which motivated us to design the proposed Algorithm 1.
The optimal solutions of problem~(\ref{eq:multi-actor-critic-bi}) and (\ref{eq:multi-actor-critic}) seems to be different, however the optimal solution set of problem~(\ref{eq:multi-actor-critic-bi}) is more difficult to guarantee because of its bi-level programming structure.
Very few works have discussed the theoretical analysis of the RL algorithm to solve the bi-level formulation, while nearly all of them focused on the single-agent case.
For instance, \citet{Yang2018ConvergentRL}, \citet{Yang2019ProvablyGC} and \citet{Hong2020ATF} consider the single-agent actor-critic algorithm as a specific solution to the corresponding bi-level problem, and give some convergence results on linear function approximation case.
As for the multi-agent case, \citet{zhang2018fully} proposes two decentralized actor-critic algorithms with function approximation, and convergence analyses of the algorithms are provided when the value functions are approximated within the class of linear functions.
\citet{zhang2018fully} still could not guarantee the convergence to the stationary point of the bi-level programming problem~(\ref{eq:multi-actor-critic-bi}).
For the stability of the training, a regularization term is added to the problem~(\ref{eq:multi-actor-critic}) to make the problem into the more general problem~(\ref{eq:framework-v3}), which can be equivalent converted into problem~(\ref{eq:framework-v4}).
In this paper, some convergence results of the proposed Algorithm 1 can be further established.
Expressly, Algorithm 1 can be incorporated into the algorithm framework of~\citet{hong2016convergence} for solving the problem (4) (5.1 in the modified version).
If Algorithm 1 can satisfy \cite[Assumption A]{hong2016convergence} then theoretical results similar to can \cite[Theorem 2.4]{hong2016convergence} be obtained, i.e., \textit{Any limit point $\{ \{\phi^{i,*}_{in}\},\{\psi^{i,*}_{in}\},\{\phi^{*}_{sh}\},\{\psi^{*}_{sh}\}, \{\boldsymbol{\lambda}^{i,*}\} \}$ of the sequence $\{ \{\phi^{i,k}_{in}\},\{\psi^{i,k}_{in}\},\{\phi^{k}_{sh}\},\{\psi^{k}_{sh}\}, \{\boldsymbol{\lambda}^{i,k}\} \}$ which is obtained from Algorithm 1 is a stationary point of problem~(\ref{eq:framework-v4}).
Because problem~(\ref{eq:framework-v4}) is equivalent with problem~(\ref{eq:framework-v3}), we have that $\{ \{\phi^{i,*}_{in}\},\{\psi^{i,*}_{in}\},\{\phi^{*}_{sh}\},\{\psi^{*}_{sh}\}, \{\boldsymbol{\lambda}^{i,*}\} \}$ is a stationary point of problem~(\ref{eq:framework-v3}).}
The policy or value approximation functions are not limited to the linear case, and the local convergence of the obtained sequence for general cases in Algorithm 1 can be proved.
To emphasize, the convergence to the stationary point of the bi-level programming formulation problem~(\ref{eq:multi-actor-critic-bi}) could not be guaranteed.
}

\subsection{Instantiation Algorithms}
In the proposed PDHG algorithm framework, the gradient of the augmented Lagrangian function $\mathcal{L}$ concerning the primal variables (see Eq.~\ref{eq:algorithm-scheme-phi} and \ref{eq:algorithm-scheme-psi}) need to be calculated to optimize the actors and critics of each agent.
The proposed optimization objective functions are different from traditional centralized and decentralized MARL algorithms. 
Therefore, the related results in these algorithms cannot be directly used.
This section will give the detailed form of the gradient of the Lagrangian function $\mathcal{L}$.
In Section 4, it can be seen that the proposed algorithm framework has good flexibility, so the proposed framework could introduce various single-agent reinforcement learning algorithms as the backbone.

Specifically, one on-policy algorithm, i.e. COMA~\citep{foerster2018counterfactual}, and three state-of-the-art single-agent off-policy actor-critic algorithms, i.e. DDPG~\citep{lillicrap2015continuous}, TD3~\citep{fujimoto2018addressing} and SAC~\citep{haarnoja2018soft} are incorporated into the proposed decentralized framework.

\subsubsection{On-policy F2A2 Instantiation Algorithms}

For the on-policy methods, the current state-of-the-art algorithm COMA is chosen as the backbone algorithm.
The corresponding proposition for the COMA algorithm is proposed to calculate the on-policy joint gradient.

\begin{tcolorbox}
\begin{prop}
        [On-Policy COMA-Based Joint Gradient]\label{the:coma}
        \textcolor{black}{$d_{\boldsymbol{\pi}}$ represents the distribution of the state-occupancy measure of policy $\boldsymbol{\pi}$}, and $\delta$ is the TD(0)-error.
        The counterfactual baseline $\mathcal{B}(\boldsymbol{o}, \boldsymbol{a}^{\setminus i})$ also be introduced from COMA.
        So the gradient of $J_{ac}^i \left( \mathbf{w}_{in}, \tilde{\mathbf{w}}_{sh}^i \right)$ is
    	\begin{equation}
    	    \begin{aligned}
    	        &\nabla_{\psi_{in}^i} J_{ac}^i \left( \mathbf{w}_{in}, \tilde{\mathbf{w}}_{sh}^i \right) = \\
    	        &\quad\quad\mathbf{E}_{s \sim d_{\boldsymbol{\pi}}, \boldsymbol{o} \sim \mathcal{E},\boldsymbol{a}\sim\boldsymbol{\pi}}\Bigg[\alpha_1\nabla_{\psi_{in}^i}\log\pi_{\psi_{in}^i}^i(a^i|o^i)\left(Q^{\boldsymbol{\pi},i}_{\tilde{\phi}^i_{sh}}(\boldsymbol{o},\boldsymbol{a}) - \mathcal{B}(\boldsymbol{o}, \boldsymbol{a}^{\setminus i})+\frac{\alpha_2}{\alpha_1}\delta^2\right)\Bigg], 
    	        \\
    	        &\nabla_{\tilde{\phi}^i_{sh}} J_{ac}^i \left( \mathbf{w}_{in}, \tilde{\mathbf{w}}_{sh}^i \right) 
    	        =\mathbf{E}_{\substack{s\sim d_{\boldsymbol{\pi}}, \boldsymbol{o}\sim\mathcal{E},\boldsymbol{a}\sim\boldsymbol{\pi}}}\left[ (\alpha_1+2\alpha_2\delta)\nabla_{\tilde{\phi}^i_{sh}}Q^{\boldsymbol{\pi},i}_{\tilde{\phi}^i_{sh}}(\boldsymbol{o}, \boldsymbol{a}) \right]. \nonumber
    	    \end{aligned}
    	\end{equation}
\end{prop}
\end{tcolorbox}

For other single-agent on-policy actor-critic algorithms need to be incorporated, we just need to replace the $J^i_{\text{actor}}$ and $J^i_{\text{critic}}$ part in $J_{ac}^i \left( \mathbf{w}_{in}, \tilde{\mathbf{w}}_{sh}^i \right)$ with the corresponding form and derive their corresponding on-policy joint gradients. 
The specific forms of F2A2-COMA and the proof of Proposition \ref{the:coma} are given in supplemental material.
With the gradient calculated above, below, the F2A2-COMA algorithm is formally proposed.

\paragraph{F2A2-COMA.} COMA (Counterfactual Multi-Agent Policy Gradient) method learns a centralized critic with a \textit{counterfactual baseline} which is inspired by \textit{difference rewards} to solve the \textit{multi-agent credit assignment} problem.
The COMA algorithm is introduced into the F2A2 framework, and the F2A2-COMA algorithm is then proposed accordingly.
F2A2-COMA focuses on settings with discrete actions but can be easily extended to continuous action spaces by estimating counterfactual baseline with Monte Carlo samples or using functional forms that render it analytical, e.g., Gaussian policies and critic.
In the F2A2-COMA algorithm, all agents share a centralized counterfactual baseline function, but the policy functions are independent of each other.
All the other settings are the same as the COMA algorithm.

\subsubsection{Off-policy F2A2 Instantiation Algorithms}

The following proposition for off-policy methods is firstly proposed since the above single-agent off-policy algorithms are either based on or related to the DDPG algorithm.

\begin{tcolorbox}
{\color{black}{
\begin{prop}
        [Off-Policy DDPG-Based Joint Gradient]\label{the:ddpg}
        $\boldsymbol{\pi}_0$ represents the data collection policy sampled from experience replay buffer, \textcolor{black}{$d_0$ represents the distribution of the state-occupancy measure of policy $\boldsymbol{\pi}_0$}, and $\delta$ is the TD(0)-error.
        $\epsilon,\alpha_1,\alpha_2$ are hyperparameters.
        $\psi_i:=\{\psi_{in}^i,{\tilde{\psi}}_{sh}^i\}$, $\phi_i:=\{\phi_{in}^i,{\tilde{\phi}}_{sh}^i\}$,
        the gradients of $J_{ac}^i \left( \mathbf{w}_{in}, \tilde{\mathbf{w}}_{sh}^i \right)$ w.r.t. $\{\psi_{in}^i,{\tilde{\psi}}_{sh}^i\}$ and $\{\phi_{in}^i,{\tilde{\phi}}_{sh}^i\}$ are
    	\begin{equation}
    		\begin{aligned}
    		    &\nabla_{\psi_{in}^i} J_{ac}^i \left( \mathbf{w}_{in}, \tilde{\mathbf{w}}_{sh}^i \right) \\
    	       & = 
    	        \mathbf{E}_{s \sim d_0, \boldsymbol{o} \sim \mathcal{E},\boldsymbol{a}\sim\boldsymbol{\pi}}\Bigg[(\alpha_1+2\alpha_2\delta \left(\frac{\pi^i_{0}(a^i|o^i)}{\pi^i_{\psi_{in}^i}(a^i|o^i)}\right)\nabla_{\psi_{in}^i}\pi^i_{\psi_{in}^i}(a^i|o^i)\nabla_{a^i} Q^{\boldsymbol{\pi},i}_{\tilde{\phi}^i_{sh}}(\boldsymbol{o},\boldsymbol{a})\Bigg],\\
    	        &\nabla_{\tilde{\phi}^i_{sh}} J_{ac}^i \left( \mathbf{w}_{in}, \tilde{\mathbf{w}}_{sh}^i \right) \\
	           &= \mathbf{E}_{\substack{s\sim d_{0}, \boldsymbol{o}\sim\mathcal{E},\boldsymbol{a}\sim\boldsymbol{\pi_0}}}\left[ \left(\alpha_1 \left(\frac{\pi^i_{\psi_{in}^i}(a^i|o^i)}{\pi^i_{0}(a^i|o^i)}\right)+2\alpha_2\delta\right)\nabla_{\tilde{\phi}^i_{sh}}Q^{\boldsymbol{\pi},i}_{\tilde{\phi}^i_{sh}}(\boldsymbol{o}, \boldsymbol{a}) \right]+\\
	            &\qquad\qquad\mathbf{E}_{\substack{s\sim d_{0}, \boldsymbol{o}\sim\mathcal{E},\boldsymbol{a}\sim\boldsymbol{\pi}}}\left[ \left(\alpha_1+2\alpha_2\delta \left(\frac{\pi^i_{0}(a^i|o^i)}{\pi^i_{\psi_{in}^i}(a^i|o^i)}\right)\right)\nabla_{\tilde{\phi}^i_{sh}}Q^{\boldsymbol{\pi},i}_{\tilde{\phi}^i_{sh}}(\boldsymbol{o}, \boldsymbol{a}) \right].
	            \nonumber
    		\end{aligned}
    	\end{equation}
    \end{prop}
}}
\end{tcolorbox}

The $J^i_{\text{actor}}$ and $J^i_{\text{critic}}$ part in $J_{ac}^i \left( \mathbf{w}_{in}, \tilde{\mathbf{w}}_{sh}^i \right)$ can be replaced with the corresponding form of the single-agent off-policy actor-critic algorithm need to be incorporated, and their corresponding joint gradients can be derived accordingly. 
The specific forms of F2A2-DDPG, F2A2-TD3, and F2A2-SAC are shown in Figure~\ref{fig:f2a2-eq} in the supplementary material.
The proof of Proposition \ref{the:ddpg} and its variants (of F2A2-TD3 and F2A2-SAC) are given in supplemental material.
With the gradient calculated above, below, three off-policy F2A2 instantiation algorithms are proposed formally.

\paragraph{F2A2-DDPG.}
The DDPG algorithm is introduced into the F2A2 framework, and the F2A2-DDPG algorithm is proposed accordingly. 
In F2A2-DDPG, all agents share a centralized value function, but the policy functions are independent. 
At the same time, considering that the simulated environments in this paper are all designed with discrete action space, DDPG cannot directly deal with the above situation. 
So this paper learns from the ideas of~\citet{lowe2017multi}, rather than using policies that deterministically output an action, policies that produce differentiable samples through a Gumbel-Softmax distribution~\citep{jang2017categorical} is then used.

\paragraph{F2A2-TD3.}
While DDPG can achieve excellent performance sometimes, it is frequently brittle concerning hyperparameters and other kinds of tuning. For example, a standard failure mode for DDPG is that the learned Q-function begins to dramatically overestimate Q-values, leading to policy breaking because it exploits the errors in the Q-function. 
TD3 (Twin Delayed DDPG) is an algorithm that addresses this issue by introducing three critical tricks, clipped double Q-learning, delayed policy updates, and target policy smoothing.
In the F2A2-TD3 algorithm, all agent's two centralized Q-value functions share each other, and the policy function is independent.
Each agent has a different policy update frequency and policy smoothing noise for better exploration.
Since TD3 cannot be applied to discrete action space, the same approach as F2A2-DDPG is adopted to modify the TD3 algorithm.

\paragraph{F2A2-SAC.}
SAC (Soft Actor Critic) is an algorithm that optimizes a stochastic policy in an off-policy way, forming a bridge between stochastic policy optimization and DDPG-style approaches. 
It is not a direct successor to TD3 (having been published roughly concurrently). 
However, it incorporates the clipped double-Q trick, and due to the inherent stochasticity of the policy in SAC, it also winds up benefiting from something like target policy smoothing.
Same as F2A2-TD3, all agents have the same centralized Q-value function and independent policy function.

\textcolor{black}{
\paragraph{Remark 4.}
In the practical implementation, we use the truncated importance sampling ratio inspired by~\citet{munos2016safe} for all off-policy instantiation algorithms (F2A2-DDPG, F2A2-TD3, F2A2-SAC) to stabilize the training process, i.e.,
$$
\left(\frac{\pi^i_{\psi_{in}^i}(a^i|o^i)}{\pi^i_{0}(a^i|o^i)}\right) \rightarrow \min\left(\epsilon, \frac{\pi^i_{\psi_{in}^i}(a^i|o^i)}{\pi^i_{0}(a^i|o^i)}\right), \text{and }\left(\frac{\pi^i_{0}(a^i|o^i)}{\pi^i_{\psi_{in}^i}(a^i|o^i)}\right) \rightarrow \min\left(\epsilon, \frac{\pi^i_{0}(a^i|o^i)}{\pi^i_{\psi_{in}^i}(a^i|o^i)}\right),
$$
and we set $\epsilon=1$ in all experiments.
Compared to original importance sampling ratio, it does not suffer from the variance explosion of the product of importance sampling ratios.
Truncated importance sampling ratio has other theoretical advantages and interested readers can refer to the original paper~\citep{munos2016safe}.
}

\subsection{Modeling Other Agents}
Existing decentralized reinforcement learning works \citep{zhang2018fully,doan2019convergence,suttle2019multi} assume that each agent \textcolor{black}{can observe others' actions.}
While such an assumption is too strict, and the more realistic assumption is that the agent has to \textcolor{black}{model other agents' policies and predict other agents' actions} based on their historical observations.
Some further improvements are involved into the proposed algorithm framework to make the proposed framework still works better under such a more general assumption.
Specifically, in this more general assumption the symbol $\boldsymbol{\pi}$, in Proposition~\ref{the:ddpg},~\ref{the:coma} and their variants, are changed to $\tilde{\boldsymbol{\pi}}$ and the symbol $\boldsymbol{a}$ is changed to $\tilde{\boldsymbol{a}}:=\tilde{\boldsymbol{\pi}}(\boldsymbol{o})$.
Another important issue is that the approximate policy $\tilde{\boldsymbol{\pi}}$ and action $\tilde{\boldsymbol{a}}$ may lead to distorted gradients in Algorithm \ref{algorithm-1}.
However, gradients with noise may help optimization algorithms converge to global optimal solution~\citep{jastrzkebski2017three,neelakantan2015adding,smith2019super}.
Although the estimated policy might bring bias, it could reduce the algorithm's variance, which further improves the robustness of the system. 
This will be demonstrated later in the experiments.

To better estimate the policies of other agents in complex multi-agent environments, a novel modeling other agents (MOA) approach is devised based on theory-of-mind (TOM) inspired by \citet{rabinowitz2018machine}, and similar approach have also been found in other works~\citep{jaques2019social}.
The method is further connected with online supervised learning.
The algorithm architecture is shown in Figure~\ref{fig:moa}. 
Specifically, each agent first randomly samples a fixed number of trajectories in the environment before training and encode each trajectory using an LSTM. 
Then, the average of all the trajectories encoding is taken as the \textit{character} of the agent. 
After training starts, for each current unfinished trajectory, another LSTM is used to encode its historical segment as the current \textit{mental} of the agent. 
On the one hand, the current state, the character, and the current mental of the agent are used together as input to the \textit{natural prediction network}; 
on the other hand, the \textit{impromptu prediction network} is proposed to predict the action of agents only relying on the current state. 
The outputs of the two networks are combined to produce the final predictions. 
Note that the entire network's training process is performed together with the reinforcement learning algorithm, and the training data is periodically collected online.
\begin{figure}[tb!]
    \centering
    \includegraphics[width=0.8\textwidth]{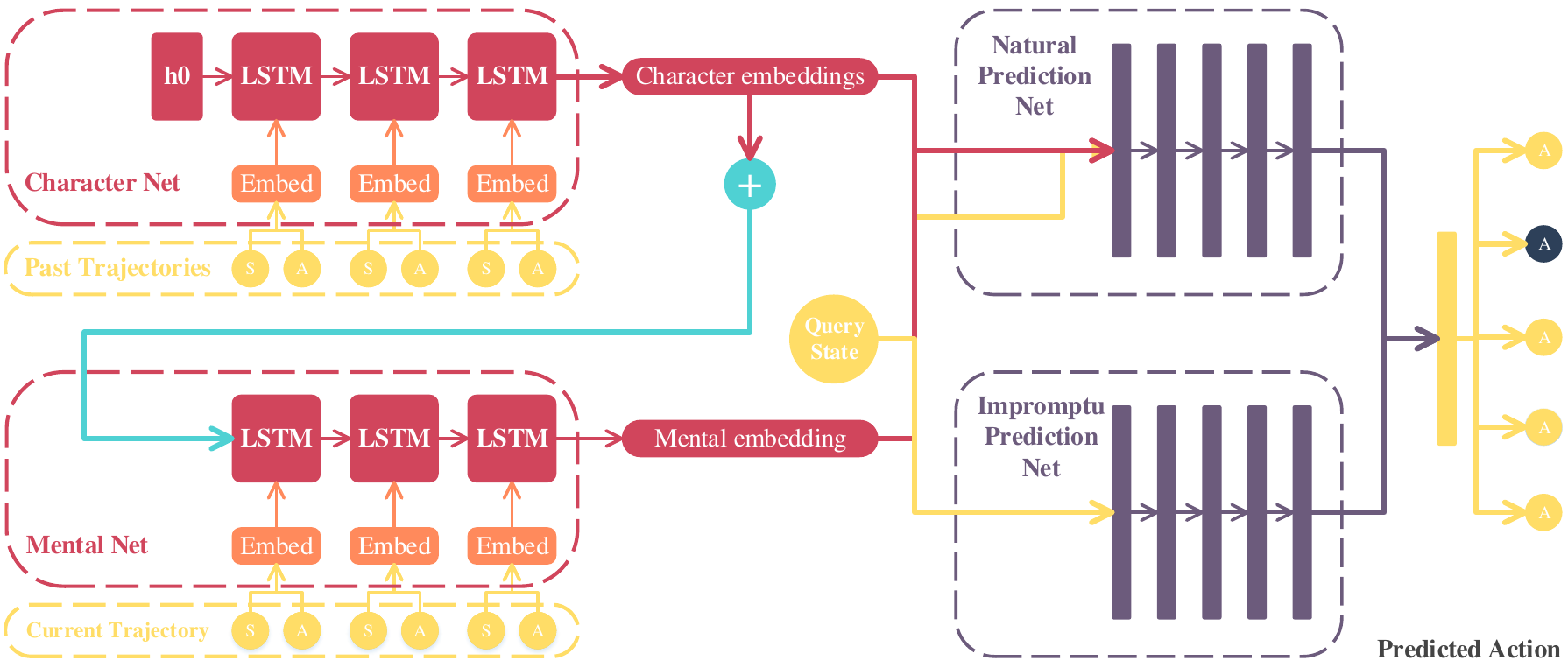}
    \vspace{-10pt}
    \caption{Modeling other agents' policies. 
    Character and Mental networks are used to encode the historical and instant information of the agent.
    Natural Prediction and Impromptu Prediction networks obtain the final prediction via online supervised learning with previous outputs and the current state.}
    \label{fig:moa}
\end{figure}

Finally, as detailed in Figure~\ref{fig:alg}, we propose the fully decentralized algorithm framework, which follows the primal-dual hybrid gradient scheme and simultaneously splits joint tasks with the divide-and-conquer strategy.
Algorithm~\ref{algorithm-1} is the corresponding pseudocode.

\begin{figure}[tb!]
    \centering
    \includegraphics[width=0.7\textwidth]{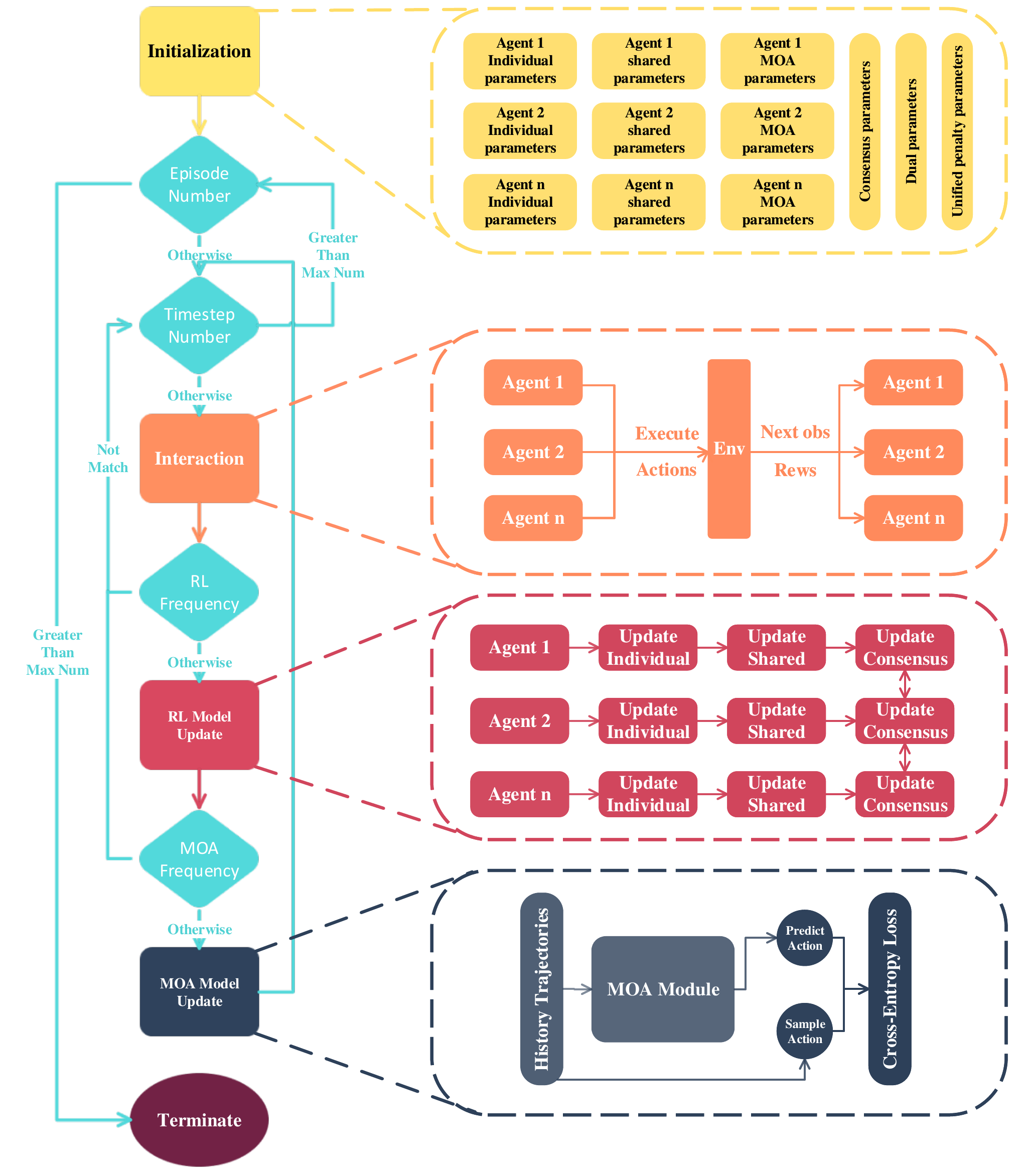}
    \vspace{-10pt}
    \caption{Flowchart corresponding to Algorithm~\ref{algorithm-1}.}
    \label{fig:alg}
\end{figure}

\subsection{Information Setting Comparison}

Finally, to clearly show the differences among the proposed F2A2, the MARL algorithms in the CTDE framework, and existing decentralized MARL algorithms, the information settings of the above algorithms is summarized and the details are shown in Table~\ref{tab:information-structures}.

\begin{table}[tbp!]
\centering
\resizebox{\textwidth}{!}{%
\begin{tabular}{|c|c|c|c|c|c|c|}
\hline
Schemes & Algorithms & Joint Observation & Joint Policy & Joint Cost & Shared Actor & Shared Critic \\ \hline
\multirow{3}{*}{\begin{tabular}[c]{@{}c@{}}Centralized\\ Algorithms\end{tabular}} & MADDPG~\citep{lowe2017multi} & Accessible & Accessible & - & - & Accessible \\ \cline{2-7} 
 & COMA~\citep{foerster2018counterfactual} & Accessible & Accessible & Accessible & - & Accessible \\ \cline{2-7} 
 & MAAC~\citep{iqbal2019actor} & Accessible & Accessible & - & - & Accessible \\ \hline
\multirow{3}{*}{\begin{tabular}[c]{@{}c@{}}Decentralized\\ Algorithms\end{tabular}} & MA-AC~\citep{zhang2019distributed} & Accessible & Communication & - & Communication & - \\ \cline{2-7} 
 & MA-AC~\citep{zhang2018fully} & Accessible & Communication & - & - & Communication \\ \cline{2-7} 
 & \textbf{F2A2} & \textbf{Accessible} & \textbf{Estimation} & \textbf{-} & \textbf{Communication} & \textbf{Communication} \\ \hline
\end{tabular}%
}
\caption{The information settings of typical centralized, decentralized algorithms, and F2A2. Dashes indicate that certain elements do not exist.}
\label{tab:information-structures}
\end{table}

For centralized MARL algorithms (i.e., MADDPG, COMA, MAAC), each agent must use joint observations (and the joint actions taken by all agents under the current joint policy) as input when calculates the joint Q-value (and trains the joint Q-value function), due to the existence of the centralized critic. 
Therefore, the joint observations and joint policy must be accessible for each agent during the algorithm training process.
For MADDPG and MAAC, because they solve the general cooperative POSG problem, each agent only needs to optimize its own expected cumulative costs. 
Therefore, each agent does not need to access the joint cost.
However, for the COMA algorithm, it is necessary to access the joint cost because it solves the fully cooperative MARL problem.

Existing decentralized algorithms achieve global access to some centralized modules (e.g., joint policy, centralized critic, centralized actor) by communicating with each other.
However, for the partially observed multi-agent environment, to make the agent's policy have a more robust representation ability, its policy function is often modeled by a recurrent neural network. 
Transmission of such a large amount of parameters through communication will make the algorithm less scalable. 
Therefore, F2A2 estimates the policies of the other agents by introducing the MOA module.

\section{Experiments}
\label{sec:experiment}
The numerical experiments of the proposed framework are conducted from the following three aspects: effectiveness, performance, and scalability.
Specifically, the effectiveness of F2A2 is verified in two general cooperative POSG environments designed in~\citet{iqbal2019actor} first.
Then, F2A2 is combined with the recurrent neural network to verify the performance on the more challenging cooperative Starcraft II unit micromanagement tasks~\citep{samvelyan2019starcraft}.
Finally, the scalability is verified in a large-scale general cooperative POSG environment MAgent\citep{zheng2018magent}.
Table \ref{table:attr} contains the concise introduction for each environments.

\begin{table}[tb!]
\centering
	\begin{tabular}{|c|c|c|c|}
	\hline
 Environment & Scenario & Scale & Agent \# \\
	\hline
	\multirow{2}{*}{\begin{tabular}[c]{@{}c@{}} Cooperative MPE\\\citep{iqbal2019actor}\end{tabular}} & Cooperative Treasure Collection & Small-Scale & 8 \\
	\cline{2-4}
	 & Rover Tower  & Small-Scale & 8\\
	\hline
	\multirow{2}{*}{\begin{tabular}[c]{@{}c@{}} StarCraft II\\Micromanagement\\\citep{samvelyan2019starcraft} \end{tabular}} & Map 2s3z & Small-Scale & 5\\
	\cline{2-4}
	& Map 3m~ & Small-Scale & 3\\
	\cline{2-4}
	& Map 8m~ & Small-Scale & 8\\
	\hline
	MAgent~\citep{zheng2018magent} & Battle & Large-Scale & 256 \\
	\hline
	\end{tabular}
    \caption{Attributes of the experiment environments.}
	\label{table:attr}
\end{table}

In the above environments,
the corresponding centralized baseline method MADDPG, MAAC, and COMA is set for the F2A2-framework instantiation F2A2-DDPG, F2A2-SAC, and F2A2-COMA.
Like MAAC, the counterfactual baseline proposed by COMA is also introduced in F2A2-SAC.
For fairness, the attentional critic is used as same as the MAAC algorithm in F2A2-SAC.
In addition, as far as we know, there is no published centralized MARL algorithm using TD3 as the backbone. 
Therefore, the performance of the F2A2-TD3 is used as an indicator to measure the adaptability of the F2A2 framework to the single-agent algorithm instead of setting a corresponding baseline for the F2A2-TD3 algorithm. 
Specifically, because the TD3 algorithm is generally better than DDPG when the F2A2 framework adopts these single agent algorithms as the backbone, it can keep the relative order of performance after instantiation (that is, F2A2-TD3 should be better than F2A2-DDPG).
For decentralized works, MA-AC~\citep{zhang2018fully} is chosen as the baseline.

\subsection{Cooperative Multi-agent Particle Environments}
The \textit{Multi-agent Particle Environments} (MPE) was first used in~\citet{lowe2017multi}. 
However, the MPE includes cooperative tasks, competitive tasks, and cooperative-competitive-hybrid tasks, and the number of agents is small. 
To this end, the larger-scale collaborative environment based on the MPE environment proposed by~\citet{iqbal2019actor}, which is denoted as \textit{Cooperative Multi-agent Particle Environments (Cooperative MPE)}, is used to more effectively verify the effectiveness of the algorithm.
The two cooperative environments are introduced in \textbf{Appendix} in detail.

\begin{figure*}[tb!]
    \centering
    \subfigure[Cooperative Treasure Collection.]{
    \label{fig:fct}
    \includegraphics[width=0.4\textwidth]{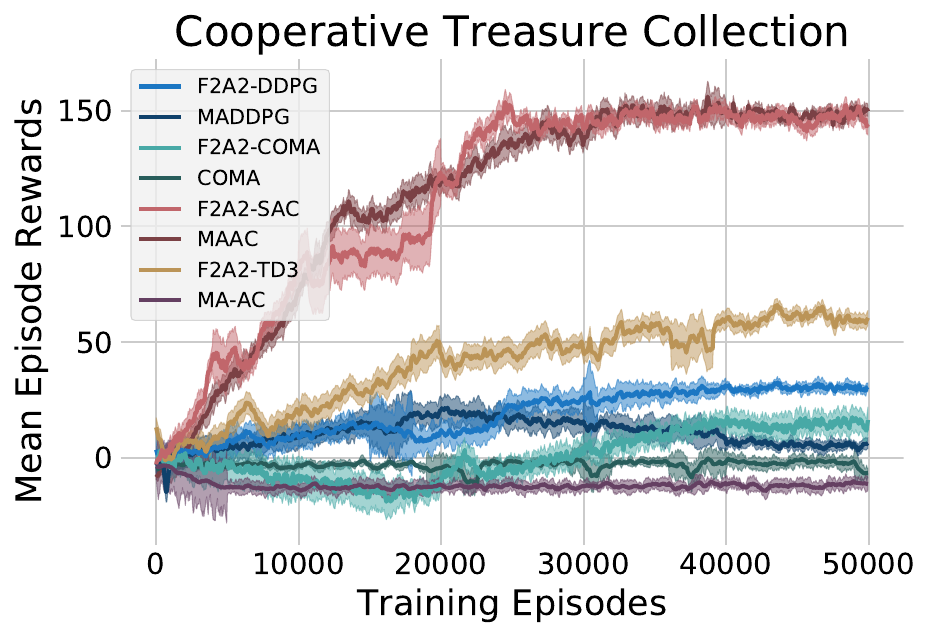}
    }
    \subfigure[Rover Tower.]{
    \label{fig:msl}
    \includegraphics[width=0.4\textwidth]{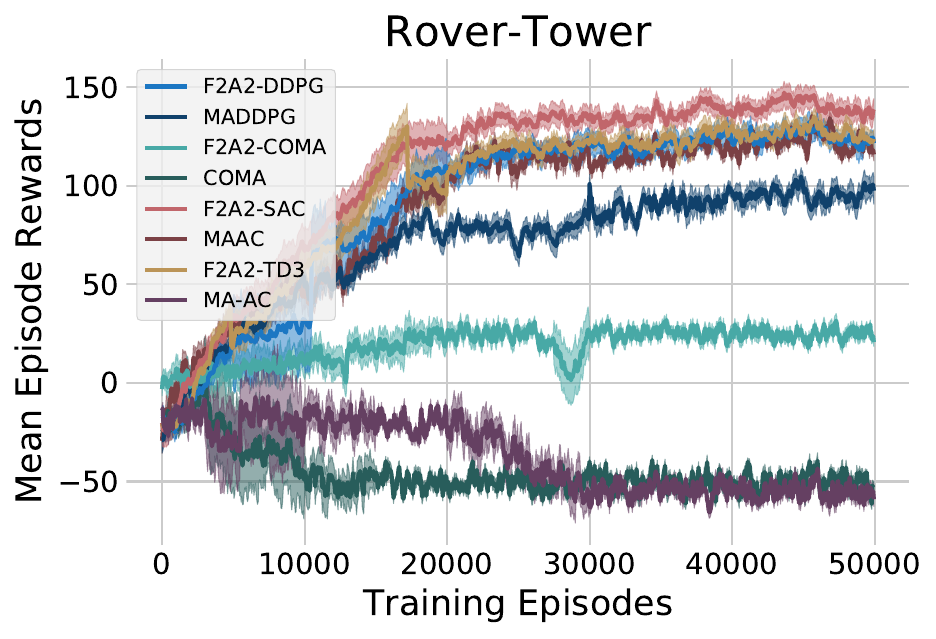}
    }
    \vspace{-10pt}
    \caption{Cooperative Multi-agent Particle Environment. The solid line represents the average under $10$ seeds, and the shaded part represents the 95\% confidence interval.}
\end{figure*}

Figure~\ref{fig:fct} and~\ref{fig:msl} plot F2A2 instantiations' performance relative to their centralized baselines and the decentralized baseline against the training episodes in Cooperative MPE.
It can be seen that the algorithms instantiated by F2A2 can reach or exceed the performance of their corresponding centralized baselines, which shows that the F2A2 framework can effectively make global decisions based on local observations and also reflects the effectiveness of PDHG and MOA. 
At the same time, the performance of the F2A2-TD3 algorithm can exceed the F2A2-DDPG algorithm, which shows that the F2A2 framework can flexibly integrate the single-agent algorithms and guarantee their original advantages.

Specifically, from Figure~\ref{fig:fct} and~\ref{fig:msl}, it can be seen that in the \textit{Rover-Tower} environment, the performance gaps between the F2A2 instantiated algorithms and corresponding centralized baselines are greater than the \textit{Cooperative Treasure Collection} environment. 
We believe this is due to the characteristics of the environments. 
Compared with the \textit{Cooperative Treasure Collection} environment, the interaction between agents in the \textit{Rover-Tower} environment is more sparse. 
Therefore, too much consideration of the behavior of other agents can easily cause the algorithm to overfit other agents and converge to a poor local optimal. 
The MOA module in F2A2 can play a role in regularization, allowing the algorithm to ignore other agents' effects partially.

We mentioned earlier that MA-AC could be seen as a decentralized version of the centralized algorithm COMA. 
From the performance comparison of the two in the figure, it can be seen that if joint training and MOA are lacking, the performance of the decentralized algorithm cannot exceed the performance of the centralized algorithm. 
It should be noted here that, in theory, if the centralized algorithm can find the optimal global solution, then the decentralized algorithm cannot exceed the performance of the centralized algorithm. 
However, the current centralized MARL algorithms can only achieve local optimal.


\subsection{StarCraft II Micromanagement}

This section focuses on the decentralized micromanagement problem in StarCraft II. 
The combat scenarios where two groups of identical units are placed symmetrically on the map are considered. 
The units of the first group, allied,  are controlled by the proposed algorithms. 
The enemy units are controlled by a built-in StarCraft II AI, which makes use of handcrafted heuristics. 
The initial placement of units within the groups varies across episodes. 
The difficulty of the computer AI controlling the enemy units is set to \textit{medium}.
The results on a set of maps where each unit group consists of 3 Marines (3m), 8 Marines (8m), and 2 Stalkers
and 3 Zealots (2s3z) are compared with baselines.

Similar to the work of~\citet{foerster2018counterfactual} and~\citet{rashid2018qmix}, the action space of agents consists of the following set of discrete actions: \textit{move[direction]}, \textit{attack[enemy id]}, \textit{stop}, and \textit{noop}. 
Agents can only move in four directions: north, south, east, or west. 
A unit is allowed to perform the \textit{attack[enemy id]} action only if the enemy is within its shooting range. 
This facilitates the decentralization of the problem.
The introduction of the unit sight range achieves partial observability.
Moreover, agents can only observe others if they are alive and cannot distinguish between units that are dead or out of range.
At each time step, the agents receive a joint cost equal to the total negative damage dealt on the enemy units. 
In addition, agents receive a bonus of $10$ points after killing each opponent and $200$ points after killing all opponents. 
These costs are all normalized to ensure the maximum cumulative cost achievable in an episode is $-20$.

\begin{figure*}[tb!]
    \centering
    \subfigure[StarCraft II 3m map.]{
    \label{fig:sc2_3m}
    \includegraphics[width=0.31\textwidth]{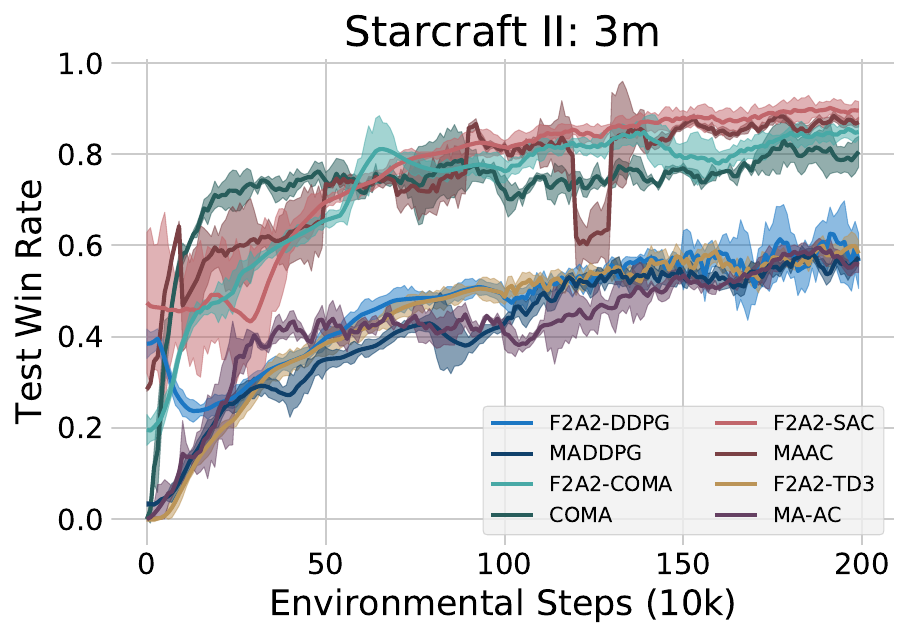}
    }
    \subfigure[StarCraft II 8m map.]{
    \label{fig:sc2_8m}
    \includegraphics[width=0.31\textwidth]{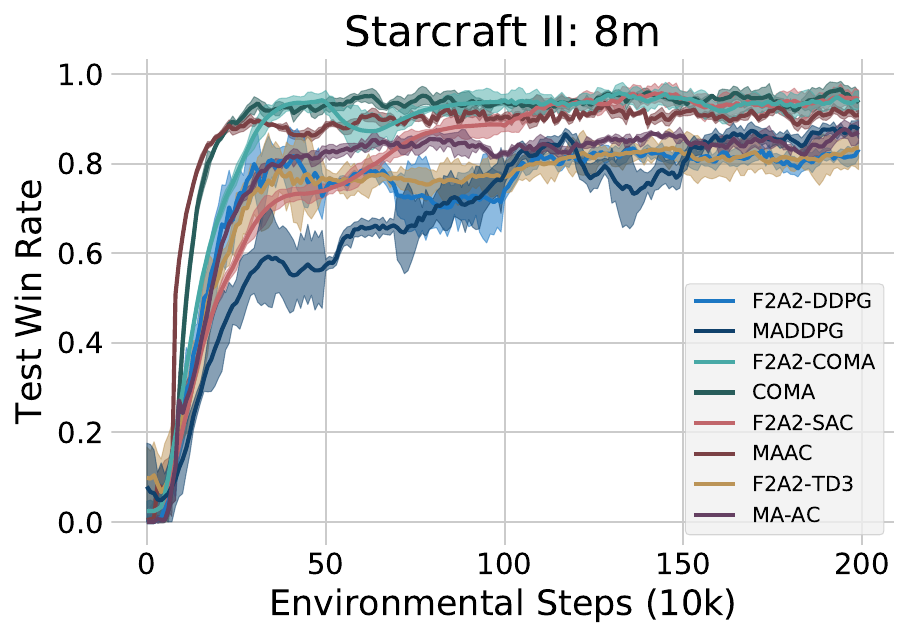}
    }
    \subfigure[StarCraft II 2s3z map.]{
    \label{fig:sc2_2s3z}
    \includegraphics[width=0.31\textwidth]{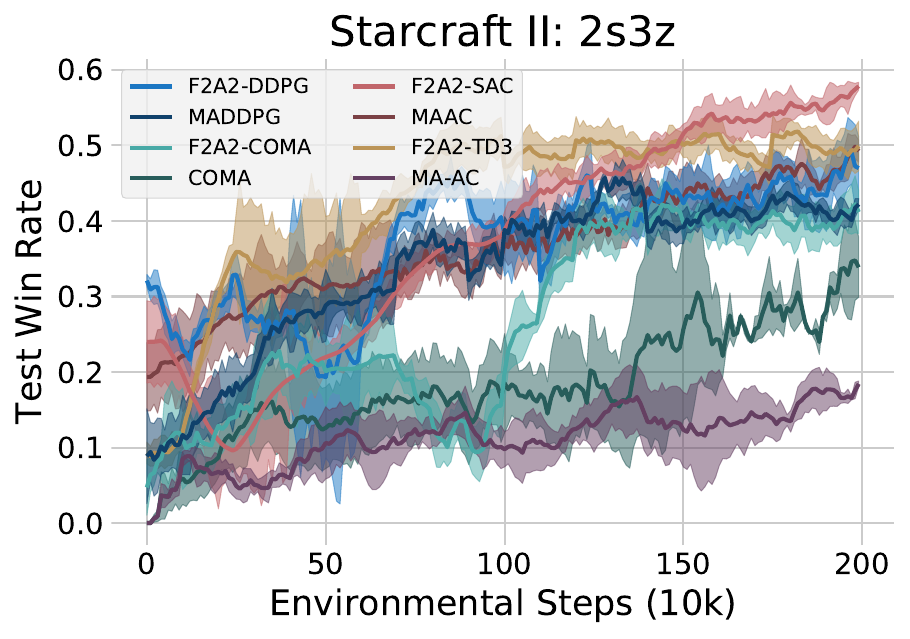}
    }
    \vspace{-10pt}
    \caption{StarCraft II Environment. The solid line represents the average under $10$ random seeds, and the shaded part represents the corresponding $95$\% confidence interval.}
\end{figure*}

Figure~\ref{fig:sc2_3m},~\ref{fig:sc2_8m} and ~\ref{fig:sc2_2s3z} plot F2A2 instantiations' win rate in test environment relative to their centralized baselines against the environmental steps in three StarCraft II micromanagement maps.
It can be seen from the figure that even if the Starcraft II environment is more complicated than Cooperative MPE, the instantiation algorithms of the F2A2 framework can still reach or exceed its corresponding centralized baseline in performance. 
This demonstrates the robustness of the F2A2 framework.
In addition, the performance comparison between MA-AC and COMA shows similar results to the Cooperative MPE.


\begin{figure}[tb!]
    \centering
    \subfigure[Negative costs of different numbers of agents over training epochs.]{
            \includegraphics[width=0.4\linewidth]{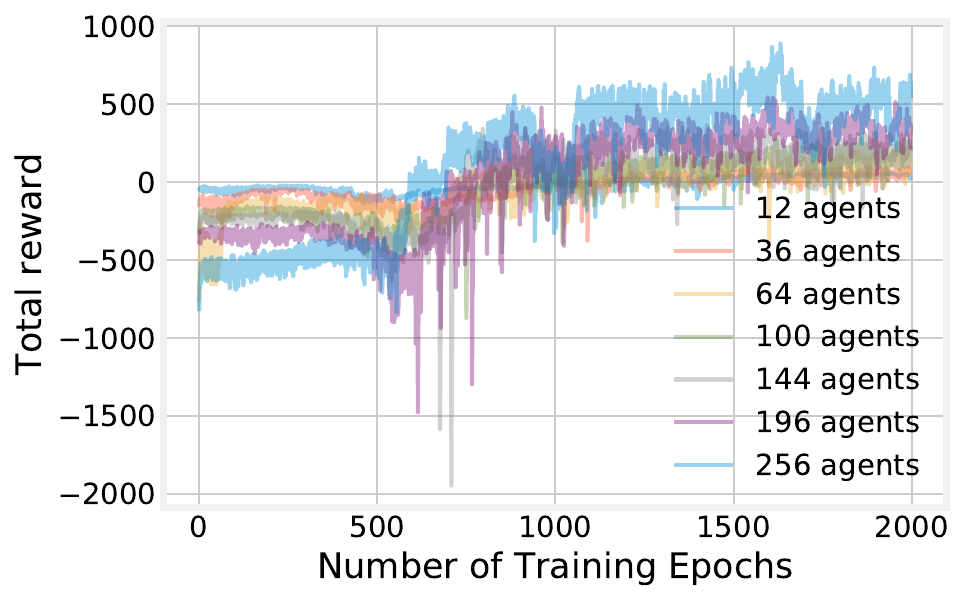}
            \label{fig:magent-rewards}
    }
    \subfigure[Convergence times as the number of agents changes.]{
            \includegraphics[width=0.4\linewidth]{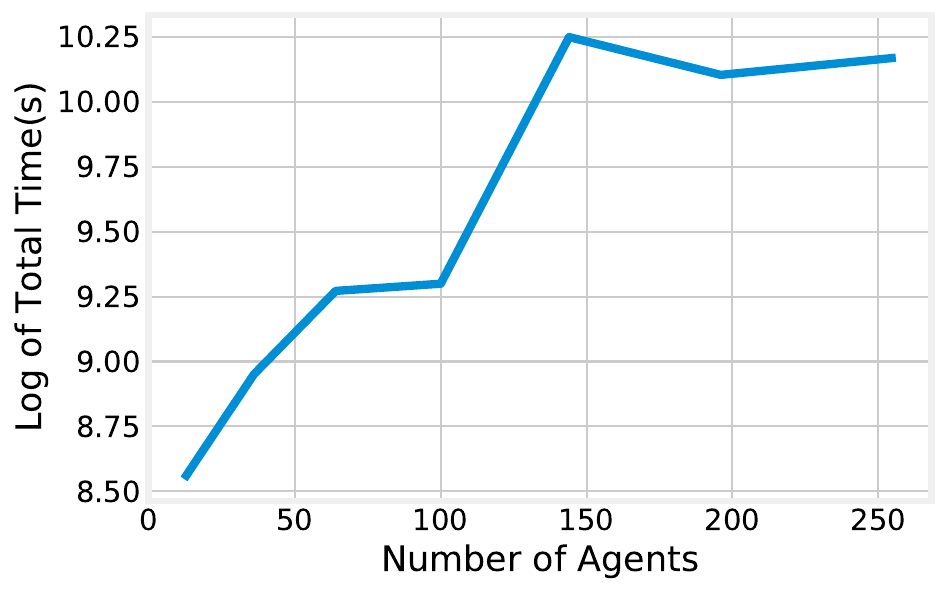}
            \label{fig:magent-times}
    }
    \subfigure[Exploring]{
            \includegraphics[width=0.25\linewidth]{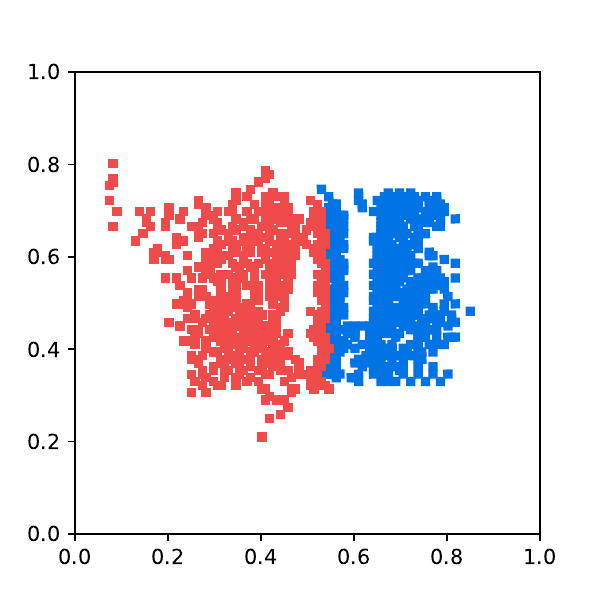}
            \label{fig:magent-explore}
    }
    \subfigure[Pursuiting]{
            \includegraphics[width=0.25\linewidth]{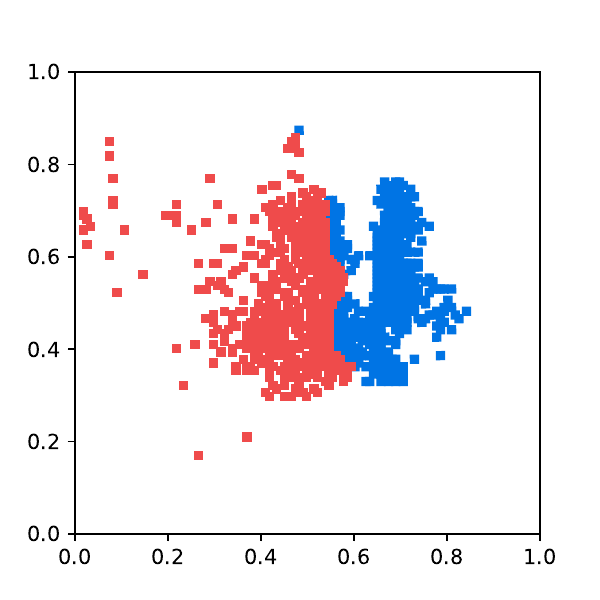}
            \label{fig:magent-pursuit}
    }
    \subfigure[Surrounding]{
            \includegraphics[width=0.25\linewidth]{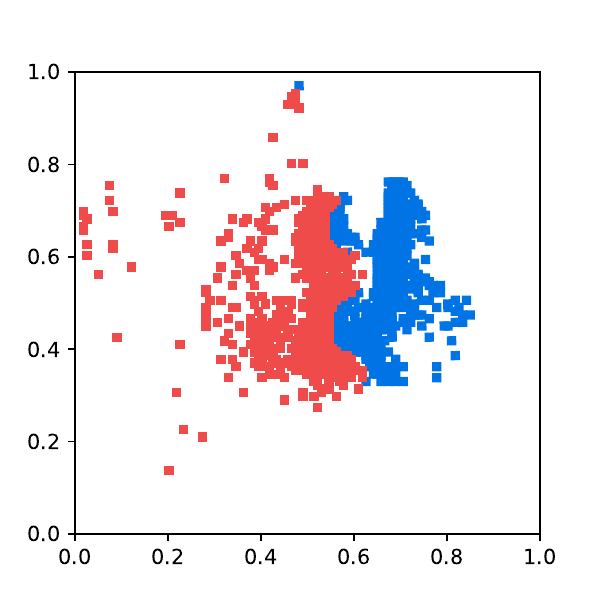}
            \label{fig:magent-surround}
    }
    \vspace{-10pt}
    \caption{MAgent training. 
    (a-b) the negative cost (reward) and convergence performance of F2A2 for $12-256$ agents, the solid line represents the average under $10$ random seeds.
    (c-e), F2A2 is first trained in the environment with $256$ agents and then transferred to a larger environment with $625$ agents directly. Here some interesting patterns that the agents show after transfer are shown.}
    \label{fig:magent}
\end{figure}

\subsection{Large-scale MAgent Environment}
The \textit{Battle} scenario in the MAgent platform provided by~\citet{zheng2018magent} is chosen as the simulation environment.
The Battle game is a general cooperative POSG scenario with two armies fighting against each other in a grid world, each empowered by a different RL algorithm. 
Each army consists of $12-256$ homogeneous agents. 
The goal of each army is to get fewer costs by collaborating with teammates to destroy all the opponents. 
The agent can take action to either move to or attack nearby grids. 
Ideally, the agent should learn skills such as chasing to hunt after training. 
The default cost setting is adopted: $0.005$ for every move, $-0.2$ for attacking an enemy, $-5$ for killing an enemy, $0.1$ for attacking an empty grid, and $0.1$ for being attacked or killed.
Since the proposed F2A2-DDPG/TD3/SAC/COMA are similar to MADDPG, they can only be run on a small scale due to global critic.
Hence F2A2 is extended to a large-scale off-policy fully-decentralized actor-critic method, i.e. F2A2-ISAC, which combines F2A2-SAC with IQL \citep{tan1993multi}.
It differs from F2A2-SAC only in the input accepted by the critic. 
The former requires observations and (estimated) policies of all agents, while F2A2-ISAC only requires individual local observation and individual policy. 
Therefore, the MOA module is no longer required in F2A2-ISAC.
Although the MAgent environment scale is enormous, each agent's task is relatively simple so that F2A2-ISAC can achieve good results.
The parameters in actor and critic for all agents are shared, similar to \citep{gupta2017cooperative}.

In Figure~\ref{fig:magent-rewards}, as the number of agents increases, the cost curves have a similar trend, which indicates that the decentralized learning ability is not affected by the large-scale setting.
In Figure~\ref{fig:magent-times}, the convergence time does not increase significantly as the number of agents grows. 
The curve shows that the proposed algorithms have good scalability. 
Algorithms also learn some interesting patterns. 
In figure~\ref{fig:magent-explore}, in the early stage of confrontation, the rear agents explore the environment because they cannot directly engage the enemy;
in \ref{fig:magent-pursuit}, when a small maniple of enemies escape, agents are splitted to chase;
in \ref{fig:magent-surround}, the large forces use the means of encirclement to conquest.

\subsection{More Analysis}

\begin{table}[tb!]
    \resizebox{\textwidth}{!}{%
    \begin{tabular}{|c|c|cccccc|}
    \hline
    \multirow{3}{*}{ctc} & MAAC & \multicolumn{6}{c|}{F2A2-SAC} \\ \cline{2-8} 
     & MAAC & \multicolumn{1}{c|}{F2A2-SAC-0} & \multicolumn{1}{c|}{F2A2-SAC-20} & \multicolumn{1}{c|}{F2A2-SAC-40} & \multicolumn{1}{c|}{F2A2-SAC-60} & \multicolumn{1}{c|}{F2A2-SAC-80} & F2A2-SAC-100 \\ \cline{2-8} 
     & \textbf{148($\pm$8)} & \multicolumn{1}{c|}{23($\pm$4)} & \multicolumn{1}{c|}{37($\pm$4)} & \multicolumn{1}{c|}{45($\pm$5)} & \multicolumn{1}{c|}{112($\pm$4)} & \multicolumn{1}{c|}{142($\pm$8)} & \textbf{147($\pm$9)} \\ \hline
    \multirow{3}{*}{rt} & MAAC & \multicolumn{6}{c|}{F2A2-SAC} \\ \cline{2-8} 
     & MAAC & \multicolumn{1}{c|}{F2A2-SAC-0} & \multicolumn{1}{c|}{F2A2-SAC-20} & \multicolumn{1}{c|}{F2A2-SAC-40} & \multicolumn{1}{c|}{F2A2-SAC-60} & \multicolumn{1}{c|}{F2A2-SAC-80} & F2A2-SAC-100 \\ \cline{2-8} 
     & 121($\pm$5) & \multicolumn{1}{c|}{52($\pm$7)} & \multicolumn{1}{c|}{52($\pm$7)} & \multicolumn{1}{c|}{78($\pm$5)} & \multicolumn{1}{c|}{126($\pm$6)} & \multicolumn{1}{c|}{135($\pm$5)} & \textbf{139($\pm$6)} \\ \hline
    \multirow{3}{*}{3m} & MAAC & \multicolumn{6}{c|}{F2A2-SAC} \\ \cline{2-8} 
     & MAAC & \multicolumn{1}{c|}{F2A2-SAC-0} & \multicolumn{1}{c|}{F2A2-SAC-20} & \multicolumn{1}{c|}{F2A2-SAC-40} & \multicolumn{1}{c|}{F2A2-SAC-60} & \multicolumn{1}{c|}{F2A2-SAC-80} & F2A2-SAC-100 \\ \cline{2-8} 
     & 0.88($\pm$0.11) & \multicolumn{1}{c|}{0.65($\pm$0.08)} & \multicolumn{1}{c|}{0.68($\pm$0.07)} & \multicolumn{1}{c|}{0.70($\pm$0.08)} & \multicolumn{1}{c|}{0.76($\pm$0.09)} & \multicolumn{1}{c|}{\textbf{0.91($\pm$0.10)}} & \textbf{0.92($\pm$0.09)} \\ \hline
    \multirow{3}{*}{8m} & COMA & \multicolumn{6}{c|}{F2A2-COMA} \\ \cline{2-8} 
     & COMA & \multicolumn{1}{c|}{F2A2-COMA-0} & \multicolumn{1}{c|}{F2A2-COMA-20} & \multicolumn{1}{c|}{F2A2-COMA-40} & \multicolumn{1}{c|}{F2A2-COMA-60} & \multicolumn{1}{c|}{F2A2-COMA-80} & F2A2-COMA-100 \\ \cline{2-8} 
     & \textbf{0.97($\pm$0.01)} & \multicolumn{1}{c|}{0.83($\pm$0.01)} & \multicolumn{1}{c|}{0.82($\pm$0.01)} & \multicolumn{1}{c|}{0.85($\pm$0.02)} & \multicolumn{1}{c|}{0.84($\pm$0.01)} & \multicolumn{1}{c|}{\textbf{0.94($\pm$0.01)}} & \textbf{0.96($\pm$0.02)} \\ \hline
    \multirow{3}{*}{2s3z} & MAAC & \multicolumn{6}{c|}{F2A2-SAC} \\ \cline{2-8} 
     & MAAC & \multicolumn{1}{c|}{F2A2-SAC-0} & \multicolumn{1}{c|}{F2A2-SAC-20} & \multicolumn{1}{c|}{F2A2-SAC-40} & \multicolumn{1}{c|}{F2A2-SAC-60} & \multicolumn{1}{c|}{F2A2-SAC-80} & F2A2-SAC-100 \\ \cline{2-8} 
     & 0.48($\pm$0.02) & \multicolumn{1}{c|}{0.35($\pm$0.02)} & \multicolumn{1}{c|}{0.35($\pm$0.03)} & \multicolumn{1}{c|}{0.41($\pm$0.02)} & \multicolumn{1}{c|}{0.43($\pm$0.02)} & \multicolumn{1}{c|}{\textbf{0.55($\pm$0.01)}} & \textbf{0.57($\pm$0.02)} \\ \hline
    \end{tabular}%
    }
    \caption{The performance comparison of each algorithm under different communication topologies. The numbers following the F2A2 algorithm represent the proportions of edges added in the ring topology, and $0$ means that no edges exist in the communication topology. ``ctc'' and ``rt'' represent the \textit{Cooperative Treasure Collection} task and the \textit{Rover Tower} task in the \textit{Cooperative MPE} environment, respectively. The data in the \textit{Cooperative MPE} environment represents the mean episode reward of all agents in the last $1k$ timesteps; And ``3m'', ``8m'', ``2s3z'' represent different maps in \textit{SC II} environment. The data in the \textit{SC II} environment represents the mean test win rate of all agents in the last $10k$ timesteps. Numbers in parentheses indicate $95\%$ confidence intervals, obtained under $5$ random seeds and blacked numbers indicate best results.}
    \label{tab:ablation-topo}
\end{table}

\begin{table}[htb!]
    \resizebox{\textwidth}{!}{%
    \begin{tabular}{|c|c|cccccc|}
    \hline
    \multirow{3}{*}{ctc} & MAAC & \multicolumn{6}{c|}{F2A2-SAC} \\ \cline{2-8} 
     & MAAC & \multicolumn{1}{c|}{F2A2-SAC-$\mathcal{N}(\infty,0)$} & \multicolumn{1}{c|}{F2A2-SAC-$\mathcal{N}(2000,1000)$} & \multicolumn{1}{c|}{F2A2-SAC-$\mathcal{N}(200,100)$} & \multicolumn{1}{c|}{F2A2-SAC-$\mathcal{N}(20,10)$} & \multicolumn{1}{c|}{F2A2-SAC-$\mathcal{N}(8,4)$} & F2A2-SAC-$\mathcal{N}(1,0)$ \\ \cline{2-8} 
     & \textbf{148($\pm$8)} & \multicolumn{1}{c|}{23($\pm$4)} & \multicolumn{1}{c|}{24($\pm$4)} & \multicolumn{1}{c|}{33($\pm$7)} & \multicolumn{1}{c|}{139($\pm$7)} & \multicolumn{1}{c|}{\textbf{146($\pm$9)}} & \textbf{147($\pm$9)} \\ \hline
    \multirow{3}{*}{rt} & MAAC & \multicolumn{6}{c|}{F2A2-SAC} \\ \cline{2-8} 
     & MAAC & \multicolumn{1}{c|}{F2A2-SAC-$\mathcal{N}(\infty,0)$} & \multicolumn{1}{c|}{F2A2-SAC-$\mathcal{N}(2000,1000)$} & \multicolumn{1}{c|}{F2A2-SAC-$\mathcal{N}(200,100)$} & \multicolumn{1}{c|}{F2A2-SAC-$\mathcal{N}(20,10)$} & \multicolumn{1}{c|}{F2A2-SAC-$\mathcal{N}(8,4)$} & F2A2-SAC-$\mathcal{N}(1,0)$ \\ \cline{2-8} 
     & 121($\pm$5) & \multicolumn{1}{c|}{52($\pm$7)} & \multicolumn{1}{c|}{52($\pm$8)} & \multicolumn{1}{c|}{56($\pm$8)} & \multicolumn{1}{c|}{\textbf{136($\pm$6)}} & \multicolumn{1}{c|}{\textbf{139($\pm$5)}} & \textbf{139($\pm$6)} \\ \hline
    \multirow{3}{*}{3m} & MAAC & \multicolumn{6}{c|}{F2A2-SAC} \\ \cline{2-8} 
     & MAAC & \multicolumn{1}{c|}{F2A2-SAC-$\mathcal{N}(\infty,0)$} & \multicolumn{1}{c|}{F2A2-SAC-$\mathcal{N}(2000,1000)$} & \multicolumn{1}{c|}{F2A2-SAC-$\mathcal{N}(200,100)$} & \multicolumn{1}{c|}{F2A2-SAC-$\mathcal{N}(20,10)$} & \multicolumn{1}{c|}{F2A2-SAC-$\mathcal{N}(8,4)$} & F2A2-SAC-$\mathcal{N}(1,0)$ \\ \cline{2-8} 
     & 0.88($\pm$0.11) & \multicolumn{1}{c|}{0.65($\pm$0.08)} & \multicolumn{1}{c|}{0.65($\pm$0.07)} & \multicolumn{1}{c|}{0.74($\pm$0.06)} & \multicolumn{1}{c|}{\textbf{0.91($\pm$0.09)}} & \multicolumn{1}{c|}{\textbf{0.91($\pm$0.10)}} & \textbf{0.92($\pm$0.09)} \\ \hline
    \multirow{3}{*}{8m} & COMA & \multicolumn{6}{c|}{F2A2-COMA} \\ \cline{2-8} 
     & COMA & \multicolumn{1}{c|}{F2A2-COMA-$\mathcal{N}(\infty,0)$} & \multicolumn{1}{c|}{F2A2-COMA-$\mathcal{N}(2000,1000)$} & \multicolumn{1}{c|}{F2A2-COMA-$\mathcal{N}(200,100)$} & \multicolumn{1}{c|}{F2A2-COMA-$\mathcal{N}(20,10)$} & \multicolumn{1}{c|}{F2A2-COMA-$\mathcal{N}(8,4)$} & F2A2-COMA-$\mathcal{N}(1,0)$ \\ \cline{2-8} 
     & \textbf{0.97($\pm$0.01)} & \multicolumn{1}{c|}{0.83($\pm$0.01)} & \multicolumn{1}{c|}{0.83($\pm$0.02)} & \multicolumn{1}{c|}{0.83($\pm$0.02)} & \multicolumn{1}{c|}{\textbf{0.93($\pm$0.02)}} & \multicolumn{1}{c|}{\textbf{0.95($\pm$0.01)}} & \textbf{0.96($\pm$0.02)} \\ \hline
    \multirow{3}{*}{2s3z} & MAAC & \multicolumn{6}{c|}{F2A2-SAC} \\ \cline{2-8} 
     & MAAC & \multicolumn{1}{c|}{F2A2-SAC-$\mathcal{N}(\infty,0)$} & \multicolumn{1}{c|}{F2A2-SAC-$\mathcal{N}(2000,1000)$} & \multicolumn{1}{c|}{F2A2-SAC-$\mathcal{N}(200,100)$} & \multicolumn{1}{c|}{F2A2-SAC-$\mathcal{N}(20,10)$} & \multicolumn{1}{c|}{F2A2-SAC-$\mathcal{N}(8,4)$} & F2A2-SAC-$\mathcal{N}(1,0)$ \\ \cline{2-8} 
     & 0.48($\pm$0.02) & \multicolumn{1}{c|}{0.35($\pm$0.02)} & \multicolumn{1}{c|}{0.37($\pm$0.01)} & \multicolumn{1}{c|}{0.42($\pm$0.01)} & \multicolumn{1}{c|}{0.53($\pm$0.01)} & \multicolumn{1}{c|}{\textbf{0.57($\pm$0.02)}} & \textbf{0.57($\pm$0.02)} \\ \hline
    \end{tabular}%
    }
    \caption{The performance comparison of each algorithm under different parameter exchange frequencies. The normal distribution followed by the F2A2 algorithm indicates that each agent samples the time interval for executing the next \textit{consensus step} from the normal distribution after each execution of the current \textit{consensus step}, where $\infty$ indicates the predefined maximum executable update step of the algorithm. ``ctc'' and ``rt'' represent the \textit{Cooperative Treasure Collection} task and the \textit{Rover Tower} task in the \textit{Cooperative MPE} environment, respectively. The data in the \textit{Cooperative MPE} environment represents the mean episode reward of all agents in the last $1k$ timesteps; And ``3m'', ``8m'', ``2s3z'' represent different maps in \textit{SC II} environment. The data in the \textit{SC II} environment represents the mean test win rate of all agents in the last $10k$ timesteps. Numbers in parentheses indicate $95\%$ confidence intervals, obtained under $5$ random seeds and blacked numbers indicate best results.}
    \label{tab:ablation-freq}
\end{table}

\paragraph{F2A2 reduce the information transmission during the decentralized algorithm learning.}
We only consider the information transmission between the agents of the decentralized algorithm here. 
For the centralized algorithm, all the information transmission is carried out on one physical machine, so it is meaningless to count the amount of information transmission.
We count the appropriate parameter amount and information transmission amount of the decentralized algorithm involved in the experiment in the two simulation environments. 
After statistics, it can be seen that F2A2 dramatically reduces the transfer of parameters in addition to shared parameters caused by policy sharing and only requires one-tenth of the transmissions.

\paragraph{The PDHG-type methods have advantages over the BCGD-type methods.}
To verify the superiority of the PDHG-type method, we select the corresponding optimal algorithms in two multi-agent environments and conduct comparative experiments. 
It can be seen from Figure~\ref{fig:scii-ablation} and Figure~\ref{fig:msl-ablation} that the PDHG-type jointly optimization methods (*-with-JOINT) have a significant performance improvement compared to the BCGD-type separately optimization methods (*-w/o-JOINT). 

\paragraph{The MOA module further improves the performance of the decentralized MARL algorithm.}
In the previous section, we argue that the MOA module can prevent the agent from overfitting other agents' policies by introducing noise, and increase exploration, thereby ultimately improving algorithm performance. 
Moreover, in some scenarios, it can exceed the centralized algorithms. 
To quantify the effect of the MOA module, the comparative experiments are also conducted.
The experimental results show in Figure~\ref{fig:msl-ablation} and Figure~\ref{fig:scii-ablation} are clearly support the above conclusions.

\begin{figure*}[tb!]
    \centering
    \subfigure[StarCraft II 2s3z scenario.]{
    \label{fig:scii-ablation}
    \includegraphics[width=0.47\textwidth]{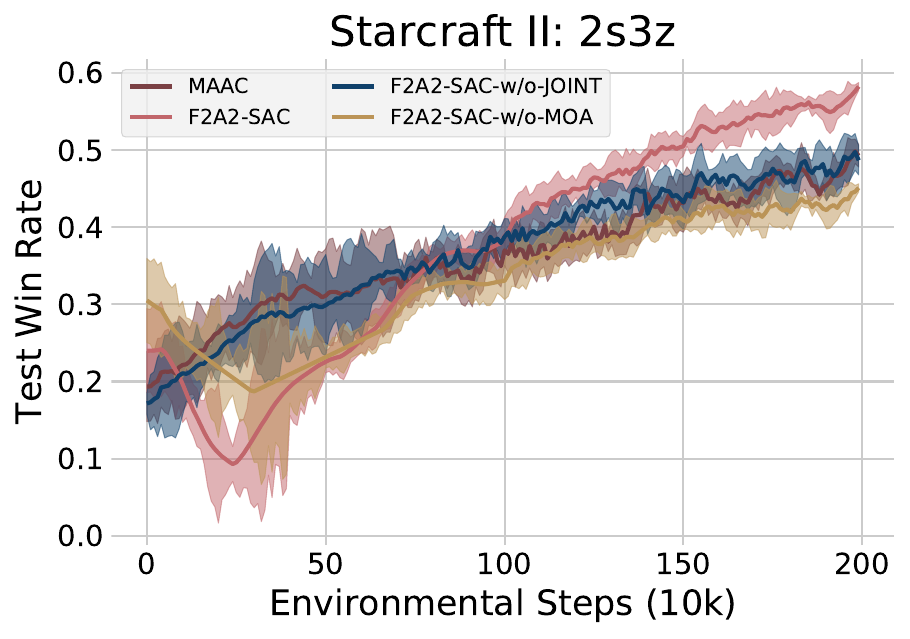}
    }
    \subfigure[Cooperative MPE Rover-Tower scenario.]{
    \label{fig:msl-ablation}
    \includegraphics[width=0.47\textwidth]{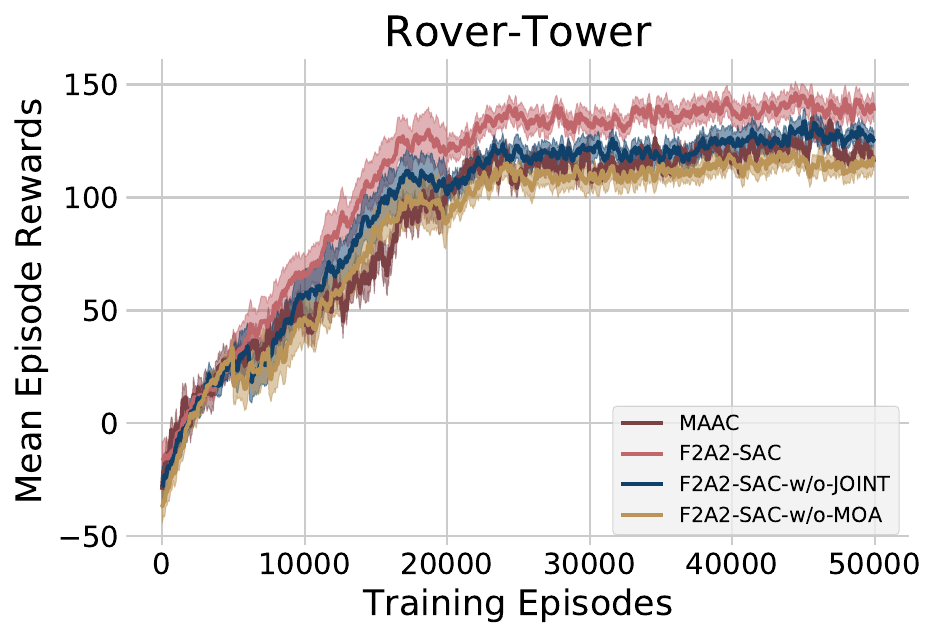}
    }
    \vspace{-10pt}
    \caption{The ablation study of F2A2-SAC algorithm in two multi-agent environments.}
\end{figure*}

Finally, to delve into why the F2A2 algorithm can outperform the centralized algorithm in some cases, we conduct a fine-grained comparison between centralized and decentralized methods for various levels of communication budgets.
Specifically, the communication budget of the algorithm is constrained in the following two ways:
    \begin{itemize}
        \item \textbf{Communication topology.} 
            The above experiments assume that the communication topology is a fully connected graph. 
            In order to add the communication budget, we construct the communication topology as follows. 
            For experiments under each random seed, we first sort all agents in random order to construct a ring topology (guarantee that there is exactly one connected component in the whole graph). 
            Then randomly add edges between agents in different predefined proportions at each update step, and the upper limit of the number of edges is the total number of edges in the fully connected graph. 
            Furthermore, we also consider an extreme case where there are no edges in the graph.
        \item \textbf{Parameter exchange frequency.}
            At each \textit{consensus step}, each agent samples the time interval to execute the next \textit{consensus step} from a normal distribution. 
            The mean of this normal distribution is set to different values from small to large to represent the change in communication budget from large to small. 
            The variance of the normal distribution is fixed, and for each communication budget, different agents sample from the same normal distribution.
            We also consider an extreme case where the mean of the normal distribution is greater than the total number of update steps in the algorithm, and the variance is $0$.
    \end{itemize}
    For each task, we select the best performing centralized algorithm and its corresponding decentralized version for comparison.
    The experimental results are shown in Table~\ref{tab:ablation-topo} and Table~\ref{tab:ablation-freq}.
    It can be seen from the table that with the increase of the communication budget, the performance of the F2A2 algorithms present \ul{a step-wise growth rather than a smooth and gradual improvement}.
    Moreover, it can be seen from the last few columns of the table that the F2A2 algorithm can still maintain good performance under the weak communication budget constraint.

\section{Conclusion and Limitations}\label{sec:conclusion}
A flexible fully decentralized approximate actor-critic MARL framework is devised to achieve applicability and scalability in interactive multi-agent environments. 
A primal-dual optimization and joint actor-critic learning are carefully designated to guarantee full decentralization and scalability, with agents modeling to increase the robustness.
The proposed approach can even exceed state-of-the-art centralized algorithms in various categories and various scales simulated cooperative environments.
In the future, we plan to introduce communication in the training process to promote more efficient cooperation to be more adaptable to complex scenarios.

Below we briefly analyze the limitations of F2A2.
We ensured the Markov property of the policy through the enrichment of the observation space, avoiding the use of non-Markovian and history-dependent policy classes. 
Although this approach facilitates theoretical analysis and engineering implementation, the resulting \textit{vast belief space} and \textit{infinite hierarchy of belief} render the Markov policy class intractable in more complex problems, further constraining its generality.
We shed light on the limitations that affect this policy class concerning crucial factors such as the error of the approximated information state and the infinite hierarchy belief representation capabilities in Appendix~\ref{sec:limit-appendix}. 
Careful consideration of these factors is essential to enhance the tractability and generality of the Markovian policy within the enriched observation space.


\acks{This work was supported in part by the National Key Research and Development Program of China (No. 2020AAA0107400), STCSM (22QB1402100), NSFC (No. 12071145), Shenzhen Science and Technology Program (JCYJ20210324120011032), Postdoctoral Science Foundation of China (2022M723039), a grant from China Academy of LVT, and a grant from Shenzhen Institute of Artificial Intelligence and Robotics for Society.}





\newpage
\appendix

\section{Environments}

\paragraph{Cooperative Treasure Collection.} 
The cooperative environment in Figure~\ref{fig:ccv} involves $8$ total agents, $6$ of which are ”treasure hunters” and 2 of which are “treasure banks”, which each correspond to a different color of treasure. 
The role of the hunters is to collect the treasure of any color, which re-spawn randomly upon being collected (with a total
of 6), and then “deposit” the treasure into the correctly colored “bank”. 
The role of each bank is to gather as much treasure as possible from the hunters simply. 
All agents can see each others’ positions concerning their own.
Hunters receive a global cost for the successful collection of treasure, and all agents receive a global cost for the
depositing of treasure. 
Hunters are additionally penalized for colliding with each other. 
As such, the task contains a mixture of shared and individual costs.

\paragraph{Rover Tower.}
The environment in Figure~\ref{fig:rt} involves $8$ total agents, $4$ of which are “rovers” and another $4$ which are “towers”. 
In each episode, rovers and towers are randomly paired. 
The pair is punished by the distance of the rover to its goal. 
The task can be thought of as a navigation task on an alien planet with limited infrastructure and low visibility. 
The rovers are unable to see in their surroundings. 
They must rely on communication from the towers, which can locate the rovers and their destinations and send one of five discrete communication messages to their paired rover. 
Note that communication is highly restricted and different from centralized policy approaches~\citep{jiang2018learning}, which allow for the free transfer of continuous information among policies.
In our setup, the communication is integrated into the environment (in the tower’s action space and the rover’s observation space), rather than being explicitly part of the model. 
It is limited to a few discrete signals.

\begin{figure*}[!htp]
\centering
\subfigure[Cooperative Treasure Collection]{
\label{fig:ccv}
\includegraphics[width=0.4\textwidth]{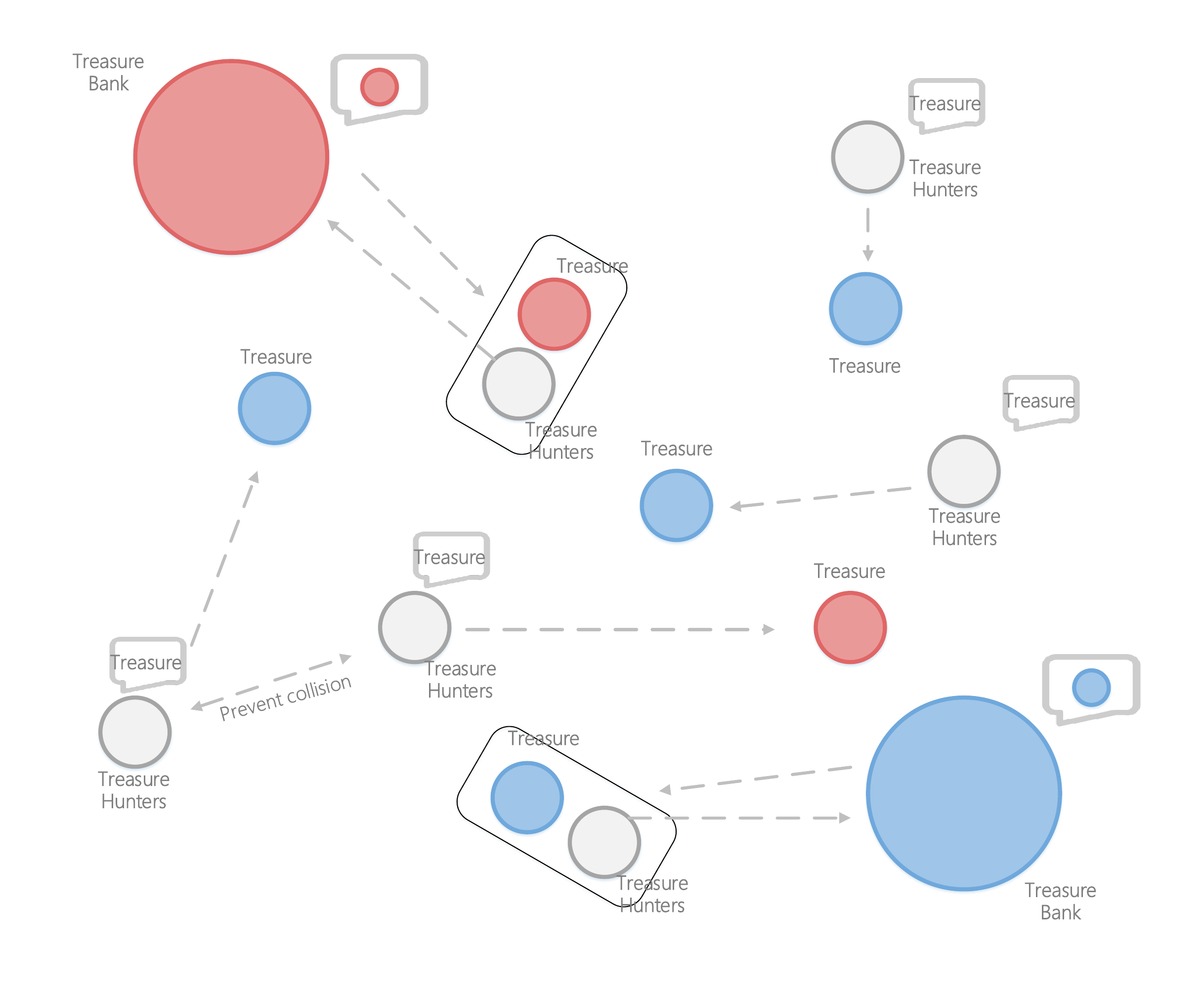}}
\subfigure[Rover Tower]{
\label{fig:rt}
\hspace{0.3in}
\includegraphics[width=0.4\textwidth]{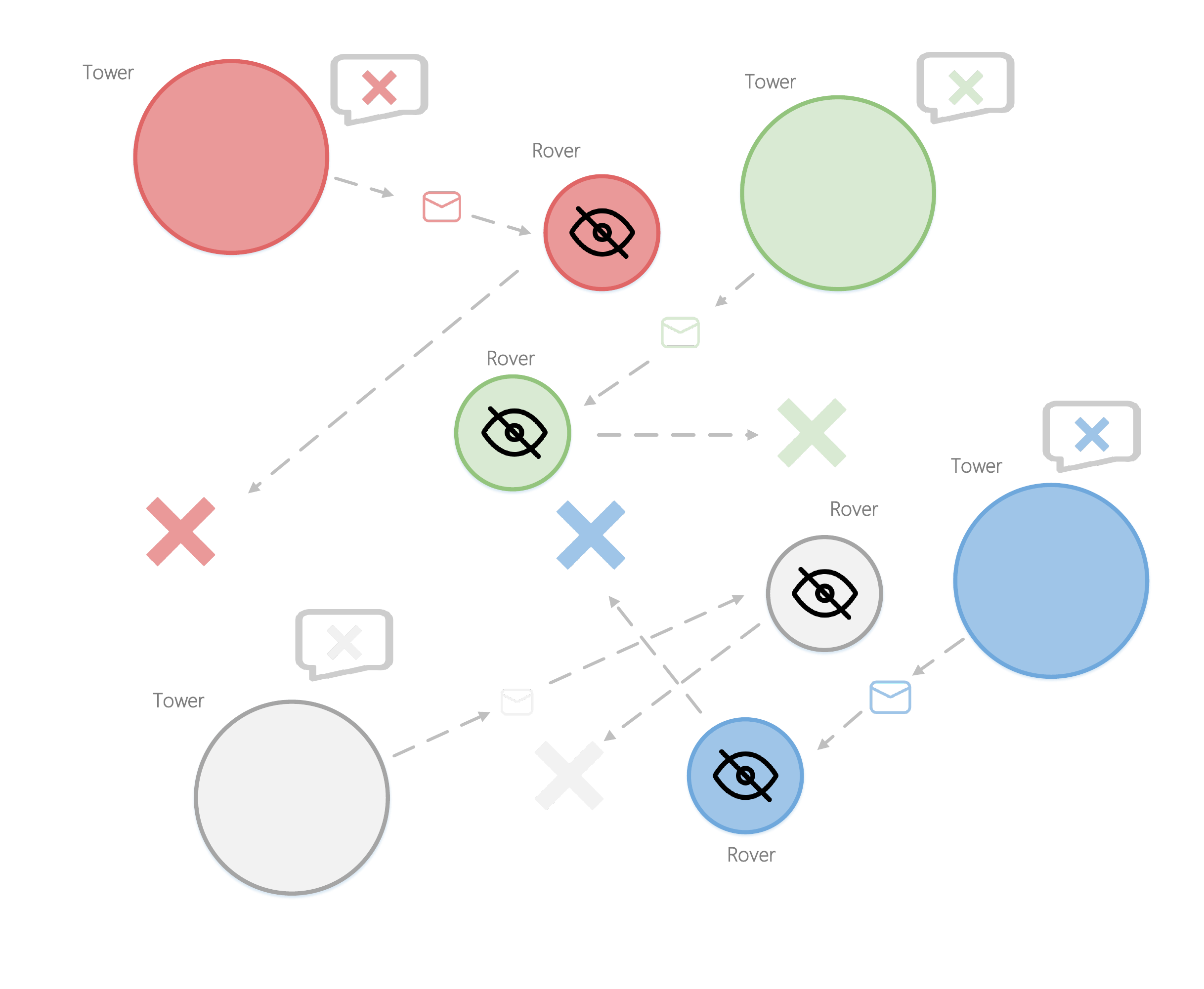}}
\caption{The Cooperative Multi-agent Particle Environments.}
\end{figure*}

\paragraph{StarCraft II Map 2s3z.} This map contains mixed unit types, where both the learnable agent and the built-in AI each control two Stalkers and three Zealots. 
Stalkers are ranged-attack units that take heavy damage from melee-type Zealots. 
Consequently, a winning strategy needs to dynamically coordinate between letting one’s Zealots attack enemy Stalkers and backtrack to defend one’s Stalkers against enemy Zealots. 

\paragraph{StarCraft II Map 3m and 8m.} The first, map 3m, presents both sides with three Marines, which are medium-ranged infantry
units. 
The coordination challenge on this map is to reduce enemy firepower as quickly as possible by focusing unit fire to defeat each enemy unit in turn. 
Secondly, map 8m scales this task up to eight Marines on both sides. 
The relatively large number of agents involved poses additional scalability challenges.

\section{Relationship with Trust-region Methods}
The proposed F2A2 algorithm contains two learning processes: the learning of multi-agent policies and the learning of MOA modules. 
Moreover, the performance of the latter will affect the former. 
Once the MOA modules are too inaccurate, the entire F2A2 algorithm training process will fall into a vicious circle. 
In order to make the training of the entire system more robust, the intuitive idea is to make the multi-agent policies update more conservative, which can effectively improve the accuracy of MOA modules.
Trust-region methods, such as Proximal Policy Optimization (PPO)~\citep{schulman2017proximal} have been demonstrated as efficient and robust RL algorithms via maximally searching the new policy in a trust region.
Coincidentally, in comparing the PPO with the F2A2, we found that the proposed F2A2 framework is closely related to trust-region methods.
This means that F2A2 naturally has the conservativeness of policy updating by jointly optimizing actors and critics.
More specifically, the simplified objective of PPO is as follows,
\begin{equation}
L\left(s, a, \theta_{k}, \theta\right)=
\min \left(\frac{\pi_{\theta}(a | s)}{\pi_{\theta_{k}}(a | s)} A^{\pi_{\theta_{k}}}(s, a), g\left(\epsilon, A^{\pi_{\theta_{k}}}(s, a)\right)\right), \nonumber
\end{equation}
where 
$$
    g(\epsilon, A) = \left\{\begin{array}{ll}{(1+\epsilon) A}, & {A \geq 0}; \\ {(1-\epsilon) A}, & {A<0}.\end{array}\right. 
$$

PPO deals with the advantage function $A^{\pi}(s,a)$, while our F2A2 deals with the gradient of value $Q^{\pi}(o,a)$.
Note that PPO and F2A2 have a clipping regularizer to prevent the new policy beyond the trust region.
The hyperparameter $\epsilon$ corresponds to the distance that the new policy can go away and still profits the objective~\citep{schulman2017proximal}.
Besides, the coefficient $\alpha_1-2\alpha_2\delta$ in the gradient of F2A2 makes the actor update more conservative when the current critic is not accurate enough(and the accuracy of the current critic is greatly affected by the accuracy of the MOA modules).
To sum up, our F2A2 framework has a similar protective trust-region mechanism, which enhances the robustness of our algorithm.

{\color{black}{

\section{Policy Gradient Theorem in POMDPs}

After enriching the observation space, we can naturally extend the policy gradient theorem under MDP to POMDP based on \citet{mao2020information} as the basis for the subsequent theoretical derivation of this paper. 
In the previous version of manuscript, we omitted this derivation process and make an addition here.
Notice that here $o$ refer to the \underline{enriched} observation unless specified.
Let $p(\tau;\theta)$ denote the probability distribution of trajectories $\tau$ under policy $\pi_{\theta}$, i.e., 
$$
    f(\tau ; \theta) d \tau:=\mathcal{P}\left(s_{0}\right) \prod_{t=0}^{\infty} \mathcal{E}\left(o_{t} \mid s_{t}\right) \pi_{\theta}\left(a_{t} \mid o_{t}\right) \mathcal{C}\left(c_{t} \mid o_{t}, a_{t}\right) \mathcal{P}\left(s_{t+1} \mid s_{t}, a_{t}\right) d \tau,
$$
For a trajectory $\tau$, a random variable $R(\tau)$ represent the cumulative $\gamma$-discounted costs in $\tau$, i.e., $R(\tau)=\sum_{t=0}^{\infty} \gamma^{t} c_{t}$. 
Therefore, solving a POMDP is to find a policy $\pi_{\theta}$ that minimizes the following optimization objectives
$$
    \eta(\theta):=\mathbb{E}_{\tau\sim f}[R(\tau)]=\int_{\tau} f(\tau ; \theta) R(\tau) d \tau \longrightarrow \theta^{*} \in \arg \min _{\theta \in \Theta} \eta(\theta).
$$
And let
$
{V}_{\theta}\left(o_{t}\right):=\mathbb{E}_{\tau\sim f}\left[\sum_{t^{\prime}=t}^{\infty} \gamma^{t^{\prime}-t} c_{t^{\prime}} \mid o_{t}\right], \quad {Q}_{\theta}\left(o_{t}, a_{t}\right):=\mathbb{E}_{\tau\sim f}\left[\sum_{t^{\prime}=t}^{\infty} \gamma^{t^{\prime}-t} c_{t^{\prime}} \mid o_{t}, a_{t}\right].
$
Now we extend the main Policy Gradient Theorem~\citep{sutton2000policy} to episodic POMDPs, derive by simple modifications in their predecessors.
\begin{tcolorbox}
    \begin{theo}[Policy Gradient in POMDPs]\label{the:popg}
        For a given policy $\pi_{\theta}$ on a POMDP,
        $
        \nabla_{\theta} \eta(\theta)=\mathbb{E}_{\tau\sim f}\left[ \sum_{t=0}^{\infty} \nabla_{\theta} \log \pi_{\theta}\left(a_{t} \mid o_{t}\right) {Q}_{\theta}\left(o_{t}, a_{t}\right) \right].
        $
    \end{theo}
\end{tcolorbox}
\begin{tcolorbox}[breakable, enhanced]
    \begin{proof}
        For a given policy $\pi_{\theta}$, lets restate the value function at a given enriched observation $o_{0}$ :
        $
        {V}_{\theta}\left(o_{0}\right)=\int_{a_{0}} \pi_{\theta}\left(a_{0} \mid o_{0}\right) {Q}_{\theta}\left(o_{0}, a_{0}\right) d a_{0}.
        $
        For this value function, we compute the gradient with respect to parameters $\theta$, i.e.,
        $$
        \begin{aligned}
            \nabla_{\theta} {V}_{\theta}\left(o_{0}\right)=&\int_{a_{0}} \nabla_{\theta} \pi_{\theta}\left(a_{0} \mid o_{0}\right) {Q}_{\theta}\left(o_{0}, a_{0}\right) d a_{0}+\int_{a_{0}} \pi_{\theta}\left(a_{0} \mid o_{0}\right) \nabla_{\theta} {Q}_{\theta}\left(o_{0}, a_{0}\right) d a_{0}.
        \end{aligned}
        $$
        To further expand this gradient, consider the definition of ${Q}_{\theta}\left(o_{0}, a_{0}\right)$ using Bellman equation (the enriched observation space satisfies Markov property~\citep{mao2020information}),
        $$
        \begin{aligned}
        {Q}_{\theta}\left(o_{0}, a_{0}\right)=&\int_{s_1}\mathcal{P}(s_1 \mid s_0, a_0) \mathcal{C}(c_0 \mid s_0, a_0, s_1) d s_1 \\
        &+ \gamma \int_{s_{1},o_1} \mathcal{P}(s_1 \mid s_0, a_0) \mathcal{E}(o_1 \mid s_1) {V}_{\theta}\left(o_{1}\right) d s_{1} d o_1,
        \end{aligned}
        $$
        resulting in,
        $
        \nabla_{\theta} {Q}_{\theta}\left(o_{0}, a_{0}\right)=\gamma \int_{s_{1},o_1} \mathcal{P}(s_1 \mid s_0, a_0) \mathcal{E}(o_1 \mid s_1) \nabla_{\theta} {V}_{\theta}\left(o_{1}\right) d s_{1} d o_1,
        $
        since the first term in the definition of ${Q}_{\theta}\left(o_{0}, a_{0}\right)$ does not depend on the parameters $\theta$. 
        Following these steps, we derive $\nabla_{\theta} {V}_{\theta}\left(o_{1}\right)$,
        $$
        \begin{aligned}
            \nabla_{\theta} {V}_{\theta}\left(o_{1}\right)=&\int_{a_{1}} \nabla_{\theta} \pi_{\theta}\left(a_{1} \mid o_{1}\right) {Q}_{\theta}\left(o_{1}, a_{1}\right) d a_{1} +\int_{a_{1}} \pi_{\theta}\left(a_{1} \mid o_{1}\right) \nabla_{\theta} {Q}_{\theta}\left(o_{1}, a_{1}\right) d a_{1}.
        \end{aligned}
        $$
        We recursively compute the gradient of the value functions for later time steps and conclude that,
        $$
        \begin{aligned}
            &\nabla_{\theta} {V}_{\theta}\left(o_{t}\right)=\int_{a_{t}} \nabla_{\theta} \pi_{\theta}\left(a_{t} \mid o_{t}\right) {Q}_{\theta}\left(o_{t}, a_{t}\right) d a_{t} \\
            &+\gamma \int_{a_{t}} \pi_{\theta}\left(a_{t} \mid o_{t}\right) \left( \int_{s_{t+1},o_{t+1}} \mathcal{P}(s_{t+1} \mid s_{t}, a_{t}) \mathcal{E}(o_{t+1} \mid s_{t+1}) \nabla_{\theta} {V}_{\theta}\left(o_{t+1}\right) d s_{t+1} d o_{t+1}, \right) d a_{t}.
        \end{aligned}
        $$
        and repeating this decomposition results in
        $$
        \begin{aligned}
            \nabla_{\theta}\eta(\theta) = &\int_{s_0, o_0} \mathcal{P}(s_0) \mathcal{E}(o_0 \mid s_0) \nabla_{\theta} {V}_{\theta}\left(o_{0}\right) d s_0 d o_0 \\
            = & \int_{s_0, o_0} \mathcal{P}(s_0) \mathcal{E}(o_0 \mid s_0) \nabla_{\theta} \int_{a_{0}} \pi_{\theta}\left(a_{0} \mid o_{0}\right) {Q}_{\theta}\left(o_{0}, a_{0}\right) d a_{0} d s_0 d o_0 \\
            = & \int_{s_0, o_0, a_0} \mathcal{P}(s_0) \mathcal{E}(o_0 \mid s_0)  \nabla_{\theta} \pi_{\theta}\left(a_{0} \mid o_{0}\right) {Q}_{\theta}\left(o_{0}, a_{0}\right) d a_{0} d s_0 d o_0 \\
            & \quad + \int_{s_0, o_0, a_0} \mathcal{P}(s_0) \mathcal{E}(o_0 \mid s_0) \pi_{\theta}\left(a_{0} \mid o_{0}\right) \nabla_{\theta} {Q}_{\theta}\left(o_{0}, a_{0}\right) d a_{0} d s_0 d o_0 \\
            = & \int_{\tau} \sum_{t=0}^{\infty} \gamma^{t} f\left(\tau_{0 . . t-1}, s_{t}, o_{t}; \theta\right) \nabla_{\theta} \pi_{\theta}\left(a_{t} \mid o_{t}\right) {Q}_{\theta}\left(o_{t}, a_{t}\right) d \tau \\
            = & \int_{\tau} \sum_{t=0}^{\infty} \gamma^{t} f\left(\tau_{0 . . t-1}, s_{t}, o_{t}; \theta\right) \pi_{\theta}\left(a_{t} \mid o_{t}\right) \nabla_{\theta} \log \pi_{\theta}\left(a_{t} \mid o_{t}\right) {Q}_{\theta}\left(o_{t}, a_{t}\right) d \tau \\
            = & \int_{\tau} \sum_{t=0}^{\infty} \gamma^{t} f\left(\tau_{0 . . t}; \theta\right) \nabla_{\theta} \log \pi_{\theta}\left(a_{t} \mid o_{t}\right) {Q}_{\theta}\left(o_{t}, a_{t}\right) d \tau \\
            = & \mathbb{E}_{\tau\sim f}\left[ \sum_{t=0}^{\infty} \gamma^{t} \nabla_{\theta} \log \pi_{\theta}\left(a_{t} \mid o_{t}\right) {Q}_{\theta}\left(o_{t}, a_{t}\right) \right].
        \end{aligned}
        $$
    \end{proof}
\end{tcolorbox}

It is worth noting that in most open-source implementations~\citep{baselines,stable-baselines3} of reinforcement learning algorithms based on policy gradient theorem, the term $\gamma^t$ in the policy gradient is ignored, and a biased policy gradient estimator $\mathbb{E}_{\tau\sim f}\left[ \sum_{t=0}^{\infty} \nabla_{\theta} \log \pi_{\theta}\left(a_{t} \mid o_{t}\right) {Q}_{\theta}\left(o_{t}, a_{t}\right) \right]$ is obtained.
However, it has remained the most popular estimator of the policy gradient due to its effectiveness when applied to practical problems. 
The precise reason for this effectiveness, especially in the episodic setting, remains an open question~\citep{nota2020policy}.
An in-depth exploration of this issue is beyond the scope of this paper, and we still use a biased policy gradient estimator in all algorithm implementations.

It is worth noting that the on-policy policy gradient have two different but equivalent forms: \textit{trajectory-oriented} and \textit{state-oriented}.
The detailed derivation process of the former in the POMDP has been developed above.
We now derive the second form in the POMDP, which is also the standard form of Sutton's policy gradient theorem. 
Since this form is more compact, this paper mainly derives the relevant policy gradient based on it. 
Specifically, we can write the optimization objective of reinforcement learning in another equivalent form
$$
    J(\theta)=\int_{s,o} d_{\pi}(s) \mathcal{E}(o \mid s) V_{\theta}(o) d s d o = \int_{s,o,a} d_{\pi}(s) \mathcal{E}(o \mid s)\pi_{\theta}(a \mid o) Q_{\theta}(o, a) d s d o d a,
$$
where $d_{\pi}$ represents the distribution of the state-occupancy measure of policy $\pi$.
Then we have following theorem:
\begin{tcolorbox}
    \begin{theo}[Policy Gradient in POMDPs]\label{the:popg2}
        For a given policy $\pi_{\theta}$ on a POMDP,
        $
            \nabla_{\theta}J(\theta) \propto \mathbb{E}_{\pi}\left[ \nabla_{\theta} \log \pi_{\theta}\left(a \mid o\right) {Q}_{\theta}\left(o, a\right) \right],
        $
        where $\mathbb{E}_{\pi}$ refers to $\mathbb{E}_{s \sim d_{\pi}, o \sim \mathcal{E}, a \sim \pi_{\theta}}$ when both state, enriched observation and action distributions follow the policy $\pi_{\theta}$ (on policy).
    \end{theo}
\end{tcolorbox}

\begin{tcolorbox}[breakable, enhanced]
    \begin{proof}
        We first start with the derivative of the state value function:
        $$
            \begin{array}{rlr} 
                & \nabla_{\theta} V_{\theta}(o) & \\
                = & \nabla_{\theta}\left(\int_{a} \pi_{\theta}(a \mid o) Q_{\theta}(o, a) d a\right)\\
                = & \int_a \left(\nabla_{\theta} \pi_{\theta}(a \mid o) Q_{\theta}(o, a) + \pi_{\theta}(a \mid o) \nabla_{\theta} Q_{\theta}(o, a)\right) d a \\
                = & \int_a \left(\nabla_{\theta} \pi_{\theta}(a \mid o) Q_{\theta}(o, a) + \pi_{\theta}(a \mid o) \right.\\
                &\left.\quad\quad\nabla_{\theta} \left( \int_{s',o',c} \mathcal{P}\left(s^{\prime} \mid s, a\right)\mathcal{E}(o' \mid s')\left(\mathcal{C}(c \mid s,a,s')+V_{\theta}\left(o^{\prime}\right)\right)\right)d s' d o' d c\right) d a\\
                = & \int_a \left(\nabla_{\theta} \pi_{\theta}(a \mid o) Q_{\theta}(o, a) + \pi_{\theta}(a \mid o) \left( \int_{s',o'} \mathcal{P}\left(s^{\prime} \mid s, a\right)\mathcal{E}(o' \mid s')\nabla_{\theta}V_{\theta}\left(o^{\prime}\right)\right)d s' d o'\right) d a.
            \end{array}
        $$
        This equation has a nice recursive form and the future value function $V_{\theta}\left(o^{\prime}\right)$ can be repeated unrolled by following the same equation.
        
        Let's consider the following visitation sequence and label the probability of transitioning from obervation $o$ to observation $\mathrm{x}$ with policy $\pi_{\theta}$ after $\mathrm{k}$ step as $\rho^{\pi}(o \rightarrow x, k)$.
        $$
            o \xrightarrow[]{o \sim \pi_\theta(.\vert o)} o' \xrightarrow[]{a \sim \pi_\theta(.\vert o')} o'' \xrightarrow[]{a \sim \pi_\theta(.\vert o'')} \dots
        $$
        \begin{itemize}
            \item When $\mathrm{k}=0: \rho^{\pi}(o \rightarrow o, k=0)=1$.
            \item When $\mathrm{k}=1$, we scan through all possible actions and sum up the transition probabilities to the target observation: $\rho^{\pi}\left(o \rightarrow o^{\prime}, k=1\right)=\int_{a,s'} \pi_{\theta}(a \mid o) \mathcal{P}\left(s^{\prime} \mid s, a\right)\mathcal{E}(o' \mid s')d a d s'$.
            \item Imagine that the goal is to go from observation o to $\mathrm{x}$ after $\mathrm{k}+1$ steps while following policy $\pi_{\theta}$. We can first travel from o to a middle point o' (any observation can be a middle point, $o^{\prime} \in \mathcal{O}$ ) after $\mathrm{k}$ steps and then go to the final observation $x$ during the last step. In this way, we are able to update the visitation probability recursively: $\rho^{\pi}(o \rightarrow x, k+1)=\int_{o'} \rho^{\pi}\left(o \rightarrow o^{\prime}, k\right) \rho^{\pi}\left(o^{\prime} \rightarrow x, 1\right) d o'$.
        \end{itemize}
        Then we go back to unroll the recursive representation of $\nabla_{\theta} V_{\theta}(o)$.
        Let $\phi(o)=\int_{a} \nabla_{\theta} \pi_{\theta}(a \mid o) Q_{\theta}(o, a) d a$ to simplify the maths.
        If we keep on extending $\nabla_{\theta} V_{\theta}(\cdot)$ infinitely, it is easy to find out that we can transition from the starting observation o to any observation after any number of steps in this unrolling process and by summing up all the visitation probabilities, we get $\nabla_{\theta} V_{\theta}(o)$.
            \begin{align*}
                & \nabla_{\theta} V_{\theta}(o) \\
                = & \phi(o) + \int_{a} \left(\pi_{\theta}(a \mid o) \left( \int_{s',o'} \mathcal{P}\left(s^{\prime} \mid s, a\right)\mathcal{E}(o' \mid s')\nabla_{\theta}V_{\theta}\left(o^{\prime}\right)\right)d s' d o'\right) d a \\
                = & \phi(o) + \int_{s',o',a} \pi_{\theta}(a \mid o) \mathcal{P}\left(s^{\prime} \mid s, a\right)\mathcal{E}(o' \mid s')\nabla_{\theta}V_{\theta}\left(o^{\prime}\right) d s' d o' d a \\
                = & \phi(o) + \int_{o'} \rho^{\pi}\left(o \rightarrow o^{\prime}, 1\right) \nabla_{\theta}V_{\theta}\left(o^{\prime}\right) d o' \\
                = & \phi(o) + \int_{o'} \rho^{\pi}\left(o \rightarrow o^{\prime}, 1\right) \left[ \phi(o') + \int_{o''} \rho^{\pi}\left(o' \rightarrow o'', 1\right) \nabla_{\theta}V_{\theta}\left(o''\right) d o'' \right] d o' \\
                = & \phi(o) + \int_{o'} \rho^{\pi}\left(o \rightarrow o^{\prime}, 1\right) \phi(o') d o' +  \int_{o''} \rho^{\pi}\left(o \rightarrow o'', 2\right) \nabla_{\theta}V_{\theta}\left(o''\right) d o'' \\
                = & \phi(o) + \int_{o'} \rho^{\pi}\left(o \rightarrow o^{\prime}, 1\right) \phi(o') d o' +  \int_{o''} \rho^{\pi}\left(o \rightarrow o'', 2\right) \phi(o'')  d o'' + \\
                & \quad \int_{o'''} \rho^{\pi}\left(o \rightarrow o''', 3\right) \nabla_{\theta}V_{\theta}\left(o'''\right) d o''' \\
                = & \int_{x}\sum_{k=0}^{\infty}\rho^{\pi}(o \rightarrow x, k) \phi(x) d x
            \end{align*}
        The nice rewriting above allows us to exclude the derivative of $Q$-value function, $\nabla_{\theta} Q_{\theta}(o, a)$. 
        By plugging it into the objective function $J(\theta)$, we are getting the following:
        $$
            \begin{aligned}
                \nabla_\theta J(\theta)
                &= \int_{s_0, o_0} \mathcal{P}(s_0) \mathcal{E}(o_0 \mid s_0) \nabla_\theta V_{\theta}(o_0) d s_0 d o_0 \\
                & = \int_{s_0, o_0} \mathcal{P}(s_0) \mathcal{E}(o_0 \mid s_0) \int_{o'}\sum_{k=0}^{\infty}\rho^{\pi}(o_0 \rightarrow o', k) \phi(o') d o' d s_0 d o_0 \\
                & = \int_{s_0, o_0} \mathcal{P}(s_0) \mathcal{E}(o_0 \mid s_0) \int_{o'}\eta(o') \phi(o') d o' d s_0 d o_0 \\
                & = \int_{s_0, o_0} \mathcal{P}(s_0) \mathcal{E}(o_0 \mid s_0) \left(\int_{o'}\eta(o') d o'\right)\int_{o'} \frac{\eta(o')}{\int_{o'}\eta(o') d o' } \phi(o') d o' d s_0 d o_0 \\
                & \propto \int_{s_0, o_0} \mathcal{P}(s_0) \mathcal{E}(o_0 \mid s_0) \int_{o'} \frac{\eta(o')}{\int_{o'}\eta(o') d o' } \phi(o') d o' d s_0 d o_0 \\
                & = \int_{s_0, o_0} \mathcal{P}(s_0) \mathcal{E}(o_0 \mid s_0) \int_{o'} d_{\pi}(o') \phi(o') d o' d s_0 d o_0 \\
                & = \int_{o'} d_{\pi}(o') \phi(o') d o' = \int_{s,o,a} d_{\pi}(s) \mathcal{E}(o \mid s) \nabla_{\theta} \pi_{\theta}\left(a \mid o\right) {Q}_{\theta}\left(o, a\right) d s d o d a \\
                & = \int_{s,o,a} d_{\pi}(s) \mathcal{E}(o \mid s) \pi_{\theta}\left(a \mid o\right) \frac{\nabla_{\theta} \pi_{\theta}\left(a \mid o\right)}{\pi_{\theta}\left(a \mid o\right)} {Q}_{\theta}\left(o, a\right) d s d o d a\\
                & = \mathbb{E}_{\pi}\left[ \nabla_{\theta} \log \pi_{\theta}\left(a \mid o\right) {Q}_{\theta}\left(o, a\right) \right].
            \end{aligned}
        $$
    \end{proof}
\end{tcolorbox}

It can be seen that the terms within the two different policy gradient expectations are the same (when $\gamma^t$ is ignored), but the distributions on which the expectations are based are different.
The former is based on trajectory distribution and the latter is based on stationary state distribution.

When there are partial observations, policy gradient theorem can be a bit subtle and the derivations in this work might require more careful scrutiny.
Fortunately, after proper enrichment of the observation space of the original POMDP, we can naturally extend the policy gradient theory under MDP to POMDP with minor modification, and use it as the basis for the derivation of the policy gradient of the proposed MARL algorithm under POSG in this paper.

}}

\section{Missing Proofs}

\subsection{Proof of the Proposition~\ref{the:coma}}

    \begin{proof}
        \textbf{We proof the first equation first}.
        Extend the COMA algorithm, we have (for convenience here we suppose state space, observation space and action space are discrete)
        \begin{equation}
        	\begin{aligned}
        		J^i_{\text{actor}} \left( \mathbf{w}_{in}, \tilde{\mathbf{w}}_{sh}^i \right) &= \mathbf{E}_{s \sim d_{\boldsymbol{\pi}}, \boldsymbol{o} \sim \mathcal{E},\boldsymbol{a}\sim\boldsymbol{\pi} }\left[Q^{\boldsymbol{\pi},i}_{\tilde{\phi}^i_{sh}}(\boldsymbol{o},\boldsymbol{a}) - \mathcal{B}(\boldsymbol{o}, \boldsymbol{a}^{\setminus i})\right] \\
        		&= \textstyle\sum_{s} d_{\boldsymbol{\pi}}(s) \textstyle\sum_{\boldsymbol{o}} \mathcal{E}(\boldsymbol{o}|s) \textstyle\sum_{\boldsymbol{a}} \boldsymbol{\pi}_{\Psi_{in}}(\boldsymbol{a}|\boldsymbol{o}) \left[Q^{\boldsymbol{\pi},i}_{\tilde{\phi}^i_{sh}}(\boldsymbol{o},\boldsymbol{a}) - \mathcal{B}(\boldsymbol{o}, \boldsymbol{a}^{\setminus i})\right]
        		\\
        		J^i_{\text{critic}} \left( \mathbf{w}_{in}, \tilde{\mathbf{w}}_{sh}^i \right) &= \mathbf{E}_{\substack{s\sim d_{\boldsymbol{\pi}}, \boldsymbol{o}\sim\mathcal{E}, \boldsymbol{a}\sim\boldsymbol{\pi}}}
        		\left[ \left( Q^{\boldsymbol{\pi}, i}_{\tilde{\phi}^i_{sh}}(\boldsymbol{o}, \boldsymbol{a}) - Q^{\boldsymbol{\pi}}_{tg} \right)^2\right]\\
        		&=\textstyle\sum_{s} d_{\boldsymbol{\pi}}(s) \textstyle\sum_{\boldsymbol{o} } \mathcal{E}_{s}^{\boldsymbol{o}} \textstyle\sum_{\boldsymbol{a}}\boldsymbol{\pi}_{\Psi_{in}}(\boldsymbol{a}|\boldsymbol{o})\left[\left( Q^{\boldsymbol{\pi},i}_{\tilde{\phi}^i_{sh}}(\boldsymbol{o}, \boldsymbol{a}) - Q^{\boldsymbol{\pi}}_{tg} \right)^2 \right], \nonumber
        	\end{aligned}
        \end{equation}
        where the specific form of $Q^{\boldsymbol{\pi}}_{tg}$ is described in Section 4.2.
        We hypothesis joint policy $\boldsymbol{\pi}_{\Psi_{in}}$ is the product of local policy functions $\prod_{i=1}^n \pi^i_{\psi_{in}^i}$. Hence \textbf{the actor part} gradient w.r.t. each parameter $\psi_{in}^i$ becomes:
	    \begin{equation}
	        \begin{aligned}
    	        &\nabla_{\psi_{in}^i} J^i_{\text{actor}} \left( \mathbf{w}_{in}, \tilde{\mathbf{w}}_{sh}^i \right) 
    	        = \nabla_{\psi_{in}^i} \textstyle\sum_{s} d_{\boldsymbol{\pi}}(s) \textstyle\sum_{\boldsymbol{o}} \mathcal{E}(\boldsymbol{o}|s) \textstyle\sum_{\boldsymbol{a}} \boldsymbol{\pi}_{\Psi_{in}}(\boldsymbol{a}|\boldsymbol{o}) \left[Q^{\boldsymbol{\pi},i}_{\tilde{\phi}^i_{sh}}(\boldsymbol{o},\boldsymbol{a})- \mathcal{B}(\boldsymbol{o}, \boldsymbol{a}^{\setminus i})\right] \\
    	        &= \textstyle\sum_{s} d_{\boldsymbol{\pi}}(s) \textstyle\sum_{\boldsymbol{o}} \mathcal{E}(\boldsymbol{o}|s) \textstyle\sum_{\boldsymbol{a}} \nabla_{\psi_{in}^i} \boldsymbol{\pi}_{\Psi_{in}}(\boldsymbol{a}|\boldsymbol{o}) \left[Q^{\boldsymbol{\pi},i}_{\tilde{\phi}^i_{sh}}(\boldsymbol{o},\boldsymbol{a}) - \mathcal{B}(\boldsymbol{o}, \boldsymbol{a}^{\setminus i})\right] \\
    	        &= \textstyle\sum_{s} d_{\boldsymbol{\pi}}(s) \textstyle\sum_{\boldsymbol{o}} \mathcal{E}(\boldsymbol{o}|s) \textstyle\sum_{\boldsymbol{a}} \boldsymbol{\pi}_{\Psi_{in}}(\boldsymbol{a}|\boldsymbol{o})
    	        \nabla_{\psi_{in}^i}\log\pi_{\psi_{in}^i}^i(a^i|o^i)\left[Q^{\boldsymbol{\pi},i}_{\tilde{\phi}^i_{1,sh}}(\boldsymbol{o},\boldsymbol{a}) - \mathcal{B}(\boldsymbol{o}, \boldsymbol{a}^{\setminus i})\right] \\
    	        &=\mathbf{E}_{s \sim d_{\boldsymbol{\pi}}, \boldsymbol{o} \sim \mathcal{E},\boldsymbol{a}\sim\boldsymbol{\pi}}\Bigg[\nabla_{\psi_{in}^i}\log\pi_{\psi_{in}^i}^i(a^i|o^i)\left(Q^{\boldsymbol{\pi},i}_{\tilde{\phi}^i_{1,sh}}(\boldsymbol{o},\boldsymbol{a}) - \mathcal{B}(\boldsymbol{o}, \boldsymbol{a}^{\setminus i})\right)\Bigg]. \nonumber
	        \end{aligned}
	    \end{equation}
	    \textbf{For the critic part, the gradient sampled by current policy} is calculated by:
	    \begin{equation}
	        \begin{aligned}
    			&\nabla_{\psi_{in}^i} J^i_{\text{critic}} \left( \mathbf{w}_{in}, \tilde{\mathbf{w}}_{sh}^i \right) 
    			= \mathbf{E}_{\substack{s\sim d_{\boldsymbol{\pi}}, \boldsymbol{o}\sim\mathcal{E},\boldsymbol{a}\sim\boldsymbol{\pi}}}\Bigg[\nabla_{\psi_{in}^i}\log\pi_{\psi_{in}^i}^i(a^i|o^i)\delta^2\Bigg]. \nonumber
	        \end{aligned}
	    \end{equation}
	    where $\delta=Q^{\boldsymbol{\pi}, i}_{\tilde{\phi}^i_{sh}}(\boldsymbol{o}, \boldsymbol{a}) - Q^{\boldsymbol{\pi}}_{tg}$.
	    
	    Finally, we get the following on-policy joint gradient w.r.t. actor parameters:
	    \begin{align*}
	            &\nabla_{\psi_{in}^i} J_{ac}^i \left( \mathbf{w}_{in}, \tilde{\mathbf{w}}_{sh}^i \right) 
	            = \alpha_1 \nabla_{\psi_{in}^i} J_{actor}^i \left( \mathbf{w}_{in}, \tilde{\mathbf{w}}_{sh}^i \right) + \alpha_2 \nabla_{\psi_{in}^i} J_{critic}^i \left( \mathbf{w}_{in}, \tilde{\mathbf{w}}_{sh}^i \right) \\ 
	            =& \mathbf{E}_{s \sim d_{\boldsymbol{\pi}}, \boldsymbol{o} \sim \mathcal{E},\boldsymbol{a}\sim\boldsymbol{\pi}}\Bigg[\alpha_1\nabla_{\psi_{in}^i}\log\pi_{\psi_{in}^i}^i(a^i|o^i)\left(Q^{\boldsymbol{\pi},i}_{\tilde{\phi}^i_{sh}}(\boldsymbol{o},\boldsymbol{a}) - \mathcal{B}(\boldsymbol{o}, \boldsymbol{a}^{\setminus i})+\frac{\alpha_2}{\alpha_1}\delta^2\right)\Bigg].
    	        \nonumber
    	\end{align*}
	    The above is the proof of the first equation, \textbf{below we prove the second}.
	    \textbf{For actor part}, we have
	    \begin{equation}
	        \begin{aligned}
    	        \nabla_{\tilde{\phi}^i_{sh}} J^i_{\text{actor}} \left( \mathbf{w}_{in}, \tilde{\mathbf{w}}_{sh}^i \right) 
    	        &= \nabla_{\tilde{\phi}^i_{sh}} \textstyle\sum_{s} d_{\boldsymbol{\pi}}(s) \textstyle\sum_{\boldsymbol{o}} \mathcal{E}(\boldsymbol{o}|s) \textstyle\sum_{\boldsymbol{a}} \boldsymbol{\pi}_{\Psi_{in}}(\boldsymbol{a}|\boldsymbol{o}) \left[Q^{\boldsymbol{\pi},i}_{\tilde{\phi}^i_{sh}}(\boldsymbol{o},\boldsymbol{a}) - \mathcal{B}(\boldsymbol{o}, \boldsymbol{a}^{\setminus i})\right] \\
    	        &= \textstyle\sum_{s} d_{\boldsymbol{\pi}}(s) \textstyle\sum_{\boldsymbol{o}} \mathcal{E}(\boldsymbol{o}|s) \textstyle\sum_{\boldsymbol{a}} \boldsymbol{\pi}_{\Psi_{in}}(\boldsymbol{a}|\boldsymbol{o}) \nabla_{\tilde{\phi}^i_{sh}}Q^{\boldsymbol{\pi},i}_{\tilde{\phi}^i_{sh}}(\boldsymbol{o},\boldsymbol{a})\\
    	        &=\mathbf{E}_{\substack{s\sim d_{\boldsymbol{\pi}}, \boldsymbol{o}\sim\mathcal{E},\boldsymbol{a}\sim\boldsymbol{\pi}}}\left[ \nabla_{\tilde{\phi}^i_{sh}}Q^{\boldsymbol{\pi},i}_{\tilde{\phi}^i_{sh}}(\boldsymbol{o},\boldsymbol{a}) \right]. \nonumber
	        \end{aligned}
	    \end{equation}
	    \textbf{For critic part}, we have
	    \begin{equation}
	        \begin{aligned}
	            \nabla_{\tilde{\phi}^i_{sh}} J^i_{\text{critic}} \left( \mathbf{w}_{in}, \tilde{\mathbf{w}}_{sh}^i \right)  
	            &= \nabla_{\tilde{\phi}^i_{1,sh}} \textstyle\sum_{s} d_{\boldsymbol{\pi}}(s) \textstyle\sum_{\boldsymbol{o}} \mathcal{E}_{s}^{\boldsymbol{o}} \textstyle\sum_{\boldsymbol{a}}\boldsymbol{\pi_{\Psi_{in}}}(\boldsymbol{a}|\boldsymbol{o})\left( Q^{\boldsymbol{\pi},i}_{\tilde{\phi}^i_{sh}}(\boldsymbol{o}, \boldsymbol{a}) - Q^{\boldsymbol{\pi}}_{tg} \right)^2 \\
	            &= \textstyle\sum_{s} d_{\boldsymbol{\pi}}(s) \textstyle\sum_{\boldsymbol{o}} \mathcal{E}_{s}^{\boldsymbol{o}} \textstyle\sum_{\boldsymbol{a}}\boldsymbol{\pi_{\Psi_{in}}}(\boldsymbol{a}|\boldsymbol{o})2\delta\nabla_{\tilde{\phi}^i_{sh}}Q^{\boldsymbol{\pi},i}_{\tilde{\phi}^i_{sh}}(\boldsymbol{o}, \boldsymbol{a})\\
	            &=\mathbf{E}_{\substack{s\sim d_{\boldsymbol{\pi}}, \boldsymbol{o}\sim\mathcal{E},\boldsymbol{a}\sim\boldsymbol{\pi}}}\left[ 2\delta\nabla_{\tilde{\phi}^i_{sh}}Q^{\boldsymbol{\pi},i}_{\tilde{\phi}^i_{sh}}(\boldsymbol{o}, \boldsymbol{a}) \right]. \nonumber
	        \end{aligned}
	    \end{equation}
	    Finally, we get following on-policy joint gradient w.r.t. the first critic parameters
	    \begin{equation}
	    \begin{aligned}
	            \nabla_{\tilde{\phi}^i_{sh}} J_{ac}^i \left( \mathbf{w}_{in}, \tilde{\mathbf{w}}_{sh}^i \right)
	            &= \alpha_1 \nabla_{\tilde{\phi}^i_{sh}} J_{actor}^i \left( \mathbf{w}_{in}, \tilde{\mathbf{w}}_{sh}^i \right) + \alpha_2 \nabla_{\tilde{\phi}^i_{sh}} J_{critic}^i \left( \mathbf{w}_{in}, \tilde{\mathbf{w}}_{sh}^i \right)\\
	            &= \mathbf{E}_{\substack{s\sim d_{\boldsymbol{\pi}}, \boldsymbol{o}\sim\mathcal{E},\boldsymbol{a}\sim\boldsymbol{\pi}}}\left[ (\alpha_1+2\alpha_2\delta)\nabla_{\tilde{\phi}^i_{sh}}Q^{\boldsymbol{\pi},i}_{\tilde{\phi}^i_{sh}}(\boldsymbol{o}, \boldsymbol{a}) \right].
    	        \nonumber
    	\end{aligned}
	    \end{equation}
    \end{proof}

{\color{black}{

\subsection{Proof of Proposition \ref{the:ddpg}}
\begin{proof}
        We proof the first equation first.
        Extend the DDPG algorithm, we have\footnote{\textcolor{black}{Actually, there is no expectation of action because the policy is deterministic, i.e., $\boldsymbol{a}=\boldsymbol{\pi}_{\Psi_{in}}(\boldsymbol{o})$. However, in order to keep the form of the gradient consistent with other algorithms, we still retain the expectation of the action here, which does not affect the calculation of the gradient.}}
        	\begin{align*}
        		J^i_{\text{actor}} \left( \mathbf{w}_{in}, \tilde{\mathbf{w}}_{sh}^i \right) &= \mathbf{E}_{s \sim d_0, \boldsymbol{o} \sim \mathcal{E},\boldsymbol{a}\sim\boldsymbol{\pi} }\left[Q^{\boldsymbol{\pi},i}_{\tilde{\phi}^i_{sh}}(\boldsymbol{o},\boldsymbol{a})\right] \\
        		&= \textstyle\sum_{s} d_{0}(s) \textstyle\sum_{\boldsymbol{o}} \mathcal{E}(\boldsymbol{o}|s) \textstyle\sum_{\boldsymbol{a}} \boldsymbol{\pi}_{\Psi_{in}}(\boldsymbol{a}|\boldsymbol{o}) \left[Q^{\boldsymbol{\pi},i}_{\tilde{\phi}^i_{sh}}(\boldsymbol{o},\boldsymbol{a})\right]
        		\\
        		J^i_{\text{critic}} \left( \mathbf{w}_{in}, \tilde{\mathbf{w}}_{sh}^i \right) &= \mathbf{E}_{\substack{s\sim d_{0}, \boldsymbol{o}\sim\mathcal{E}, \boldsymbol{a}\sim\boldsymbol{\pi_0}}}
        		\left[ \left( Q^{\boldsymbol{\pi}, i}_{\tilde{\phi}^i_{sh}}(\boldsymbol{o}, \boldsymbol{a}) - Q^{\boldsymbol{\pi}}_{tg} \right)^2\right]\\
        		&=\textstyle\sum_{s} d_{0}(s) \textstyle\sum_{\boldsymbol{o} } \mathcal{E}_{0,s}^{\boldsymbol{o}} \textstyle\sum_{\boldsymbol{a}}\boldsymbol{\pi_0}(\boldsymbol{a}|\boldsymbol{o})\left[\left( Q^{\boldsymbol{\pi},i}_{\tilde{\phi}^i_{sh}}(\boldsymbol{o}, \boldsymbol{a}) - Q^{\boldsymbol{\pi}}_{tg} \right)^2\right], \nonumber
        	\end{align*}
        where $Q^{\boldsymbol{\pi}}_{tg} = \sum_{s^{\prime}} \mathcal{P}_{s, \boldsymbol{a}}^{s^{\prime}} \left( \mathcal{C}_{s, \boldsymbol{a}}^{i, s^{\prime}} + \gamma \sum_{\boldsymbol{o}^{\prime}} \mathcal{E}^{\boldsymbol{o}^{\prime}}_{0,s^{\prime}} \sum_{\boldsymbol{a}^{\prime}} \boldsymbol{\pi}_{\Psi_{in}}(\boldsymbol{a}^{\prime}|\boldsymbol{o}^{\prime}) Q^{\boldsymbol{\pi},i}_{\tilde{\phi}^i_{sh}}\left(\boldsymbol{o}^{\prime},\boldsymbol{a}^{\prime}\right) \right)$.
        Note that the joint policy $\boldsymbol{\pi}_{\Psi_{in}}$ is a deterministic policy.
        
        We hypothesis joint policy $\boldsymbol{\pi}_{\Psi_{in}}$ is the product of local policy functions $\prod_{i=1}^n \pi_{\psi_{in}^i}$.
	    Hence the gradient concerning each parameter $\psi_{in}^i$ becomes
	        \begin{align*}
    	        \nabla_{\psi_{in}^i} J^i_{\text{actor}} \left( \mathbf{w}_{in}, \tilde{\mathbf{w}}_{sh}^i \right) 
    	        &= \nabla_{\psi_{in}^i} \textstyle\sum_{s} d_{0}(s) \textstyle\sum_{\boldsymbol{o}} \mathcal{E}(\boldsymbol{o}|s) \textstyle\sum_{\boldsymbol{a}} \boldsymbol{\pi}_{\Psi_{in}}(\boldsymbol{a}|\boldsymbol{o}) \left[Q^{\boldsymbol{\pi},i}_{\tilde{\phi}^i_{sh}}(\boldsymbol{o},\boldsymbol{a}) \right] \\
    	        &= \textstyle\sum_{s} d_{0}(s) \textstyle\sum_{\boldsymbol{o}} \mathcal{E}(\boldsymbol{o}|s) \textstyle\sum_{\boldsymbol{a}} \boldsymbol{\pi}_{\Psi_{in}}(\boldsymbol{a}|\boldsymbol{o}) \left[\nabla_{\psi_{in}^i} Q^{\boldsymbol{\pi},i}_{\tilde{\phi}^i_{sh}}(\boldsymbol{o},\boldsymbol{a}) \right] \\
    	        &= \textstyle\sum_{s} d_{0}(s) \textstyle\sum_{\boldsymbol{o}} \mathcal{E}(\boldsymbol{o}|s) \textstyle\sum_{\boldsymbol{a}} \boldsymbol{\pi}_{\Psi_{in}}(\boldsymbol{a}|\boldsymbol{o}) \left[\nabla_{\psi_{in}^i}\pi^i_{\psi_{in}^i}(a^i|o^i)\nabla_{a^i} Q^{\boldsymbol{\pi},i}_{\tilde{\phi}^i_{sh}}(\boldsymbol{o},\boldsymbol{a}) \right] \\
    	        &=\mathbf{E}_{s \sim d_0, \boldsymbol{o} \sim \mathcal{E},\boldsymbol{a}\sim\boldsymbol{\pi}}\left[\nabla_{\psi_{in}^i}\pi^i_{\psi_{in}^i}(a^i|o^i)\nabla_{a^i} Q^{\boldsymbol{\pi},i}_{\tilde{\phi}^i_{sh}}(\boldsymbol{o},\boldsymbol{a})\right]. \nonumber
	        \end{align*}
	   Note that here we need to resample to calculate the unbias policy gradient so that we can't use the off-policy data directly.
       We solve the above problem by importance sampling.
        \textbf{For off-policy data saved in experience replay buffer} we have
        \begin{equation}
            \begin{aligned}
    	        \nabla_{\psi_{in}^i} J^i_{\text{actor}} \left( \mathbf{w}_{in}, \tilde{\mathbf{w}}_{sh}^i \right) 
    	        =\mathbf{E}_{s \sim d_0, \boldsymbol{o} \sim \mathcal{E},\boldsymbol{a}\sim\boldsymbol{\pi_0}}\left[\left(\frac{\pi^i_{\psi_{in}^i}(a^i|o^i)}{\pi^i_{0}(a^i|o^i)}\right)\nabla_{\psi_{in}^i}\pi^i_{\psi_{in}^i}(a^i|o^i)\nabla_{a^i} Q^{\boldsymbol{\pi},i}_{\tilde{\phi}^i_{sh}}(\boldsymbol{o},\boldsymbol{a})\right]=0. \nonumber
            \end{aligned}
        \end{equation}
        combined \textbf{the gradient calculated by the data which is resampled by current policy}, we get the joint off-policy policy gradient associated with the actor: 
	   \begin{equation}
            \begin{aligned}
    	        \nabla_{\psi_{in}^i} J^i_{\text{actor}} \left( \mathbf{w}_{in}, \tilde{\mathbf{w}}_{sh}^i \right)
    	       = & \mathbf{E}_{s \sim d_0, \boldsymbol{o} \sim \mathcal{E},\boldsymbol{a}\sim\boldsymbol{\pi}}\left[\nabla_{\psi_{in}^i}\pi^i_{\psi_{in}^i}(a^i|o^i)\nabla_{a^i} Q^{\boldsymbol{\pi},i}_{\tilde{\phi}^i_{sh}}(\boldsymbol{o},\boldsymbol{a})\right]. \nonumber
            \end{aligned}
        \end{equation}
        On the contrary, we directly use the off-policy data when calculate the value function gradient, so that we can't use the above resampled data.
	    \textbf{For the critic part}, we use above trick in reverse. 
	    \textbf{For off-policy data saved in experience replay buffer} we have
	    \begin{equation}
	        \begin{aligned}
	            \nabla_{\psi_{in}^i} J^i_{\text{critic}} \left( \mathbf{w}_{in}, \tilde{\mathbf{w}}_{sh}^i \right)
        		= \nabla_{\psi_{in}^i} \textstyle\sum_{s} d_{0}(s) \textstyle\sum_{\boldsymbol{o} } \mathcal{E}_{0,s}^{\boldsymbol{o}} \textstyle\sum_{\boldsymbol{a}}\boldsymbol{\pi_0}(\boldsymbol{a}|\boldsymbol{o})\left[\left( Q^{\boldsymbol{\pi},i}_{\tilde{\phi}^i_{sh}}(\boldsymbol{o}, \boldsymbol{a}) - Q^{\boldsymbol{\pi}}_{tg} \right)^2\right]=0. \nonumber
	        \end{aligned}
	    \end{equation}
	    Combining \textbf{the gradient calculated by data which is resampled by current policy}, we get the joint off-policy policy gradient associates with the critic part\footnote{\textcolor{black}{Since the policy is deterministic, we cannot directly get the specific values of the denominator in the importance ratio during implementation. Fortunately, we can draw on a probabilistic reinforcement learning framework~\citep{levine2018reinforcement} to estimate the denominator}.}:
	    \begin{equation}
	        \begin{aligned}
    			\nabla_{\psi_{in}^i} J^i_{\text{critic}} \left( \mathbf{w}_{in}, \tilde{\mathbf{w}}_{sh}^i \right)
    			= & \mathbf{E}_{\substack{s\sim d_{0}, \boldsymbol{o}\sim\mathcal{E},\boldsymbol{a}\sim\boldsymbol{\pi}}}\Bigg[2\delta\left(\frac{\pi^i_{0}(a^i|o^i)}{\pi^i_{\psi_{in}^i}(a^i|o^i)}\right)\nabla_{\psi_{in}^i}\pi^i_{\psi_{in}^i}(a^i|o^{i})\left(\nabla_{a^i}Q^{\boldsymbol{\pi}, i}_{\tilde{\phi}^i_{sh}}(\boldsymbol{o}, \boldsymbol{a})\right)\Bigg]. \nonumber 
	        \end{aligned}
	    \end{equation}
	    where $\delta=Q^{\boldsymbol{\pi}, i}_{\tilde{\phi}^i_{sh}}(\boldsymbol{o}, \boldsymbol{a}) - Q^{\boldsymbol{\pi}}_{tg}$.
	    We denote the clipped importance sampling term
	    $\min\left(\epsilon,\frac{\pi^i_{\psi_{in}^i}(a^i|o^i)}{\pi^i_{0}(a^i|o^i)}\right)$ as $CIM_{\epsilon}(\pi^i_{\psi_{in}^i};\pi^i_0)$ and 
	    $\min\left(\epsilon,\frac{\pi^i_{0}(a^i|o^i)}{\pi^i_{\psi_{in}^i}(a^i|o^i)}\right)$ as $CIM_{\epsilon}(\pi^i_0;\pi^i_{\psi_{in}^i})$.
	    Finally, we get following off-policy joint gradient w.r.t. actor parameters
	    \begin{equation}
	    \begin{aligned}
	            &\nabla_{\psi_{in}^i} J_{ac}^i \left( \mathbf{w}_{in}, \tilde{\mathbf{w}}_{sh}^i \right)
	            = \alpha_1 \nabla_{\psi_{in}^i} J_{actor}^i \left( \mathbf{w}_{in}, \tilde{\mathbf{w}}_{sh}^i \right) + \alpha_2 \nabla_{\psi_{in}^i} J_{critic}^i \left( \mathbf{w}_{in}, \tilde{\mathbf{w}}_{sh}^i \right) \\
	            &= 
    	        \mathbf{E}_{s \sim d_0, \boldsymbol{o} \sim \mathcal{E},\boldsymbol{a}\sim\boldsymbol{\pi}}\Bigg[(\alpha_1+2\alpha_2\delta \left(\frac{\pi^i_{0}(a^i|o^i)}{\pi^i_{\psi_{in}^i}(a^i|o^i)}\right)\nabla_{\psi_{in}^i}\pi^i_{\psi_{in}^i}(a^i|o^i)\nabla_{a^i} Q^{\boldsymbol{\pi},i}_{\tilde{\phi}^i_{sh}}(\boldsymbol{o},\boldsymbol{a})\Bigg]\\
	            &\simeq 
    	        \mathbf{E}_{s \sim d_0, \boldsymbol{o} \sim \mathcal{E},\boldsymbol{a}\sim\boldsymbol{\pi}}\Bigg[(\alpha_1+2\alpha_2\delta CIM_{\epsilon}(\pi^i_0;\pi^i_{\psi_{in}^i}))\nabla_{\psi_{in}^i}\pi^i_{\psi_{in}^i}(a^i|o^i)\nabla_{a^i} Q^{\boldsymbol{\pi},i}_{\tilde{\phi}^i_{sh}}(\boldsymbol{o},\boldsymbol{a})\Bigg].
    	        \nonumber
    	\end{aligned}
	    \end{equation}
	    The above is the proof of the first equation, \textbf{below we prove the second}.
	    \textbf{For actor part, we first calculate the gradient use the data resampled by the current policy}
	    \begin{equation}
	        \begin{aligned}
    	        \nabla_{\tilde{\phi}^i_{sh}} J^i_{\text{actor}} \left( \mathbf{w}_{in}, \tilde{\mathbf{w}}_{sh}^i \right) 
    	        &= \nabla_{\tilde{\phi}^i_{sh}} \textstyle\sum_{s} d_{0}(s) \textstyle\sum_{\boldsymbol{o}} \mathcal{E}(\boldsymbol{o}|s) \textstyle\sum_{\boldsymbol{a}} \boldsymbol{\pi}_{\Psi_{in}}(\boldsymbol{a}|\boldsymbol{o}) \left[Q^{\boldsymbol{\pi},i}_{\tilde{\phi}^i_{sh}}(\boldsymbol{o},\boldsymbol{a})\right] \\
    	        &= \textstyle\sum_{s} d_{0}(s) \textstyle\sum_{\boldsymbol{o}} \mathcal{E}(\boldsymbol{o}|s) \textstyle\sum_{\boldsymbol{a}} \boldsymbol{\pi}_{\Psi_{in}}(\boldsymbol{a}|\boldsymbol{o}) \nabla_{\tilde{\phi}^i_{sh}}Q^{\boldsymbol{\pi},i}_{\tilde{\phi}^i_{sh}}(\boldsymbol{o},\boldsymbol{a})\\
    	        &=\mathbf{E}_{\substack{s\sim d_{0}, \boldsymbol{o}\sim\mathcal{E},\boldsymbol{a}\sim\boldsymbol{\pi}}}\left[ \nabla_{\tilde{\phi}^i_{sh}}Q^{\boldsymbol{\pi},i}_{\tilde{\phi}^i_{sh}}(\boldsymbol{o},\boldsymbol{a}) \right]. \nonumber
	        \end{aligned}
	    \end{equation}
	    \textbf{For calculate the gradient use off-policy data} we also introduce importance sampling, then we have
	    \begin{equation}
	        \begin{aligned}
    	        \nabla_{\tilde{\phi}^i_{sh}} J^i_{\text{actor}} \left( \mathbf{w}_{in}, \tilde{\mathbf{w}}_{sh}^i \right) 
    	        &=\mathbf{E}_{\substack{s\sim d_{0}, \boldsymbol{o}\sim\mathcal{E},\boldsymbol{a}\sim\boldsymbol{\pi_0}}}\left[ \left(\frac{\pi^i_{\psi_{in}^i}(a^i|o^i)}{\pi^i_{0}(a^i|o^i)}\right)\nabla_{\tilde{\phi}^i_{sh}}Q^{\boldsymbol{\pi},i}_{\tilde{\phi}^i_{sh}}(\boldsymbol{o},\boldsymbol{a}) \right]. \nonumber
	        \end{aligned}
	    \end{equation}
	    We then combine the above two part gradient:
	    \begin{equation}
	        \begin{aligned}
    	        \nabla_{\tilde{\phi}^i_{sh}} J^i_{\text{actor}} \left( \mathbf{w}_{in}, \tilde{\mathbf{w}}_{sh}^i \right) 
    	        &=\mathbf{E}_{\substack{s\sim d_{0}, \boldsymbol{o}\sim\mathcal{E},\boldsymbol{a}\sim\boldsymbol{\pi}}}\left[ \nabla_{\tilde{\phi}^i_{sh}}Q^{\boldsymbol{\pi},i}_{\tilde{\phi}^i_{sh}}(\boldsymbol{o},\boldsymbol{a}) \right]+\\
    	        &\qquad\qquad\mathbf{E}_{\substack{s\sim d_{0}, \boldsymbol{o}\sim\mathcal{E},\boldsymbol{a}\sim\boldsymbol{\pi_0}}}\left[ \left(\frac{\pi^i_{\psi_{in}^i}(a^i|o^i)}{\pi^i_{0}(a^i|o^i)}\right)\nabla_{\tilde{\phi}^i_{sh}}Q^{\boldsymbol{\pi},i}_{\tilde{\phi}^i_{sh}}(\boldsymbol{o},\boldsymbol{a}) \right]. \nonumber
	        \end{aligned}
	    \end{equation}
	    \textbf{For critic part}, we first \textbf{calculate the gradient use the off-policy data}
	    \begin{equation}
	        \begin{aligned}
	            \nabla_{\tilde{\phi}^i_{sh}} J^i_{\text{critic}} \left( \mathbf{w}_{in}, \tilde{\mathbf{w}}_{sh}^i \right)  
	            &= \nabla_{\tilde{\phi}^i_{sh}} \textstyle\sum_{s} d_{0}(s) \textstyle\sum_{\boldsymbol{o}} \mathcal{E}_{0,s}^{\boldsymbol{o}} \textstyle\sum_{\boldsymbol{a}}\boldsymbol{\pi_0}(\boldsymbol{a}|\boldsymbol{o})\left[\left( Q^{\boldsymbol{\pi},i}_{\tilde{\phi}^i_{sh}}(\boldsymbol{o}, \boldsymbol{a}) - Q^{\boldsymbol{\pi}}_{tg} \right)^2\right] \\
	            &= \textstyle\sum_{s} d_{0}(s) \textstyle\sum_{\boldsymbol{o}} \mathcal{E}_{0,s}^{\boldsymbol{o}} \textstyle\sum_{\boldsymbol{a}}\boldsymbol{\pi_0}(\boldsymbol{a}|\boldsymbol{o})2\delta\nabla_{\tilde{\phi}^i_{sh}}Q^{\boldsymbol{\pi},i}_{\tilde{\phi}^i_{sh}}(\boldsymbol{o}, \boldsymbol{a})\\
	            &=\mathbf{E}_{\substack{s\sim d_{0}, \boldsymbol{o}\sim\mathcal{E},\boldsymbol{a}\sim\boldsymbol{\pi_0}}}\left[ 2\delta\nabla_{\tilde{\phi}^i_{sh}}Q^{\boldsymbol{\pi},i}_{\tilde{\phi}^i_{sh}}(\boldsymbol{o}, \boldsymbol{a}) \right]. \nonumber
	        \end{aligned}
	    \end{equation}
	    Then, \textbf{for resampled data}, we have:
	    \begin{equation}
	        \begin{aligned}
	            \nabla_{\tilde{\phi}^i_{sh}} J^i_{\text{critic}} \left( \mathbf{w}_{in}, \tilde{\mathbf{w}}_{sh}^i \right)  
	            &\simeq \mathbf{E}_{\substack{s\sim d_{0}, \boldsymbol{o}\sim\mathcal{E},\boldsymbol{a}\sim\boldsymbol{\pi}}}\left[ 2\delta \left(\frac{\pi^i_{0}(a^i|o^i)}{\pi^i_{\psi_{in}^i}(a^i|o^i)}\right)\nabla_{\tilde{\phi}^i_{sh}}Q^{\boldsymbol{\pi},i}_{\tilde{\phi}^i_{sh}}(\boldsymbol{o}, \boldsymbol{a}) \right]. \nonumber
	        \end{aligned}
	    \end{equation}
	    Combining the above two parts' gradient, we have:
	    \begin{equation}
	       \begin{aligned}
	            \nabla_{\tilde{\phi}^i_{sh}} J^i_{\text{critic}} \left( \mathbf{w}_{in}, \tilde{\mathbf{w}}_{sh}^i \right)  
	            = & \mathbf{E}_{\substack{s\sim d_{0}, \boldsymbol{o}\sim\mathcal{E},\boldsymbol{a}\sim\boldsymbol{\pi_0}}}\left[ 2\delta\nabla_{\tilde{\phi}^i_{sh}}Q^{\boldsymbol{\pi},i}_{\tilde{\phi}^i_{sh}}(\boldsymbol{o}, \boldsymbol{a}) \right]+\\
	            &\qquad\qquad\mathbf{E}_{\substack{s\sim d_{0}, \boldsymbol{o}\sim\mathcal{E},\boldsymbol{a}\sim\boldsymbol{\pi}}}\left[ 2\delta \left(\frac{\pi^i_{0}(a^i|o^i)}{\pi^i_{\psi_{in}^i}(a^i|o^i)}\right)\nabla_{\tilde{\phi}^i_{sh}}Q^{\boldsymbol{\pi},i}_{\tilde{\phi}^i_{sh}}(\boldsymbol{o}, \boldsymbol{a}) \right]. \nonumber
	        \end{aligned}
	    \end{equation}
	    Finally, we get following off-policy joint gradient w.r.t. the critic parameters
	    \begin{equation}
	    \begin{aligned}
	            \nabla_{\tilde{\phi}^i_{sh}} J_{ac}^i \left( \mathbf{w}_{in}, \tilde{\mathbf{w}}_{sh}^i \right)
	            &= \alpha_1 \nabla_{\tilde{\phi}^i_{sh}} J_{actor}^i \left( \mathbf{w}_{in}, \tilde{\mathbf{w}}_{sh}^i \right) + \alpha_2 \nabla_{\tilde{\phi}^i_{sh}} J_{critic}^i \left( \mathbf{w}_{in}, \tilde{\mathbf{w}}_{sh}^i \right)\\
	            &= \mathbf{E}_{\substack{s\sim d_{0}, \boldsymbol{o}\sim\mathcal{E},\boldsymbol{a}\sim\boldsymbol{\pi_0}}}\left[ \left(\alpha_1 \left(\frac{\pi^i_{\psi_{in}^i}(a^i|o^i)}{\pi^i_{0}(a^i|o^i)}\right)+2\alpha_2\delta\right)\nabla_{\tilde{\phi}^i_{sh}}Q^{\boldsymbol{\pi},i}_{\tilde{\phi}^i_{sh}}(\boldsymbol{o}, \boldsymbol{a}) \right]+\\
	            &\qquad\qquad\mathbf{E}_{\substack{s\sim d_{0}, \boldsymbol{o}\sim\mathcal{E},\boldsymbol{a}\sim\boldsymbol{\pi}}}\left[ \left(\alpha_1+2\alpha_2\delta \left(\frac{\pi^i_{0}(a^i|o^i)}{\pi^i_{\psi_{in}^i}(a^i|o^i)}\right)\right)\nabla_{\tilde{\phi}^i_{sh}}Q^{\boldsymbol{\pi},i}_{\tilde{\phi}^i_{sh}}(\boldsymbol{o}, \boldsymbol{a}) \right]\\
	            &\simeq \mathbf{E}_{\substack{s\sim d_{0}, \boldsymbol{o}\sim\mathcal{E},\boldsymbol{a}\sim\boldsymbol{\pi_0}}}\left[ \left(\alpha_1 CIM_{\epsilon}(\pi^i_{\psi_{in}^i};\pi^i_0)+2\alpha_2\delta\right)\nabla_{\tilde{\phi}^i_{sh}}Q^{\boldsymbol{\pi},i}_{\tilde{\phi}^i_{sh}}(\boldsymbol{o}, \boldsymbol{a}) \right]+\\
	            &\qquad\qquad\mathbf{E}_{\substack{s\sim d_{0}, \boldsymbol{o}\sim\mathcal{E},\boldsymbol{a}\sim\boldsymbol{\pi}}}\left[ \left(\alpha_1+2\alpha_2\delta CIM_{\epsilon}(\pi^i_0;\pi^i_{\psi_{in}^i})\right)\nabla_{\tilde{\phi}^i_{sh}}Q^{\boldsymbol{\pi},i}_{\tilde{\phi}^i_{sh}}(\boldsymbol{o}, \boldsymbol{a}) \right].
    	        \nonumber
    	\end{aligned}
	    \end{equation}

    \end{proof}

}}

\section{Extend Proposition~\ref{the:ddpg} to Other Off-policy Algorithm}

{\color{black}{

\subsection{F2A2-TD3: Extend Proposition \ref{the:ddpg} to TD3}
    The Twin-Delayed Deep Deterministic Policy Gradient algorithm is similar to the DDPG algorithm, just adding a twin $Q$-value function to a stable training process except for some tricks for implementation.
    Formally, we can extend it to the fully decentralized multi-agent scenario use the variant of Proposition \ref{the:ddpg}.
    
    \begin{tcolorbox}[breakable, enhanced]
	\begin{prop}
        [Off-Policy TD3-Based Joint Gradient]\label{the:td3}
        We set $\boldsymbol{\pi}_0$ the data collection policy sampled from experience replay buffer, \textcolor{black}{$d_0$ represents the distribution of the state-occupancy measure of policy $\boldsymbol{\pi}_0$} and $\delta$ the TD(0)-error.
        So the gradient of $J_{ac}^i \left( \mathbf{w}_{in}, \tilde{\mathbf{w}}_{sh}^i \right) $ is:
    	\begin{equation}
    	    \begin{aligned}
    	        & \nabla_{\psi_{in}^i} J_{ac}^i \left( \mathbf{w}_{in}, \tilde{\mathbf{w}}_{sh}^i \right) \\
    	       &= \mathbf{E}_{s \sim d_0, \boldsymbol{o} \sim \mathcal{E},\boldsymbol{a}\sim\boldsymbol{\pi}}\Bigg[\nabla_{\psi_{in}^i} \pi^i_{\psi_{in}^i}(a^i|o^i)\left((\alpha_1+2\alpha_2\delta_1\left(\frac{\pi^i_{0}(a^i|o^i)}{\pi^i_{\psi_{in}^i}(a^i|o^i)}\right)) \nabla_{a^i}Q^{\boldsymbol{\pi},i}_{\tilde{\phi}^i_{1,sh}}(\boldsymbol{o},\boldsymbol{a})+\right.\\
    	       &\quad\quad\quad\quad\quad\quad\quad\quad\left. 2\alpha_2\delta_2\left(\frac{\pi^i_{0}(a^i|o^i)}{\pi^i_{\psi_{in}^i}(a^i|o^i)}\right) \nabla_{a^i}Q^{\boldsymbol{\pi},i}_{\tilde{\phi}^i_{2,sh}}(\boldsymbol{o},\boldsymbol{a})\right)\Bigg]. \\
    	        &\nabla_{\tilde{\phi}^i_{1,sh}} J_{ac}^i \left( \mathbf{w}_{in}, \tilde{\mathbf{w}}_{sh}^i \right) \\
	           =& \mathbf{E}_{\substack{s\sim d_{0}, \boldsymbol{o}\sim\mathcal{E},\boldsymbol{a}\sim\boldsymbol{\pi_0}}}\left[ (\alpha_1\left(\frac{\pi^i_{\psi_{in}^i}(a^i|o^i)}{\pi^i_{0}(a^i|o^i)}\right)+2\alpha_2\delta_1)\nabla_{\tilde{\phi}^i_{1,sh}}Q^{\boldsymbol{\pi},i}_{\tilde{\phi}^i_{1,sh}}(\boldsymbol{o}, \boldsymbol{a}) \right]+\\
	            &\qquad\mathbf{E}_{\substack{s\sim d_{0}, \boldsymbol{o}\sim\mathcal{E},\boldsymbol{a}\sim\boldsymbol{\pi}}}\left[ (\alpha_1+2\alpha_2\delta_1\left(\frac{\pi^i_{0}(a^i|o^i)}{\pi^i_{\psi_{in}^i}(a^i|o^i)}\right)) \nabla_{\tilde{\phi}^i_{1,sh}}Q^{\boldsymbol{\pi},i}_{\tilde{\phi}^i_{1,sh}}(\boldsymbol{o}, \boldsymbol{a}) \right] \\
	            & \nabla_{\tilde{\phi}^i_{2,sh}} J_{ac}^i \left( \mathbf{w}_{in}, \tilde{\mathbf{w}}_{sh}^i \right) \\
	           =& \mathbf{E}_{\substack{s\sim d_{0}, \boldsymbol{o}\sim\mathcal{E},\boldsymbol{a}\sim\boldsymbol{\pi_0}}}\left[ 2\alpha_2\delta_2\nabla_{\tilde{\phi}^i_{2,sh}}Q^{\boldsymbol{\pi},i}_{\tilde{\phi}^i_{2,sh}}(\boldsymbol{o}, \boldsymbol{a}) \right]+\\
	           &\qquad\mathbf{E}_{\substack{s\sim d_{0}, \boldsymbol{o}\sim\mathcal{E},\boldsymbol{a}\sim\boldsymbol{\pi}}}\left[ 2\alpha_2\delta_2 \left(\frac{\pi^i_{0}(a^i|o^i)}{\pi^i_{\psi_{in}^i}(a^i|o^i)}\right)\nabla_{\tilde{\phi}^i_{2,sh}}Q^{\boldsymbol{\pi},i}_{\tilde{\phi}^i_{2,sh}}(\boldsymbol{o}, \boldsymbol{a}) \right].
    	        \nonumber
    	    \end{aligned}
    	\end{equation}
    \end{prop}
    \end{tcolorbox}
	
	\begin{proof}
        We proof the first equation first.
        Extend the TD3 algorithm, we have
        \begin{equation}
        	\begin{aligned}
                &J^i_{\text{actor}} \left( \mathbf{w}_{in}, \tilde{\mathbf{w}}_{sh}^i \right) = \mathbf{E}_{s \sim d_0, \boldsymbol{o} \sim \mathcal{E},\boldsymbol{a}\sim\boldsymbol{\pi} }\left[Q^{\boldsymbol{\pi},i}_{\tilde{\phi}^i_{1,sh}}(\boldsymbol{o},\boldsymbol{a})\right] \\
        		&= \textstyle\sum_{s} d_{0}(s) \textstyle\sum_{\boldsymbol{o}} \mathcal{E}(\boldsymbol{o}|s) \textstyle\sum_{\boldsymbol{a}} \boldsymbol{\pi}_{\Psi_{in}}(\boldsymbol{a}|\boldsymbol{o}) \left[Q^{\boldsymbol{\pi},i}_{\tilde{\phi}^i_{1,sh}}(\boldsymbol{o},\boldsymbol{a})\right]
        		\\
        		&J^i_{\text{critic}} \left( \mathbf{w}_{in}, \tilde{\mathbf{w}}_{sh}^i \right) = \mathbf{E}_{\substack{s\sim d_{0}, \boldsymbol{o}\sim\mathcal{E}, \boldsymbol{a}\sim\boldsymbol{\pi_0}}}
        		\left[ \left( Q^{\boldsymbol{\pi}, i}_{\tilde{\phi}^i_{1,sh}}(\boldsymbol{o}, \boldsymbol{a}) - Q^{\boldsymbol{\pi}}_{tg} \right)^2 + \left( Q^{\boldsymbol{\pi}, i}_{\tilde{\phi}^i_{2,sh}}(\boldsymbol{o}, \boldsymbol{a}) - Q^{\boldsymbol{\pi}}_{tg} \right)^2 \right]\\
        		&=\textstyle\sum_{s} d_{0}(s) \textstyle\sum_{\boldsymbol{o} } \mathcal{E}_{0,s}^{\boldsymbol{o}} \textstyle\sum_{\boldsymbol{a}}\boldsymbol{\pi_0}(\boldsymbol{a}|\boldsymbol{o})\left[\left( Q^{\boldsymbol{\pi},i}_{\tilde{\phi}^i_{1,sh}}(\boldsymbol{o}, \boldsymbol{a}) - Q^{\boldsymbol{\pi}}_{tg} \right)^2 + \left( Q^{\boldsymbol{\pi},i}_{\tilde{\phi}^i_{2,sh}}(\boldsymbol{o}, \boldsymbol{a}) - Q^{\boldsymbol{\pi}}_{tg} \right)^2\right], \nonumber
        	\end{aligned}
        \end{equation}
        where $Q^{\boldsymbol{\pi}}_{tg} = \sum_{s^{\prime}} \mathcal{P}_{s, \boldsymbol{a}}^{s^{\prime}} \left( \mathcal{C}_{s, \boldsymbol{a}}^{i, s^{\prime}} + \gamma \sum_{\boldsymbol{o}^{\prime}} \mathcal{E}^{\boldsymbol{o}^{\prime}}_{0,s^{\prime}} \sum_{\boldsymbol{a}^{\prime}} \boldsymbol{\pi}_{\Psi_{in}}(\boldsymbol{a}^{\prime}|\boldsymbol{o}^{\prime})\left(\min_{j=1,2} Q^{\boldsymbol{\pi},i}_{\tilde{\phi}^i_{j,sh}}\left(\boldsymbol{o}^{\prime},\boldsymbol{a}^{\prime}\right)\right) \right)$.
        Note that the joint policy $\boldsymbol{\pi}_{\Psi_{in}}$ is a deterministic policy.
        
        We hypothesis joint policy $\boldsymbol{\pi}_{\Psi_{in}}$ is the product of local policy functions $\prod_{i=1}^n \pi_{\psi_{in}^i}$.
	    Hence the gradient concerning  each parameter $\psi_{in}^i$ becomes
	        \begin{align*}
    	        \nabla_{\psi_{in}^i} J^i_{\text{actor}} \left( \mathbf{w}_{in}, \tilde{\mathbf{w}}_{sh}^i \right) 
    	        &= \nabla_{\psi_{in}^i} \textstyle\sum_{s} d_{0}(s) \textstyle\sum_{\boldsymbol{o}} \mathcal{E}(\boldsymbol{o}|s) \textstyle\sum_{\boldsymbol{a}} \boldsymbol{\pi}_{\Psi_{in}}(\boldsymbol{a}|\boldsymbol{o}) \left[Q^{\boldsymbol{\pi},i}_{\tilde{\phi}^i_{1,sh}}(\boldsymbol{o},\boldsymbol{a}) \right] \\
    	        &= \textstyle\sum_{s} d_{0}(s) \textstyle\sum_{\boldsymbol{o}} \mathcal{E}(\boldsymbol{o}|s) \textstyle\sum_{\boldsymbol{a}} \boldsymbol{\pi}_{\Psi_{in}}(\boldsymbol{a}|\boldsymbol{o}) \left[\nabla_{\psi_{in}^i} Q^{\boldsymbol{\pi},i}_{\tilde{\phi}^i_{1,sh}}(\boldsymbol{o},\boldsymbol{a}) \right] \\
    	        &= \textstyle\sum_{s} d_{0}(s) \textstyle\sum_{\boldsymbol{o}} \mathcal{E}(\boldsymbol{o}|s) \textstyle\sum_{\boldsymbol{a}} \boldsymbol{\pi}_{\Psi_{in}}(\boldsymbol{a}|\boldsymbol{o}) \left[\nabla_{\psi_{in}^i}\pi^i_{\psi_{in}^i}(a^i|o^i)\nabla_{a^i} Q^{\boldsymbol{\pi},i}_{\tilde{\phi}^i_{1,sh}}(\boldsymbol{o},\boldsymbol{a}) \right] \\
    	        &=\mathbf{E}_{s \sim d_0, \boldsymbol{o} \sim \mathcal{E},\boldsymbol{a}\sim\boldsymbol{\pi}}\left[\nabla_{\psi_{in}^i}\pi^i_{\psi_{in}^i}(a^i|o^i)\nabla_{a^i} Q^{\boldsymbol{\pi},i}_{\tilde{\phi}^i_{1,sh}}(\boldsymbol{o},\boldsymbol{a})\right]. \nonumber
	        \end{align*}
	    Note that we need to resample to calculate the unbiased policy gradient so that we cannot use the off-policy data directly. We solve the above problem by importance sampling.
        \textbf{For off-policy data saved in experience replay buffer} we have:
        \begin{equation}
    	        \nabla_{\psi_{in}^i} J^i_{\text{actor}} \left( \mathbf{w}_{in}, \tilde{\mathbf{w}}_{sh}^i \right)
    	        =\mathbf{E}_{s \sim d_0, \boldsymbol{o} \sim \mathcal{E},\boldsymbol{a}\sim\boldsymbol{\pi_0}}\left[\left(\frac{\pi^i_{\psi_{in}^i}(a^i|o^i)}{\pi^i_{0}(a^i|o^i)}\right)\nabla_{\psi_{in}^i}\pi^i_{\psi_{in}^i}(a^i|o^i)\nabla_{a^i} Q^{\boldsymbol{\pi},i}_{\tilde{\phi}^i_{1,sh}}(\boldsymbol{o},\boldsymbol{a})\right]=0. \nonumber
        \end{equation}
        Combining \textbf{the gradient calculated by data which is resampled by current policy}, we get the joint
        off-policy policy gradient associates with the actor 
	   \begin{equation}
            \begin{aligned}
    	        \nabla_{\psi_{in}^i} J^i_{\text{actor}} \left( \mathbf{w}_{in}, \tilde{\mathbf{w}}_{sh}^i \right) 
    	        = & 
    	        \mathbf{E}_{s \sim d_0, \boldsymbol{o} \sim \mathcal{E},\boldsymbol{a}\sim\boldsymbol{\pi}}\left[\nabla_{\psi_{in}^i}\pi^i_{\psi_{in}^i}(a^i|o^i)\nabla_{a^i} Q^{\boldsymbol{\pi},i}_{\tilde{\phi}^i_{1,sh}}(\boldsymbol{o},\boldsymbol{a})\right]. \nonumber
            \end{aligned}
        \end{equation}
        On the contrary, we directly use the off-policy data when calculating the value function gradient so that we cannot use the above-resampled data.
	    \textbf{For the critic part}, we use the above trick in reverse. 
	    \textbf{For off-policy data saved in experience replay buffer} we have
	    \begin{equation}
	    \begin{aligned}
	            &\nabla_{\psi_{in}^i} J^i_{\text{critic}} \left( \mathbf{w}_{in}, \tilde{\mathbf{w}}_{sh}^i \right)  \\
	            &= \nabla_{\psi_{in}^i}\textstyle\sum_{s} d_{0}(s) \textstyle\sum_{\boldsymbol{o}} \mathcal{E}_{0,s}^{\boldsymbol{o}} \textstyle\sum_{\boldsymbol{a}}\boldsymbol{\pi_0}(\boldsymbol{a}|\boldsymbol{o})\left[\left( Q^{\boldsymbol{\pi},i}_{\tilde{\phi}^i_{1,sh}}(\boldsymbol{o}, \boldsymbol{a}) - Q^{\boldsymbol{\pi}}_{tg} \right)^2+\left( Q^{\boldsymbol{\pi},i}_{\tilde{\phi}^i_{2,sh}}(\boldsymbol{o}, \boldsymbol{a}) - Q^{\boldsymbol{\pi}}_{tg} \right)^2\right] \\
	            &= 0, \nonumber
	    \end{aligned}
	    \end{equation}
	    combined \textbf{the gradient calculated by the data which is resampled by current policy}, we get the joint off-policy policy gradient associates with the critic part 
	    \begin{equation}
	        \begin{aligned}
    			&\nabla_{\psi_{in}^i} J^i_{\text{critic}} \left( \mathbf{w}_{in}, \tilde{\mathbf{w}}_{sh}^i \right) \\
    			&= \mathbf{E}_{\substack{s\sim d_{0}, \boldsymbol{o}\sim\mathcal{E},\boldsymbol{a}\sim\boldsymbol{\pi}}}\Bigg[2\left(\frac{\pi^i_{0}(a^i|o^i)}{\pi^i_{\psi_{in}^i}(a^i|o^i)}\right)\nabla_{\psi_{in}^i}\pi^i_{\psi_{in}^i}(a^i|o^{i})\left(\delta_1\nabla_{a^i}Q^{\boldsymbol{\pi}, i}_{\tilde{\phi}^i_{1,sh}}(\boldsymbol{o}, \boldsymbol{a})+\delta_2\nabla_{a^i}Q^{\boldsymbol{\pi}, i}_{\tilde{\phi}^i_{2,sh}}(\boldsymbol{o}, \boldsymbol{a})\right)\Bigg]. \nonumber
	        \end{aligned}
	    \end{equation}
	    where $\delta_j=Q^{\boldsymbol{\pi}, i}_{\tilde{\phi}^i_{j,sh}}(\boldsymbol{o}, \boldsymbol{a}) - Q^{\boldsymbol{\pi}}_{tg}$.
	    We denote the clipped importance sampling term
	    $\min\left(\epsilon,\frac{\pi^i_{\psi_{in}^i}(a^i|o^i)}{\pi^i_{0}(a^i|o^i)}\right)$ as $CIM_{\epsilon}(\pi^i_{\psi_{in}^i};\pi^i_0)$ and 
	    $\min\left(\epsilon,\frac{\pi^i_{0}(a^i|o^i)}{\pi^i_{\psi_{in}^i}(a^i|o^i)}\right)$ as $CIM_{\epsilon}(\pi^i_0;\pi^i_{\psi_{in}^i})$.
	    Finally, we get following off-policy joint gradient w.r.t. actor parameters:
	    \begin{equation}
	    \begin{aligned}
	            &\nabla_{\psi_{in}^i} J_{ac}^i \left( \mathbf{w}_{in}, \tilde{\mathbf{w}}_{sh}^i \right) = \alpha_1 \nabla_{\psi_{in}^i} J_{actor}^i \left( \mathbf{w}_{in}, \tilde{\mathbf{w}}_{sh}^i \right) + \alpha_2 \nabla_{\psi_{in}^i} J_{critic}^i \left( \mathbf{w}_{in}, \tilde{\mathbf{w}}_{sh}^i \right) \\ 
    	       &= \mathbf{E}_{s \sim d_0, \boldsymbol{o} \sim \mathcal{E},\boldsymbol{a}\sim\boldsymbol{\pi}}\Bigg[\nabla_{\psi_{in}^i} \pi^i_{\psi_{in}^i}(a^i|o^i)\left((\alpha_1+2\alpha_2\delta_1\left(\frac{\pi^i_{0}(a^i|o^i)}{\pi^i_{\psi_{in}^i}(a^i|o^i)}\right)) \nabla_{a^i}Q^{\boldsymbol{\pi},i}_{\tilde{\phi}^i_{1,sh}}(\boldsymbol{o},\boldsymbol{a})+\right.\\
    	       &\quad\quad\quad\quad\quad\quad\quad\quad\left. 2\alpha_2\delta_2\left(\frac{\pi^i_{0}(a^i|o^i)}{\pi^i_{\psi_{in}^i}(a^i|o^i)}\right) \nabla_{a^i}Q^{\boldsymbol{\pi},i}_{\tilde{\phi}^i_{2,sh}}(\boldsymbol{o},\boldsymbol{a})\right)\Bigg]. \\
    	       &\simeq \mathbf{E}_{s \sim d_0, \boldsymbol{o} \sim \mathcal{E},\boldsymbol{a}\sim\boldsymbol{\pi}}\Bigg[\nabla_{\psi_{in}^i} \pi^i_{\psi_{in}^i}(a^i|o^i)\left((\alpha_1+2\alpha_2\delta_1CIM_{\epsilon}(\pi^i_0;\pi^i_{\psi_{in}^i})) \nabla_{a^i}Q^{\boldsymbol{\pi},i}_{\tilde{\phi}^i_{1,sh}}(\boldsymbol{o},\boldsymbol{a})+\right.\\
    	       &\quad\quad\quad\quad\quad\quad\quad\quad\left.2\alpha_2\delta_2CIM_{\epsilon}(\pi^i_0;\pi^i_{\psi_{in}^i}) \nabla_{a^i}Q^{\boldsymbol{\pi},i}_{\tilde{\phi}^i_{2,sh}}(\boldsymbol{o},\boldsymbol{a})\right)\Bigg].
    	        \nonumber
    	\end{aligned}
	    \end{equation}
	    The above is the proof of the first equation, \textbf{below we prove the second}.
	    \textbf{For actor part, we first calculate the gradient use the data resampled by the current policy}
	    \begin{equation}
	        \begin{aligned}
    	        \nabla_{\tilde{\phi}^i_{1,sh}} J^i_{\text{actor}} \left( \mathbf{w}_{in}, \tilde{\mathbf{w}}_{sh}^i \right) 
    	        &= \nabla_{\tilde{\phi}^i_{1,sh}} \textstyle\sum_{s} d_{0}(s) \textstyle\sum_{\boldsymbol{o}} \mathcal{E}(\boldsymbol{o}|s) \textstyle\sum_{\boldsymbol{a}} \boldsymbol{\pi}_{\Psi_{in}}(\boldsymbol{a}|\boldsymbol{o}) \left[Q^{\boldsymbol{\pi},i}_{\tilde{\phi}^i_{1,sh}}(\boldsymbol{o},\boldsymbol{a})\right] \\
    	        &= \textstyle\sum_{s} d_{0}(s) \textstyle\sum_{\boldsymbol{o}} \mathcal{E}(\boldsymbol{o}|s) \textstyle\sum_{\boldsymbol{a}} \boldsymbol{\pi}_{\Psi_{in}}(\boldsymbol{a}|\boldsymbol{o}) \nabla_{\tilde{\phi}^i_{1,sh}}Q^{\boldsymbol{\pi},i}_{\tilde{\phi}^i_{1,sh}}(\boldsymbol{o},\boldsymbol{a})\\
    	        &=\mathbf{E}_{\substack{s\sim d_{0}, \boldsymbol{o}\sim\mathcal{E},\boldsymbol{a}\sim\boldsymbol{\pi}}}\left[ \nabla_{\tilde{\phi}^i_{1,sh}}Q^{\boldsymbol{\pi},i}_{\tilde{\phi}^i_{1,sh}}(\boldsymbol{o},\boldsymbol{a}) \right]. \nonumber
	        \end{aligned}
	    \end{equation}
	    \textbf{For calculating the gradient use off-policy data} we also introduce importance sampling, then we have
	    \begin{equation}
	        \begin{aligned}
    	        \nabla_{\tilde{\phi}^i_{1,sh}} J^i_{\text{actor}} \left( \mathbf{w}_{in}, \tilde{\mathbf{w}}_{sh}^i \right) 
    	        &=\mathbf{E}_{\substack{s\sim d_{0}, \boldsymbol{o}\sim\mathcal{E},\boldsymbol{a}\sim\boldsymbol{\pi_0}}}\left[ \left(\frac{\pi^i_{\psi_{in}^i}(a^i|o^i)}{\pi^i_{0}(a^i|o^i)}\right)\nabla_{\tilde{\phi}^i_{1,sh}}Q^{\boldsymbol{\pi},i}_{\tilde{\phi}^i_{1,sh}}(\boldsymbol{o},\boldsymbol{a}) \right]. \nonumber
	        \end{aligned}
	    \end{equation}
	    We then combine the above two part gradient:
	    \begin{equation}
	        \begin{aligned}
    	        \nabla_{\tilde{\phi}^i_{1,sh}} J^i_{\text{actor}} \left( \mathbf{w}_{in}, \tilde{\mathbf{w}}_{sh}^i \right) 
    	        &=\mathbf{E}_{\substack{s\sim d_{0}, \boldsymbol{o}\sim\mathcal{E},\boldsymbol{a}\sim\boldsymbol{\pi}}}\left[ \nabla_{\tilde{\phi}^i_{1,sh}}Q^{\boldsymbol{\pi},i}_{\tilde{\phi}^i_{1,sh}}(\boldsymbol{o},\boldsymbol{a}) \right]+\\
    	        &\qquad\qquad\mathbf{E}_{\substack{s\sim d_{0}, \boldsymbol{o}\sim\mathcal{E},\boldsymbol{a}\sim\boldsymbol{\pi_0}}}\left[ \left(\frac{\pi^i_{\psi_{in}^i}(a^i|o^i)}{\pi^i_{0}(a^i|o^i)}\right)\nabla_{\tilde{\phi}^i_{1,sh}}Q^{\boldsymbol{\pi},i}_{\tilde{\phi}^i_{1,sh}}(\boldsymbol{o},\boldsymbol{a}) \right]. \nonumber
	        \end{aligned}
	    \end{equation}
	    \textbf{For critic part}, we first \textbf{calculate the gradient use the off-policy data}
	    \begin{equation}
	        \begin{aligned}
	            &\nabla_{\tilde{\phi}^i_{1,sh}} J^i_{\text{critic}} \left( \mathbf{w}_{in}, \tilde{\mathbf{w}}_{sh}^i \right) \\ 
	            =& \nabla_{\tilde{\phi}^i_{1,sh}} \textstyle\sum_{s} d_{0}(s) \textstyle\sum_{\boldsymbol{o}} \mathcal{E}_{0,s}^{\boldsymbol{o}} \textstyle\sum_{\boldsymbol{a}}\boldsymbol{\pi_0}(\boldsymbol{a}|\boldsymbol{o})\left(\left( Q^{\boldsymbol{\pi},i}_{\tilde{\phi}^i_{1,sh}}(\boldsymbol{o}, \boldsymbol{a}) - Q^{\boldsymbol{\pi}}_{tg} \right)^2+\left( Q^{\boldsymbol{\pi},i}_{\tilde{\phi}^i_{2,sh}}(\boldsymbol{o}, \boldsymbol{a}) - Q^{\boldsymbol{\pi}}_{tg} \right)^2\right) \\
	            =& \textstyle\sum_{s} d_{0}(s) \textstyle\sum_{\boldsymbol{o}} \mathcal{E}_{0,s}^{\boldsymbol{o}} \textstyle\sum_{\boldsymbol{a}}\boldsymbol{\pi_0}(\boldsymbol{a}|\boldsymbol{o})2\delta_1\nabla_{\tilde{\phi}^i_{1,sh}}Q^{\boldsymbol{\pi},i}_{\tilde{\phi}^i_{1,sh}}(\boldsymbol{o}, \boldsymbol{a})\\
	            =&\mathbf{E}_{\substack{s\sim d_{0}, \boldsymbol{o}\sim\mathcal{E},\boldsymbol{a}\sim\boldsymbol{\pi_0}}}\left[ 2\delta_1\nabla_{\tilde{\phi}^i_{1,sh}}Q^{\boldsymbol{\pi},i}_{\tilde{\phi}^i_{1,sh}}(\boldsymbol{o}, \boldsymbol{a}) \right]. \nonumber
	        \end{aligned}
	    \end{equation}
	    Then, \textbf{for resampled data}, we have
	    \begin{equation}
	        \begin{aligned}
	            \nabla_{\tilde{\phi}^i_{1,sh}} J^i_{\text{critic}} \left( \mathbf{w}_{in}, \tilde{\mathbf{w}}_{sh}^i \right)  
	            &= \mathbf{E}_{\substack{s\sim d_{0}, \boldsymbol{o}\sim\mathcal{E},\boldsymbol{a}\sim\boldsymbol{\pi}}}\left[ 2\delta_1 \left(\frac{\pi^i_{0}(a^i|o^i)}{\pi^i_{\psi_{in}^i}(a^i|o^i)}\right)\nabla_{\tilde{\phi}^i_{1,sh}}Q^{\boldsymbol{\pi},i}_{\tilde{\phi}^i_{1,sh}}(\boldsymbol{o}, \boldsymbol{a}) \right]. \nonumber
	        \end{aligned}
	    \end{equation}
	    Combining above two parts' gradient, we have:
	    \begin{equation}
	        \begin{aligned}
	            \nabla_{\tilde{\phi}^i_{1,sh}} J^i_{\text{critic}} \left( \mathbf{w}_{in}, \tilde{\mathbf{w}}_{sh}^i \right)  
	            =& \mathbf{E}_{\substack{s\sim d_{0}, \boldsymbol{o}\sim\mathcal{E},\boldsymbol{a}\sim\boldsymbol{\pi_0}}}\left[ 2\delta_1\nabla_{\tilde{\phi}^i_{1,sh}}Q^{\boldsymbol{\pi},i}_{\tilde{\phi}^i_{1,sh}}(\boldsymbol{o}, \boldsymbol{a}) \right]+\\
	            &\qquad\qquad \mathbf{E}_{\substack{s\sim d_{0}, \boldsymbol{o}\sim\mathcal{E},\boldsymbol{a}\sim\boldsymbol{\pi}}}\left[ 2\delta_1 \left(\frac{\pi^i_{0}(a^i|o^i)}{\pi^i_{\psi_{in}^i}(a^i|o^i)}\right)\nabla_{\tilde{\phi}^i_{1,sh}}Q^{\boldsymbol{\pi},i}_{\tilde{\phi}^i_{1,sh}}(\boldsymbol{o}, \boldsymbol{a}) \right]. \nonumber
	        \end{aligned}
	    \end{equation}
	    Finally, we get the following off-policy joint gradient w.r.t. the first critic parameters:
	    \begin{equation}
	    \begin{aligned}
	           \nabla_{\tilde{\phi}^i_{1,sh}} J_{ac}^i \left( \mathbf{w}_{in}, \tilde{\mathbf{w}}_{sh}^i \right)
	            =  &\alpha_1 \nabla_{\tilde{\phi}^i_{1,sh}} J_{actor}^i \left( \mathbf{w}_{in}, \tilde{\mathbf{w}}_{sh}^i \right) + \alpha_2 \nabla_{\tilde{\phi}^i_{1,sh}} J_{critic}^i \left( \mathbf{w}_{in}, \tilde{\mathbf{w}}_{sh}^i \right)\\
	            =& \mathbf{E}_{\substack{s\sim d_{0}, \boldsymbol{o}\sim\mathcal{E},\boldsymbol{a}\sim\boldsymbol{\pi_0}}}\left[ (\alpha_1\left(\frac{\pi^i_{\psi_{in}^i}(a^i|o^i)}{\pi^i_{0}(a^i|o^i)}\right)+2\alpha_2\delta_1)\nabla_{\tilde{\phi}^i_{1,sh}}Q^{\boldsymbol{\pi},i}_{\tilde{\phi}^i_{1,sh}}(\boldsymbol{o}, \boldsymbol{a}) \right]+\\
	            &\qquad\mathbf{E}_{\substack{s\sim d_{0}, \boldsymbol{o}\sim\mathcal{E},\boldsymbol{a}\sim\boldsymbol{\pi}}}\left[ (\alpha_1+2\alpha_2\delta_1\left(\frac{\pi^i_{0}(a^i|o^i)}{\pi^i_{\psi_{in}^i}(a^i|o^i)}\right)) \nabla_{\tilde{\phi}^i_{1,sh}}Q^{\boldsymbol{\pi},i}_{\tilde{\phi}^i_{1,sh}}(\boldsymbol{o}, \boldsymbol{a}) \right] \\
	            \simeq& \mathbf{E}_{\substack{s\sim d_{0}, \boldsymbol{o}\sim\mathcal{E},\boldsymbol{a}\sim\boldsymbol{\pi_0}}}\left[ (\alpha_1 CIM_{\epsilon}(\pi^i_{\psi_{in}^i};\pi^i_0) +2\alpha_2\delta_1)\nabla_{\tilde{\phi}^i_{1,sh}}Q^{\boldsymbol{\pi},i}_{\tilde{\phi}^i_{1,sh}}(\boldsymbol{o}, \boldsymbol{a}) \right]+\\
	            &\qquad\mathbf{E}_{\substack{s\sim d_{0}, \boldsymbol{o}\sim\mathcal{E},\boldsymbol{a}\sim\boldsymbol{\pi}}}\left[ (\alpha_1+2\alpha_2\delta_1 CIM_{\epsilon}(\pi^i_0;\pi^i_{\psi_{in}^i})) \nabla_{\tilde{\phi}^i_{1,sh}}Q^{\boldsymbol{\pi},i}_{\tilde{\phi}^i_{1,sh}}(\boldsymbol{o}, \boldsymbol{a}) \right].
    	        \nonumber
    	\end{aligned}
	    \end{equation}
	    \textbf{For the third equation and use resampled data to calculate the actor part gradient}, we have
	    \begin{equation}
	        \begin{aligned}
    	        \nabla_{\tilde{\phi}^i_{2,sh}} J^i_{\text{actor}} \left( \mathbf{w}_{in}, \tilde{\mathbf{w}}_{sh}^i \right) 
    	        &= \nabla_{\tilde{\phi}^i_{2,sh}} \textstyle\sum_{s} d_{0}(s) \textstyle\sum_{\boldsymbol{o}} \mathcal{E}(\boldsymbol{o}|s) \textstyle\sum_{\boldsymbol{a}} \boldsymbol{\pi}_{\Psi_{in}}(\boldsymbol{a}|\boldsymbol{o}) \left[Q^{\boldsymbol{\pi},i}_{\tilde{\phi}^i_{1,sh}}(\boldsymbol{o},\boldsymbol{a})\right] = 0. \nonumber
	        \end{aligned}
	    \end{equation}
	    Similarly, \textbf{for off-policy data} the gradient also is 0.
	    For \textbf{critic part} the gradient is similar as the second equation
	    \begin{equation}
	        \begin{aligned}
	            \nabla_{\tilde{\phi}^i_{2,sh}} J^i_{\text{critic}} \left( \mathbf{w}_{in}, \tilde{\mathbf{w}}_{sh}^i \right)  
	            =& \mathbf{E}_{\substack{s\sim d_{0}, \boldsymbol{o}\sim\mathcal{E},\boldsymbol{a}\sim\boldsymbol{\pi_0}}}\left[ 2\delta_2\nabla_{\tilde{\phi}^i_{2,sh}}Q^{\boldsymbol{\pi},i}_{\tilde{\phi}^i_{2,sh}}(\boldsymbol{o}, \boldsymbol{a}) \right]+\\
	            &\qquad\qquad\mathbf{E}_{\substack{s\sim d_{0}, \boldsymbol{o}\sim\mathcal{E},\boldsymbol{a}\sim\boldsymbol{\pi}}}\left[ 2\delta_2 \left(\frac{\pi^i_{0}(a^i|o^i)}{\pi^i_{\psi_{in}^i}(a^i|o^i)}\right)\nabla_{\tilde{\phi}^i_{2,sh}}Q^{\boldsymbol{\pi},i}_{\tilde{\phi}^i_{2,sh}}(\boldsymbol{o}, \boldsymbol{a}) \right]. \nonumber
	        \end{aligned}
	    \end{equation}
	    Finally, we get following off-policy joint gradient w.r.t. the second critic parameters:
	    \begin{equation}
	    \begin{aligned}
	            &\nabla_{\tilde{\phi}^i_{2,sh}} J_{ac}^i \left( \mathbf{w}_{in}, \tilde{\mathbf{w}}_{sh}^i \right) = \alpha_1 \nabla_{\tilde{\phi}^i_{2,sh}} J_{actor}^i \left( \mathbf{w}_{in}, \tilde{\mathbf{w}}_{sh}^i \right) + \alpha_2 \nabla_{\tilde{\phi}^i_{2,sh}} J_{critic}^i \left( \mathbf{w}_{in}, \tilde{\mathbf{w}}_{sh}^i \right)\\
	            =& \mathbf{E}_{\substack{s\sim d_{0}, \boldsymbol{o}\sim\mathcal{E},\boldsymbol{a}\sim\boldsymbol{\pi_0}}}\left[ 2\alpha_2\delta_2\nabla_{\tilde{\phi}^i_{2,sh}}Q^{\boldsymbol{\pi},i}_{\tilde{\phi}^i_{2,sh}}(\boldsymbol{o}, \boldsymbol{a}) \right]+\\
                &\qquad\mathbf{E}_{\substack{s\sim d_{0}, \boldsymbol{o}\sim\mathcal{E},\boldsymbol{a}\sim\boldsymbol{\pi}}}\left[ 2\alpha_2\delta_2 \left(\frac{\pi^i_{0}(a^i|o^i)}{\pi^i_{\psi_{in}^i}(a^i|o^i)}\right)\nabla_{\tilde{\phi}^i_{2,sh}}Q^{\boldsymbol{\pi},i}_{\tilde{\phi}^i_{2,sh}}(\boldsymbol{o}, \boldsymbol{a}) \right].
    	        \nonumber
    	\end{aligned}
	    \end{equation}

    \end{proof}
    
}}

{\color{black}{
\subsection{F2A2-SAC: Extend Proposition \ref{the:ddpg} to SAC}
    Soft Actor Critic (SAC) is an algorithm that optimizes a stochastic policy in an off-policy way.
    A central feature of SAC is entropy regularization. 
    The policy is trained to maximize a trade-off between expected return and entropy, a measure of randomness in the policy.
    We only used one $Q$-value function and omitted the estimation of the state-value($V$) function.
    Formally, we can extend it to the fully decentralized multi-agent scenario use the variant of Proposition \ref{the:ddpg}:
    \begin{tcolorbox}[breakable, enhanced]
	\begin{prop}
        [Off-Policy SAC-Based Joint Gradient]\label{the:sac}
        We set $\boldsymbol{\pi}_0$ the data collection policy sampled from experience replay buffer, \textcolor{black}{$d_0$ represents the distribution of the state-occupancy measure of policy $\boldsymbol{\pi}_0$} and $\delta$ the TD(0)-error. 
        So the gradient of $J_{ac}^i \left( \mathbf{w}_{in}, \tilde{\mathbf{w}}_{sh}^i \right)$ is:
    	    \begin{align*}
    	        &\nabla_{\psi_{in}^i} J_{ac}^i \left( \mathbf{w}_{in}, \tilde{\mathbf{w}}_{sh}^i \right) \\
    	       = &\mathbf{E}_{s \sim d_0, \boldsymbol{o} \sim \mathcal{E},\boldsymbol{a}\sim\boldsymbol{\pi_0}}\Bigg[\alpha_1 \left(\frac{\pi^i_{\psi_{in}^i}(a^i|o^i)}{\pi^i_{0}(a^i|o^i)}\right)\nabla_{\psi_{in}^i}\log\pi_{\psi_{in}^i}^i(a^i|o^i)\\
    	       & \qquad\qquad\qquad\quad\left(Q^{\boldsymbol{\pi},i}_{\tilde{\phi}^i_{1,sh}}(\boldsymbol{o},\boldsymbol{a})-\alpha\log\pi^i_{\psi_{in}^i}(a^i|o^i) - \mathcal{B}(\boldsymbol{o}, \boldsymbol{a}^{\setminus i})\right)\Bigg] + \\
	            &\mathbf{E}_{s \sim d_{0}, \boldsymbol{o} \sim \mathcal{E},\boldsymbol{a}\sim\boldsymbol{\pi}}\Bigg[\alpha_1\nabla_{\psi_{in}^i}\log\pi_{\psi_{in}^i}^i(a^i|o^i)\\
	            &\qquad\qquad\qquad\quad\left(Q^{\boldsymbol{\pi},i}_{\tilde{\phi}^i_{1,sh}}(\boldsymbol{o},\boldsymbol{a})-\alpha\log\pi^i_{\psi_{in}^i}(a^i|o^i) - \mathcal{B}(\boldsymbol{o}, \boldsymbol{a}^{\setminus i})\right)\Bigg] + \\
    	        &\mathbf{E}_{s \sim d_{0}, \boldsymbol{o} \sim \mathcal{E},\boldsymbol{a}\sim\boldsymbol{\pi}}\Bigg[\alpha_2 \left(\frac{\pi^i_{0}(a^i|o^i)}{\pi^i_{\psi_{in}^i}(a^i|o^i)}\right)\nabla_{\psi_{in}^i}\log\pi_{\psi_{in}^i}^i(a^i|o^i)\left(\delta_1^2+\delta_2^2\right)\Bigg].
    	        \\
    	        &\nabla_{\tilde{\phi}^i_{1,sh}} J_{ac}^i \left( \mathbf{w}_{in}, \tilde{\mathbf{w}}_{sh}^i \right) \\
	            &= \mathbf{E}_{\substack{s\sim d_{0}, \boldsymbol{o}\sim\mathcal{E},\boldsymbol{a}\sim\boldsymbol{\pi_0}}}\left[ (\alpha_1 \left(\frac{\pi^i_{\psi_{in}^i}(a^i|o^i)}{\pi^i_{0}(a^i|o^i)}\right)+2\alpha_2\delta_1)\nabla_{\tilde{\phi}^i_{1,sh}}Q^{\boldsymbol{\pi},i}_{\tilde{\phi}^i_{1,sh}}(\boldsymbol{o}, \boldsymbol{a}) \right]+\\
	            &\qquad\qquad\mathbf{E}_{\substack{s\sim d_{0}, \boldsymbol{o}\sim\mathcal{E},\boldsymbol{a}\sim\boldsymbol{\pi}}}\left[ \left(\alpha_1+2\alpha_2\delta_1 \left(\frac{\pi^i_{0}(a^i|o^i)}{\pi^i_{\psi_{in}^i}(a^i|o^i)}\right)\right)\nabla_{\tilde{\phi}^i_{1,sh}}Q^{\boldsymbol{\pi},i}_{\tilde{\phi}^i_{1,sh}}(\boldsymbol{o}, \boldsymbol{a}) \right].
	            \\
	            &\nabla_{\tilde{\phi}^i_{2,sh}} J_{ac}^i \left( \mathbf{w}_{in}, \tilde{\mathbf{w}}_{sh}^i \right) =
	           \mathbf{E}_{\substack{s\sim d_{0}, \boldsymbol{o}\sim\mathcal{E},\boldsymbol{a}\sim\boldsymbol{\pi_0}}}\left[ 2\alpha_2\delta_2\nabla_{\tilde{\phi}^i_{2,sh}}Q^{\boldsymbol{\pi},i}_{\tilde{\phi}^i_{2,sh}}(\boldsymbol{o}, \boldsymbol{a}) \right]+\\
	            &\qquad\qquad\mathbf{E}_{\substack{s\sim d_{0}, \boldsymbol{o}\sim\mathcal{E},\boldsymbol{a}\sim\boldsymbol{\pi}}}\left[ 2\alpha_2\delta_2 \left(\frac{\pi^i_{0}(a^i|o^i)}{\pi^i_{\psi_{in}^i}(a^i|o^i)}\right)\nabla_{\tilde{\phi}^i_{2,sh}}Q^{\boldsymbol{\pi},i}_{\tilde{\phi}^i_{2,sh}}(\boldsymbol{o}, \boldsymbol{a}) \right].
	            \nonumber
    	    \end{align*}
    \end{prop}
    \end{tcolorbox}
    
    \begin{proof}
        \textbf{We proof the first equation first}.
        Extend the SAC algorithm, we have (for convenience here we suppose state space, observation space and action space are discrete)
        \begin{equation}
        	\begin{aligned}
        		&J^i_{\text{actor}} \left( \mathbf{w}_{in}, \tilde{\mathbf{w}}_{sh}^i \right) = \mathbf{E}_{s \sim d_{0}, \boldsymbol{o} \sim \mathcal{E},\boldsymbol{a}\sim\boldsymbol{\pi} }\left[Q^{\boldsymbol{\pi},i}_{\tilde{\phi}^i_{1,sh}}(\boldsymbol{o},\boldsymbol{a})-\alpha\log\pi^i_{\psi_{in}^i}(a^i|o^i) - \mathcal{B}(\boldsymbol{o}, \boldsymbol{a}^{\setminus i})\right] \\
        		&= \textstyle\sum_{s} d_{0}(s) \textstyle\sum_{\boldsymbol{o}} \mathcal{E}(\boldsymbol{o}|s) \textstyle\sum_{\boldsymbol{a}} \boldsymbol{\pi}_{\Psi_{in}}(\boldsymbol{a}|\boldsymbol{o}) \left[Q^{\boldsymbol{\pi},i}_{\tilde{\phi}^i_{1,sh}}(\boldsymbol{o},\boldsymbol{a})-\alpha\log\pi^i_{\psi_{in}^i}(a^i|o^i) - \mathcal{B}(\boldsymbol{o}, \boldsymbol{a}^{\setminus i})\right]
        		\\
        		&J^i_{\text{critic}} \left( \mathbf{w}_{in}, \tilde{\mathbf{w}}_{sh}^i \right) = \mathbf{E}_{\substack{s\sim d_{0}, \boldsymbol{o}\sim\mathcal{E}, \boldsymbol{a}\sim\boldsymbol{\pi_0}}}
        		\left[ \left( Q^{\boldsymbol{\pi}, i}_{\tilde{\phi}^i_{1,sh}}(\boldsymbol{o}, \boldsymbol{a}) - Q^{\boldsymbol{\pi}}_{tg} \right)^2 + \left( Q^{\boldsymbol{\pi}, i}_{\tilde{\phi}^i_{2,sh}}(\boldsymbol{o}, \boldsymbol{a}) - Q^{\boldsymbol{\pi}}_{tg} \right)^2 \right]\\
        		&=\textstyle\sum_{s} d_{0}(s) \textstyle\sum_{\boldsymbol{o} } \mathcal{E}_{0,s}^{\boldsymbol{o}} \textstyle\sum_{\boldsymbol{a}}\boldsymbol{\pi_0}(\boldsymbol{a}|\boldsymbol{o})\left[\left( Q^{\boldsymbol{\pi},i}_{\tilde{\phi}^i_{1,sh}}(\boldsymbol{o}, \boldsymbol{a}) - Q^{\boldsymbol{\pi}}_{tg} \right)^2 + \left( Q^{\boldsymbol{\pi},i}_{\tilde{\phi}^i_{2,sh}}(\boldsymbol{o}, \boldsymbol{a}) - Q^{\boldsymbol{\pi}}_{tg} \right)^2\right], \nonumber
        	\end{aligned}
        \end{equation}
        where $$Q^{\boldsymbol{\pi}}_{tg} = \sum_{s^{\prime}} \mathcal{P}_{s, \boldsymbol{a}}^{s^{\prime}} \left( \mathcal{C}_{s, \boldsymbol{a}}^{i, s^{\prime}} + \gamma \sum_{\boldsymbol{o}^{\prime}} \mathcal{E}^{\boldsymbol{o}^{\prime}}_{0,s^{\prime}} \sum_{\boldsymbol{a}^{\prime}} \boldsymbol{\pi}_{\Psi_{in}}(\boldsymbol{a}^{\prime}|\boldsymbol{o}^{\prime})\left(\min_{j=1,2} Q^{\boldsymbol{\pi},i}_{\tilde{\phi}^i_{j,sh}}\left(\boldsymbol{o}^{\prime},\boldsymbol{a}^{\prime}\right)-\alpha\log\pi^i_{\psi_{in}^i}(a^{\prime,i}|o^{\prime,i})\right) \right).$$
        
        We assume joint policy $\boldsymbol{\pi}_{\Psi_{in}}$ is the product of local policy functions $\prod_{i=1}^n \pi^i_{\psi_{in}^i}$. Hence \textbf{the actor part} gradient concerning  each parameter $\psi_{in}^i$ becomes:
	    \begin{equation}
	        \begin{aligned}
    	        &\nabla_{\psi_{in}^i} J^i_{\text{actor}} \left( \mathbf{w}_{in}, \tilde{\mathbf{w}}_{sh}^i \right) \\
    	        =& \nabla_{\psi_{in}^i} \textstyle\sum_{s} d_{0}(s) \textstyle\sum_{\boldsymbol{o}} \mathcal{E}(\boldsymbol{o}|s) \textstyle\sum_{\boldsymbol{a}} \boldsymbol{\pi}_{\Psi_{in}}(\boldsymbol{a}|\boldsymbol{o}) \left[Q^{\boldsymbol{\pi},i}_{\tilde{\phi}^i_{1,sh}}(\boldsymbol{o},\boldsymbol{a})-\alpha\log\pi^i_{\psi_{in}^i}(a^i|o^i) - \mathcal{B}(\boldsymbol{o}, \boldsymbol{a}^{\setminus i})\right] \\
    	        =& \textstyle\sum_{s} d_{0}(s) \textstyle\sum_{\boldsymbol{o}} \mathcal{E}(\boldsymbol{o}|s) \textstyle\sum_{\boldsymbol{a}} \nabla_{\psi_{in}^i} \boldsymbol{\pi}_{\Psi_{in}}(\boldsymbol{a}|\boldsymbol{o}) \left[Q^{\boldsymbol{\pi},i}_{\tilde{\phi}^i_{1,sh}}(\boldsymbol{o},\boldsymbol{a})-\alpha\log\pi^i_{\psi_{in}^i}(a^i|o^i) - \mathcal{B}(\boldsymbol{o}, \boldsymbol{a}^{\setminus i})\right] \\
    	        =&\mathbf{E}_{s \sim d_{0}, \boldsymbol{o} \sim \mathcal{E},\boldsymbol{a}\sim\boldsymbol{\pi}}\Bigg[\nabla_{\psi_{in}^i}\log\pi_{\psi_{in}^i}^i(a^i|o^i)\left(Q^{\boldsymbol{\pi},i}_{\tilde{\phi}^i_{1,sh}}(\boldsymbol{o},\boldsymbol{a})-\alpha\log\pi^i_{\psi_{in}^i}(a^i|o^i) - \mathcal{B}(\boldsymbol{o}, \boldsymbol{a}^{\setminus i})\right)\Bigg]. \nonumber
	        \end{aligned}
	    \end{equation}
	    Note that we need to resample to calculate the unbias policy gradient so that we cannot use the off-policy data directly.
	    We solve the above problem by importance sampling.
	    \textbf{For off-policy data saved in experience replay buffer} we have
	    \begin{equation}
	    \begin{aligned}
    	        &\nabla_{\psi_{in}^i} J^i_{\text{actor}} \left( \mathbf{w}_{in}, \tilde{\mathbf{w}}_{sh}^i \right) = \mathbf{E}_{s \sim d_0, \boldsymbol{o} \sim \mathcal{E},\boldsymbol{a}\sim\boldsymbol{\pi_0}}\Bigg[\left(\frac{\pi^i_{\psi_{in}^i}(a^i|o^i)}{\pi^i_{0}(a^i|o^i)}\right)\nabla_{\psi_{in}^i}\log\pi_{\psi_{in}^i}^i(a^i|o^i)\\
                &\qquad\left(Q^{\boldsymbol{\pi},i}_{\tilde{\phi}^i_{1,sh}}(\boldsymbol{o},\boldsymbol{a})-\alpha\log\pi^i_{\psi_{in}^i}(a^i|o^i) - \mathcal{B}(\boldsymbol{o}, \boldsymbol{a}^{\setminus i})\right)\Bigg], \nonumber
    	\end{aligned}
	    \end{equation}
	    combined \textbf{the gradient calculated by the data which is resampled by current policy}, we get the joint off-policy policy gradient associates with the actor part: 
	    \begin{equation}
	    \begin{aligned}
    	        &\nabla_{\psi_{in}^i} J^i_{\text{actor}} \left( \mathbf{w}_{in}, \tilde{\mathbf{w}}_{sh}^i \right) = \mathbf{E}_{s \sim d_0, \boldsymbol{o} 
    	       \sim\mathcal{E},\boldsymbol{a}\sim\boldsymbol{\pi_0}}\Bigg[\left(\frac{\pi^i_{\psi_{in}^i}(a^i|o^i)}{\pi^i_{0}(a^i|o^i)}\right)\nabla_{\psi_{in}^i}\log\pi_{\psi_{in}^i}^i(a^i|o^i)\\
                &\qquad\left(Q^{\boldsymbol{\pi},i}_{\tilde{\phi}^i_{1,sh}}(\boldsymbol{o},\boldsymbol{a})-\alpha\log\pi^i_{\psi_{in}^i}(a^i|o^i) - \mathcal{B}(\boldsymbol{o}, \boldsymbol{a}^{\setminus i})\right)\Bigg] + \\
    	        & \qquad\qquad \mathbf{E}_{s \sim d_{0}, \boldsymbol{o} \sim \mathcal{E},\boldsymbol{a}\sim\boldsymbol{\pi}}\Bigg[\nabla_{\psi_{in}^i}\log\pi_{\psi_{in}^i}^i(a^i|o^i)\left(Q^{\boldsymbol{\pi},i}_{\tilde{\phi}^i_{1,sh}}(\boldsymbol{o},\boldsymbol{a})-\alpha\log\pi^i_{\psi_{in}^i}(a^i|o^i) - \mathcal{B}(\boldsymbol{o}, \boldsymbol{a}^{\setminus i})\right)\Bigg].
    	        \nonumber
    	\end{aligned}
	    \end{equation}
	    On the contrary, we directly use the off-policy data when calculating the value function gradient so that we cannot use the above-resampled data.
	    \textbf{For the critic part}, we use the above trick in reverse. 
	    \textbf{For off-policy data saved in experience replay buffer} we have
	    \begin{equation}
	    \begin{aligned}
	            &\nabla_{\psi_{in}^i} J^i_{\text{critic}} \left( \mathbf{w}_{in}, \tilde{\mathbf{w}}_{sh}^i \right)  \\
	            &= \nabla_{\psi_{in}^i}\textstyle\sum_{s} d_{0}(s) \textstyle\sum_{\boldsymbol{o}} \mathcal{E}_{0,s}^{\boldsymbol{o}} \textstyle\sum_{\boldsymbol{a}}\boldsymbol{\pi_0}(\boldsymbol{a}|\boldsymbol{o})\left(\left( Q^{\boldsymbol{\pi},i}_{\tilde{\phi}^i_{1,sh}}(\boldsymbol{o}, \boldsymbol{a}) - Q^{\boldsymbol{\pi}}_{tg} \right)^2+\left( Q^{\boldsymbol{\pi},i}_{\tilde{\phi}^i_{2,sh}}(\boldsymbol{o}, \boldsymbol{a}) - Q^{\boldsymbol{\pi}}_{tg} \right)^2\right) \\
	            &= 0, \nonumber
	    \end{aligned}
	    \end{equation}
	    combined \textbf{the gradient calculated by the data which is resampled by current policy}, we get the joint off-policy policy gradient associates with the critic part 
	    
	    \begin{equation}
	        \begin{aligned}
    			&\nabla_{\psi_{in}^i} J^i_{\text{critic}} \left( \mathbf{w}_{in}, \tilde{\mathbf{w}}_{sh}^i \right) 
    			= \mathbf{E}_{\substack{s\sim d_{0}, \boldsymbol{o}\sim\mathcal{E},\boldsymbol{a}\sim\boldsymbol{\pi}}}\Bigg[\left(\frac{\pi^i_{0}(a^i|o^i)}{\pi^i_{\psi_{in}^i}(a^i|o^i)}\right)\nabla_{\psi_{in}^i}\log\pi_{\psi_{in}^i}^i(a^i|o^i)\left(\delta_1^2+\delta_2^2\right)\Bigg]. \nonumber
	        \end{aligned}
	    \end{equation}
	    where $\delta_j=Q^{\boldsymbol{\pi}, i}_{\tilde{\phi}^i_{j,sh}}(\boldsymbol{o}, \boldsymbol{a}) - Q^{\boldsymbol{\pi}}_{tg}$.
	    We denote the clipped importance sampling term
	    $\min\left(\epsilon,\frac{\pi^i_{\psi_{in}^i}(a^i|o^i)}{\pi^i_{0}(a^i|o^i)}\right)$ as $CIM_{\epsilon}(\pi^i_{\psi_{in}^i};\pi^i_0)$ and 
	    $\min\left(\epsilon,\frac{\pi^i_{0}(a^i|o^i)}{\pi^i_{\psi_{in}^i}(a^i|o^i)}\right)$ as $CIM_{\epsilon}(\pi^i_0;\pi^i_{\psi_{in}^i})$.
	    Finally, we get the following off-policy joint gradient w.r.t. actor parameters:
	    \begin{align*}
	            &\nabla_{\psi_{in}^i} J_{ac}^i \left( \mathbf{w}_{in}, \tilde{\mathbf{w}}_{sh}^i \right) \\
	            = &\alpha_1 \nabla_{\psi_{in}^i} J_{actor}^i \left( \mathbf{w}_{in}, \tilde{\mathbf{w}}_{sh}^i \right) + \alpha_2 \nabla_{\psi_{in}^i} J_{critic}^i \left( \mathbf{w}_{in}, \tilde{\mathbf{w}}_{sh}^i \right) \\
	            = & \mathbf{E}_{s \sim d_0, \boldsymbol{o} \sim \mathcal{E},\boldsymbol{a}\sim\boldsymbol{\pi_0}}\Bigg[\alpha_1 \left(\frac{\pi^i_{\psi_{in}^i}(a^i|o^i)}{\pi^i_{0}(a^i|o^i)}\right)\nabla_{\psi_{in}^i}\log\pi_{\psi_{in}^i}^i(a^i|o^i)\\
                &\qquad\left(Q^{\boldsymbol{\pi},i}_{\tilde{\phi}^i_{1,sh}}(\boldsymbol{o},\boldsymbol{a})-\alpha\log\pi^i_{\psi_{in}^i}(a^i|o^i) - \mathcal{B}(\boldsymbol{o}, \boldsymbol{a}^{\setminus i})\right)\Bigg] + \\
	            &\qquad\mathbf{E}_{s \sim d_{0}, \boldsymbol{o} \sim \mathcal{E},\boldsymbol{a}\sim\boldsymbol{\pi}}\Bigg[\alpha_1\nabla_{\psi_{in}^i}\log\pi_{\psi_{in}^i}^i(a^i|o^i)\left(Q^{\boldsymbol{\pi},i}_{\tilde{\phi}^i_{1,sh}}(\boldsymbol{o},\boldsymbol{a})-\alpha\log\pi^i_{\psi_{in}^i}(a^i|o^i) - \mathcal{B}(\boldsymbol{o}, \boldsymbol{a}^{\setminus i})\right)\Bigg] + \\
    	        &\qquad\mathbf{E}_{s \sim d_{0}, \boldsymbol{o} \sim \mathcal{E},\boldsymbol{a}\sim\boldsymbol{\pi}}\Bigg[\alpha_2 \left(\frac{\pi^i_{0}(a^i|o^i)}{\pi^i_{\psi_{in}^i}(a^i|o^i)}\right)\nabla_{\psi_{in}^i}\log\pi_{\psi_{in}^i}^i(a^i|o^i)\left(\delta_1^2+\delta_2^2\right)\Bigg]\\
	            \simeq & \mathbf{E}_{s \sim d_0, \boldsymbol{o} \sim \mathcal{E},\boldsymbol{a}\sim\boldsymbol{\pi_0}}\Bigg[\alpha_1 CIM_{\epsilon}(\pi^i_{\psi_{in}^i};\pi^i_0)\nabla_{\psi_{in}^i}\log\pi_{\psi_{in}^i}^i(a^i|o^i)\\
                &\qquad\left(Q^{\boldsymbol{\pi},i}_{\tilde{\phi}^i_{1,sh}}(\boldsymbol{o},\boldsymbol{a})-\alpha\log\pi^i_{\psi_{in}^i}(a^i|o^i) - \mathcal{B}(\boldsymbol{o}, \boldsymbol{a}^{\setminus i})\right)\Bigg] + \\
	            &\qquad\mathbf{E}_{s \sim d_{0}, \boldsymbol{o} \sim \mathcal{E},\boldsymbol{a}\sim\boldsymbol{\pi}}\Bigg[\alpha_1\nabla_{\psi_{in}^i}\log\pi_{\psi_{in}^i}^i(a^i|o^i)\left(Q^{\boldsymbol{\pi},i}_{\tilde{\phi}^i_{1,sh}}(\boldsymbol{o},\boldsymbol{a})-\alpha\log\pi^i_{\psi_{in}^i}(a^i|o^i) - \mathcal{B}(\boldsymbol{o}, \boldsymbol{a}^{\setminus i})\right)\Bigg] + \\
    	        &\qquad\mathbf{E}_{s \sim d_{0}, \boldsymbol{o} \sim \mathcal{E},\boldsymbol{a}\sim\boldsymbol{\pi}}\Bigg[\alpha_2CIM_{\epsilon}(\pi^i_0;\pi^i_{\psi_{in}^i})\nabla_{\psi_{in}^i}\log\pi_{\psi_{in}^i}^i(a^i|o^i)\left(\delta_1^2+\delta_2^2\right)\Bigg].
    	\end{align*}
	    The above is the proof of the first equation, \textbf{below we prove the second}.
	    \textbf{For actor part, we first calculate the gradient use the data resampled by the current policy}
	    \begin{equation}
	        \begin{aligned}
    	        \nabla_{\tilde{\phi}^i_{1,sh}} J^i_{\text{actor}} \left( \mathbf{w}_{in}, \tilde{\mathbf{w}}_{sh}^i \right) 
    	        &= \nabla_{\tilde{\phi}^i_{1,sh}} \textstyle\sum_{s} d_{0}(s) \textstyle\sum_{\boldsymbol{o}} \mathcal{E}(\boldsymbol{o}|s) \textstyle\sum_{\boldsymbol{a}} \boldsymbol{\pi}_{\Psi_{in}}(\boldsymbol{a}|\boldsymbol{o}) \\
    	        &\qquad\qquad\qquad\quad\left[Q^{\boldsymbol{\pi},i}_{\tilde{\phi}^i_{1,sh}}(\boldsymbol{o},\boldsymbol{a})-\alpha\log\pi^i_{\psi_{in}^i}(a^i|o^i) - \mathcal{B}(\boldsymbol{o}, \boldsymbol{a}^{\setminus i})\right] \\
    	        &= \textstyle\sum_{s} d_{0}(s) \textstyle\sum_{\boldsymbol{o}} \mathcal{E}(\boldsymbol{o}|s) \textstyle\sum_{\boldsymbol{a}} \boldsymbol{\pi}_{\Psi_{in}}(\boldsymbol{a}|\boldsymbol{o}) \nabla_{\tilde{\phi}^i_{1,sh}}Q^{\boldsymbol{\pi},i}_{\tilde{\phi}^i_{1,sh}}(\boldsymbol{o},\boldsymbol{a})\\
    	        &=\mathbf{E}_{\substack{s\sim d_{0}, \boldsymbol{o}\sim\mathcal{E},\boldsymbol{a}\sim\boldsymbol{\pi}}}\left[ \nabla_{\tilde{\phi}^i_{1,sh}}Q^{\boldsymbol{\pi},i}_{\tilde{\phi}^i_{1,sh}}(\boldsymbol{o},\boldsymbol{a}) \right]. \nonumber
	        \end{aligned}
	    \end{equation}
	    \textbf{For calculate the gradient use off-policy data} we also introduce importance sampling, then we have
	    \begin{equation}
	        \begin{aligned}
    	        \nabla_{\tilde{\phi}^i_{1,sh}} J^i_{\text{actor}} \left( \mathbf{w}_{in}, \tilde{\mathbf{w}}_{sh}^i \right) 
    	        &=\mathbf{E}_{\substack{s\sim d_{0}, \boldsymbol{o}\sim\mathcal{E},\boldsymbol{a}\sim\boldsymbol{\pi_0}}}\left[ \left(\frac{\pi^i_{\psi_{in}^i}(a^i|o^i)}{\pi^i_{0}(a^i|o^i)}\right)\nabla_{\tilde{\phi}^i_{1,sh}}Q^{\boldsymbol{\pi},i}_{\tilde{\phi}^i_{1,sh}}(\boldsymbol{o},\boldsymbol{a}) \right]. \nonumber
	        \end{aligned}
	    \end{equation}
	    We then combine the above two part gradient:
	    \begin{equation}
	        \begin{aligned}
    	        \nabla_{\tilde{\phi}^i_{1,sh}} J^i_{\text{actor}} \left( \mathbf{w}_{in}, \tilde{\mathbf{w}}_{sh}^i \right) 
    	        &=\mathbf{E}_{\substack{s\sim d_{0}, \boldsymbol{o}\sim\mathcal{E},\boldsymbol{a}\sim\boldsymbol{\pi}}}\left[ \nabla_{\tilde{\phi}^i_{1,sh}}Q^{\boldsymbol{\pi},i}_{\tilde{\phi}^i_{1,sh}}(\boldsymbol{o},\boldsymbol{a}) \right]+\\
    	        &\qquad\qquad\mathbf{E}_{\substack{s\sim d_{0}, \boldsymbol{o}\sim\mathcal{E},\boldsymbol{a}\sim\boldsymbol{\pi_0}}}\left[ \left(\frac{\pi^i_{\psi_{in}^i}(a^i|o^i)}{\pi^i_{0}(a^i|o^i)}\right)\nabla_{\tilde{\phi}^i_{1,sh}}Q^{\boldsymbol{\pi},i}_{\tilde{\phi}^i_{1,sh}}(\boldsymbol{o},\boldsymbol{a}) \right]. \nonumber
	        \end{aligned}
	    \end{equation}
	    \textbf{For critic part}, we first \textbf{calculate the gradient use the off-policy data}
	    \begin{equation}
	        \begin{aligned}
	            \nabla_{\tilde{\phi}^i_{1,sh}} J^i_{\text{critic}} \left( \mathbf{w}_{in}, \tilde{\mathbf{w}}_{sh}^i \right)  
	            &= \nabla_{\tilde{\phi}^i_{1,sh}} \textstyle\sum_{s} d_{0}(s) \textstyle\sum_{\boldsymbol{o}} \mathcal{E}_{0,s}^{\boldsymbol{o}} \textstyle\sum_{\boldsymbol{a}}\boldsymbol{\pi_0}(\boldsymbol{a}|\boldsymbol{o})\\
	            &\qquad\qquad\qquad\quad\left(\left( Q^{\boldsymbol{\pi},i}_{\tilde{\phi}^i_{1,sh}}(\boldsymbol{o}, \boldsymbol{a}) - Q^{\boldsymbol{\pi}}_{tg} \right)^2+\left( Q^{\boldsymbol{\pi},i}_{\tilde{\phi}^i_{2,sh}}(\boldsymbol{o}, \boldsymbol{a}) - Q^{\boldsymbol{\pi}}_{tg} \right)^2\right) \\
	            &= \textstyle\sum_{s} d_{0}(s) \textstyle\sum_{\boldsymbol{o}} \mathcal{E}_{0,s}^{\boldsymbol{o}} \textstyle\sum_{\boldsymbol{a}}\boldsymbol{\pi_0}(\boldsymbol{a}|\boldsymbol{o})2\delta_1\nabla_{\tilde{\phi}^i_{1,sh}}Q^{\boldsymbol{\pi},i}_{\tilde{\phi}^i_{1,sh}}(\boldsymbol{o}, \boldsymbol{a})\\
	            &=\mathbf{E}_{\substack{s\sim d_{0}, \boldsymbol{o}\sim\mathcal{E},\boldsymbol{a}\sim\boldsymbol{\pi_0}}}\left[ 2\delta_1\nabla_{\tilde{\phi}^i_{1,sh}}Q^{\boldsymbol{\pi},i}_{\tilde{\phi}^i_{1,sh}}(\boldsymbol{o}, \boldsymbol{a}) \right]. \nonumber
	        \end{aligned}
	    \end{equation}
	    Then, \textbf{for resampled data}, we have
	    \begin{equation}
	        \begin{aligned}
	            \nabla_{\tilde{\phi}^i_{1,sh}} J^i_{\text{critic}} \left( \mathbf{w}_{in}, \tilde{\mathbf{w}}_{sh}^i \right)  
	            &= \mathbf{E}_{\substack{s\sim d_{0}, \boldsymbol{o}\sim\mathcal{E},\boldsymbol{a}\sim\boldsymbol{\pi}}}\left[ 2\delta_1 \left(\frac{\pi^i_{0}(a^i|o^i)}{\pi^i_{\psi_{in}^i}(a^i|o^i)}\right)\nabla_{\tilde{\phi}^i_{1,sh}}Q^{\boldsymbol{\pi},i}_{\tilde{\phi}^i_{1,sh}}(\boldsymbol{o}, \boldsymbol{a}) \right]. \nonumber
	        \end{aligned}
	    \end{equation}
	    Combining above two part gradient, we obtain:
	    \begin{equation}
	        \begin{aligned}
	            \nabla_{\tilde{\phi}^i_{1,sh}} J^i_{\text{critic}} \left( \mathbf{w}_{in}, \tilde{\mathbf{w}}_{sh}^i \right)  
	            =& \mathbf{E}_{\substack{s\sim d_{0}, \boldsymbol{o}\sim\mathcal{E},\boldsymbol{a}\sim\boldsymbol{\pi_0}}}\left[ 2\delta_1\nabla_{\tilde{\phi}^i_{1,sh}}Q^{\boldsymbol{\pi},i}_{\tilde{\phi}^i_{1,sh}}(\boldsymbol{o}, \boldsymbol{a}) \right]+\\
	            &\qquad\qquad\mathbf{E}_{\substack{s\sim d_{0}, \boldsymbol{o}\sim\mathcal{E},\boldsymbol{a}\sim\boldsymbol{\pi}}}\left[ 2\delta_1 \left(\frac{\pi^i_{0}(a^i|o^i)}{\pi^i_{\psi_{in}^i}(a^i|o^i)}\right)\nabla_{\tilde{\phi}^i_{1,sh}}Q^{\boldsymbol{\pi},i}_{\tilde{\phi}^i_{1,sh}}(\boldsymbol{o}, \boldsymbol{a}) \right]. \nonumber
	        \end{aligned}
	    \end{equation}
	    Finally, we get following off-policy joint gradient w.r.t. the first critic parameters
	    \begin{equation}
	    \begin{aligned}
    	       &\nabla_{\tilde{\phi}^i_{1,sh}} J_{ac}^i \left( \mathbf{w}_{in}, \tilde{\mathbf{w}}_{sh}^i \right)
	            = \alpha_1 \nabla_{\tilde{\phi}^i_{1,sh}} J_{actor}^i \left( \mathbf{w}_{in}, \tilde{\mathbf{w}}_{sh}^i \right) + \alpha_2 \nabla_{\tilde{\phi}^i_{1,sh}} J_{critic}^i \left( \mathbf{w}_{in}, \tilde{\mathbf{w}}_{sh}^i \right)\\
	            &= \mathbf{E}_{\substack{s\sim d_{0}, \boldsymbol{o}\sim\mathcal{E},\boldsymbol{a}\sim\boldsymbol{\pi_0}}}\left[ (\alpha_1 \left(\frac{\pi^i_{\psi_{in}^i}(a^i|o^i)}{\pi^i_{0}(a^i|o^i)}\right)+2\alpha_2\delta_1)\nabla_{\tilde{\phi}^i_{1,sh}}Q^{\boldsymbol{\pi},i}_{\tilde{\phi}^i_{1,sh}}(\boldsymbol{o}, \boldsymbol{a}) \right]+\\
	            &\qquad\qquad\mathbf{E}_{\substack{s\sim d_{0}, \boldsymbol{o}\sim\mathcal{E},\boldsymbol{a}\sim\boldsymbol{\pi}}}\left[ \left(\alpha_1+2\alpha_2\delta_1 \left(\frac{\pi^i_{0}(a^i|o^i)}{\pi^i_{\psi_{in}^i}(a^i|o^i)}\right)\right)\nabla_{\tilde{\phi}^i_{1,sh}}Q^{\boldsymbol{\pi},i}_{\tilde{\phi}^i_{1,sh}}(\boldsymbol{o}, \boldsymbol{a}) \right]\\
	            &\simeq \mathbf{E}_{\substack{s\sim d_{0}, \boldsymbol{o}\sim\mathcal{E},\boldsymbol{a}\sim\boldsymbol{\pi_0}}}\left[ (\alpha_1 CIM_{\epsilon}(\pi^i_{\psi_{in}^i};\pi^i_0)+2\alpha_2\delta_1)\nabla_{\tilde{\phi}^i_{1,sh}}Q^{\boldsymbol{\pi},i}_{\tilde{\phi}^i_{1,sh}}(\boldsymbol{o}, \boldsymbol{a}) \right]+\\
	            &\qquad\qquad\mathbf{E}_{\substack{s\sim d_{0}, \boldsymbol{o}\sim\mathcal{E},\boldsymbol{a}\sim\boldsymbol{\pi}}}\left[ \left(\alpha_1+2\alpha_2\delta_1 CIM_{\epsilon}(\pi^i_0;\pi^i_{\psi_{in}^i})\right)\nabla_{\tilde{\phi}^i_{1,sh}}Q^{\boldsymbol{\pi},i}_{\tilde{\phi}^i_{1,sh}}(\boldsymbol{o}, \boldsymbol{a}) \right].
    	        \nonumber
    	\end{aligned}
	    \end{equation}
	    \textbf{For the third equation and use resampled data to calculate the actor part gradient}, we have
	    \begin{equation}
	        \begin{aligned}
    	        &\nabla_{\tilde{\phi}^i_{2,sh}} J^i_{\text{actor}} \left( \mathbf{w}_{in}, \tilde{\mathbf{w}}_{sh}^i \right) \\
    	        & = \nabla_{\tilde{\phi}^i_{2,sh}} \textstyle\sum_{s} d_{0}(s) \textstyle\sum_{\boldsymbol{o}} \mathcal{E}(\boldsymbol{o}|s) \textstyle\sum_{\boldsymbol{a}} \boldsymbol{\pi}_{\Psi_{in}}(\boldsymbol{a}|\boldsymbol{o}) \left[Q^{\boldsymbol{\pi},i}_{\tilde{\phi}^i_{1,sh}}(\boldsymbol{o},\boldsymbol{a})-\alpha\log\pi^i_{\psi_{in}^i}(a^i|o^i) - \mathcal{B}(\boldsymbol{o}, \boldsymbol{a}^{\setminus i})\right] \\
    	        &= 0\nonumber
	        \end{aligned}
	    \end{equation}
	    Similarly, \textbf{for off-policy data} the gradient also is 0.
	    For \textbf{critic part} the gradient is similar as the second equation
	    \begin{equation}
	        \begin{aligned}
	            \nabla_{\tilde{\phi}^i_{2,sh}} J^i_{\text{critic}} \left( \mathbf{w}_{in}, \tilde{\mathbf{w}}_{sh}^i \right)  
	            =& \mathbf{E}_{\substack{s\sim d_{0}, \boldsymbol{o}\sim\mathcal{E},\boldsymbol{a}\sim\boldsymbol{\pi_0}}}\left[ 2\delta_2\nabla_{\tilde{\phi}^i_{2,sh}}Q^{\boldsymbol{\pi},i}_{\tilde{\phi}^i_{2,sh}}(\boldsymbol{o}, \boldsymbol{a}) \right]+\\
	            &\qquad\qquad\mathbf{E}_{\substack{s\sim d_{0}, \boldsymbol{o}\sim\mathcal{E},\boldsymbol{a}\sim\boldsymbol{\pi}}}\left[ 2\delta_2 \left(\frac{\pi^i_{0}(a^i|o^i)}{\pi^i_{\psi_{in}^i}(a^i|o^i)}\right)\nabla_{\tilde{\phi}^i_{2,sh}}Q^{\boldsymbol{\pi},i}_{\tilde{\phi}^i_{2,sh}}(\boldsymbol{o}, \boldsymbol{a}) \right]. \nonumber
	        \end{aligned}
	    \end{equation}
	    Finally, we get following off-policy joint gradient w.r.t. the second critic parameters
	    \begin{align*}
	            \nabla_{\tilde{\phi}^i_{2,sh}} J_{ac}^i \left( \mathbf{w}_{in}, \tilde{\mathbf{w}}_{sh}^i \right)
	            &= \alpha_1 \nabla_{\tilde{\phi}^i_{2,sh}} J_{actor}^i \left( \mathbf{w}_{in}, \tilde{\mathbf{w}}_{sh}^i \right) + \alpha_2 \nabla_{\tilde{\phi}^i_{2,sh}} J_{critic}^i \left( \mathbf{w}_{in}, \tilde{\mathbf{w}}_{sh}^i \right)\\
	            &= \mathbf{E}_{\substack{s\sim d_{0}, \boldsymbol{o}\sim\mathcal{E},\boldsymbol{a}\sim\boldsymbol{\pi_0}}}\left[ 2\alpha_2\delta_2\nabla_{\tilde{\phi}^i_{2,sh}}Q^{\boldsymbol{\pi},i}_{\tilde{\phi}^i_{2,sh}}(\boldsymbol{o}, \boldsymbol{a}) \right]+\\
	            &\qquad\qquad\mathbf{E}_{\substack{s\sim d_{0}, \boldsymbol{o}\sim\mathcal{E},\boldsymbol{a}\sim\boldsymbol{\pi}}}\left[ 2\alpha_2\delta_2 \left(\frac{\pi^i_{0}(a^i|o^i)}{\pi^i_{\psi_{in}^i}(a^i|o^i)}\right)\nabla_{\tilde{\phi}^i_{2,sh}}Q^{\boldsymbol{\pi},i}_{\tilde{\phi}^i_{2,sh}}(\boldsymbol{o}, \boldsymbol{a}) \right]\\
	            &\simeq \mathbf{E}_{\substack{s\sim d_{0}, \boldsymbol{o}\sim\mathcal{E},\boldsymbol{a}\sim\boldsymbol{\pi_0}}}\left[ 2\alpha_2\delta_2\nabla_{\tilde{\phi}^i_{2,sh}}Q^{\boldsymbol{\pi},i}_{\tilde{\phi}^i_{2,sh}}(\boldsymbol{o}, \boldsymbol{a}) \right]+\\
	            &\qquad\qquad\mathbf{E}_{\substack{s\sim d_{0}, \boldsymbol{o}\sim\mathcal{E},\boldsymbol{a}\sim\boldsymbol{\pi}}}\left[ 2\alpha_2\delta_2 CIM_{\epsilon}(\pi^i_0;\pi^i_{\psi_{in}^i})\nabla_{\tilde{\phi}^i_{2,sh}}Q^{\boldsymbol{\pi},i}_{\tilde{\phi}^i_{2,sh}}(\boldsymbol{o}, \boldsymbol{a}) \right].
    	        \nonumber
    	\end{align*}
    \end{proof}

}}

\section{The Specific Forms of Instantiation Algorithms}

The specific forms of F2A2-COMA, F2A2-DDPG, F2A2-TD3 and F2A2-SAC are shown in Figure~\ref{fig:f2a2-eq}.

\begin{figure*}
    \centering
    \includegraphics[width=\textwidth]{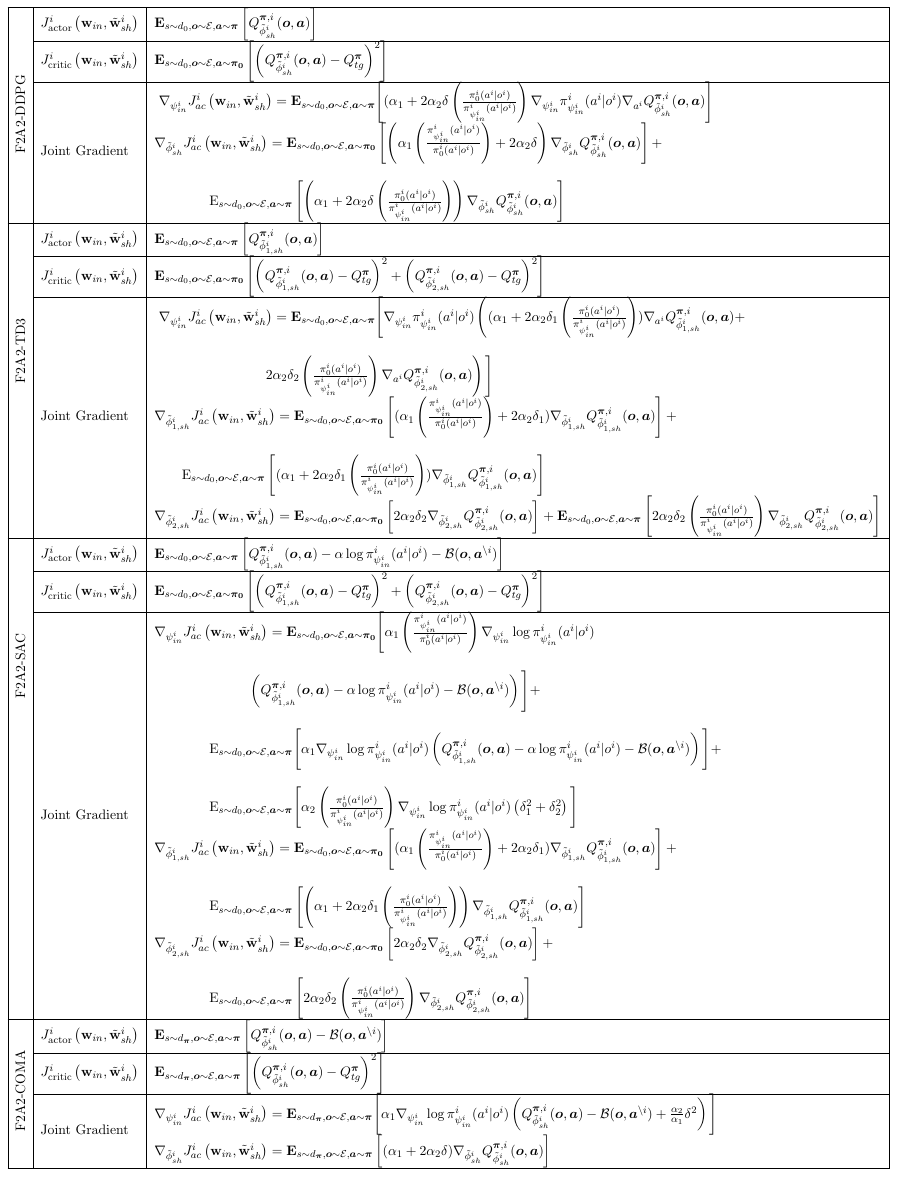}
    \caption{Extensions of off-policy and on-policy actor-critic joint gradient.}
    \label{fig:f2a2-eq}
\end{figure*}

\section{Training procedure}
\paragraph{Cooperative MPE.} For all algorithms expect for COMA and F2A2-COMA, our training procedure consists of performing $12$ parallel rollouts, and adding a tuple of $\left(o_{t}, a_{t}, c_{t}, o_{t+1}\right)_{1 \dots N}$ to a replay buffer (with maximum length $1e6$) for each timestep.
We reset each environment after every $100$ steps for \textit{Cooperative Treasure Collection} and $25$ steps for \textit{Rover Tower} (across all rollouts), we perform $4$ updates for the all actors and critics.
For each update, we sample mini-batches of $1024$ timesteps from the replay buffer (for COMA and F2A2-COMA, we used the most recent $1024$ timesteps) and then perform gradient descent on the corresponding loss objective, using Adam~\citep{kingma2014adam} as the optimizer with a learning rate of $0.001$. 
After the updates are complete, we update the target network parameters(if there are) to move toward our learned parameters, as in~\citep{haarnoja2018soft}: $\overline{\Theta}=(1-\tau) \overline{\Theta}+\tau \Theta$, where $tau$ is the soft update rate (set to $0.005$). 
We use a discount factor, $\gamma$, of $0.99$.
All networks use a hidden dimension of $128$ and Leaky Rectified Linear Units as the nonlinearity. 
We use $0.01$ as our temperature setting for MAAC and F2A2-SAC. 
We use $4$ attention heads in MAAC and F2A2-SAC. 

\paragraph{StarCraft II.} All policies are implemented as two-layer recurrent neural networks (GRUs) with $64$ hidden units,
while the critic is the same as the settings of \textit{Cooperative Multi-agent Particle Environments}. 
For exploration, we use a bounded softmax distribution in which the agent samples from a softmax over the policy logits with probability ($1 - \epsilon$) and samples randomly with probability $\epsilon$. 
We anneal $\epsilon$ from $0.5$ to $0.01$ across the first $50k$ environment steps.
Episodes are collected using eight parallel SCII environments. 
Optimization is carried out on a single GPU with Adam and a learning rate of $0.0005$ for both the agents and the critics. 
The policies are fully unrolled and updated in a large mini-batch of $T × B$ entries, where $T = 60$ and $B = 8$. 
By contrast, the critic is optimized in small mini-batches of size $8$, one for each time-step, looping backward in
time. 
We found that this stabilized and accelerated training compared to full batch updates for the critic. 
The target network for the critic is updated after every $200$ critic update.

\paragraph{MAgent Environment.} For the MAgent environment, we fix the model of the enemy agents to IQL after a fixed number ($2000$) of self-play training rounds.
The learning rate for actors and critics is $0.0001$.
The discount factor is set to 0.95, and the mini-batch size is $128$.
The size of the replay buffer is $500000$.
And the temperature for F2A2-ISAC is $0.08$.\\

While for the algorithm to model other agents, we first collect $1000$ trajectories randomly to construct the past trajectories dataset and fixed it. 
We keep the most recent $1200$ state-action pairs of all agents; thus, the mini-batches size of the impromptu prediction net is also $1200$. 
The char-net and mental-net are $2$-layer GRU networks with $64$ hidden dimensions, and two prediction net are MLP networks with $128$ hidden dimensions. 
The length of the former is fixed to the length of the episode, $25$, and the latter is fixed at $5$. 
The update frequency of the whole network is the same as the reinforcement learning part, and the mini-batches size of the natural prediction net is $1600$. 
The optimizer and learning rate are also set the same as those in the reinforcement learning part.

Except that the F2A2-ISAC algorithm uses the Tensorflow~\citep{tensorflow2015-whitepaper} framework, the other algorithms we proposed use the PyTorch~\citep{paszke2017automatic} framework. 
We run all the experiments on a machine with 44 CPU cores, 128G RAM, and 4 Nvidia 1080Ti GPUs. 
We use the original papers ' open-source code for all of the simulation environments involved in the experiments.

\section{More Case Studies}

In the following, we will discuss more state-of-the-art MARL algorithms to motivate the above optimization formulation (\ref{eq:framework-v3}) and the following proposed algorithm framework.
There contain to the communication learning algorithm \citep{jiang2018learning} and common knowledge based algorithm \citep{de2019multi}.

\subsection{Attentional Communication (ATOC)}
The communication learning algorithms advocate learning collaborative policies through communication.
However, for early communication learning algorithms\citep{foerster2016learning,peng2017multiagent,sukhbaatar2016learning}, information sharing among all agents
or in predefined communication architectures, these methods adopt can be problematic. 
When there is a large number of agents, agents cannot differentiate valuable information that helps cooperative
decision making from globally shared information, and hence communication barely helps and could
even jeopardize the learning of cooperation. 
Moreover, in real-world applications, it is costly that all agents communicate with each other since receiving a large amount of information requires high
bandwidth and incurs long delay and high computational complexity\citep{jiang2018learning}.

ATOC\citep{jiang2018learning} propose an attentional communication model to solve the above problems.
ATOC introduces a shared attention unit so that the communication architecture between agents can dynamically change according to needs; 
at the same time, in a self-organized communication group, the message is generated through a shared Bi-LSTM network, the message generation module.
In a word, ATOC sharing all actor parameters and all critic parameters for all agents, which means $\phi_{sh}=\varnothing$ and $\psi_{sh}=\varnothing$.
The ATOC algorithm uses the DDPG\citep{lillicrap2015continuous} as the backbone, and the objective function for each agent $i$ in actor and critic phases can be reformulated as

\begin{eqnarray}
J^i_{\text{actor}} \left(\psi_{sh}, \phi_{sh} \right) &=& \mathbb{E}\left[Q^{\boldsymbol{\pi},i}_{\phi_{sh}}(o^i,\pi^i_{\psi_{sh}}(o^i))\right], \nonumber\\
J^i_{\text{critic}} \left(\psi_{sh}, \phi_{sh} \right) &=& \mathbb{E} \left[ \left( Q^{\boldsymbol{\pi},i}_{\phi_{sh}}(o^i,a^i) - Q^{\boldsymbol{\pi},i}_{tg} \right)^2 \right], \nonumber
\end{eqnarray}

where the expectations are take on $o^i, a^i, c^i, {o^{\prime}}^{i} \sim \mathcal{D}$ because the DDPG algorithm are off-policy algorithm, and $\mathcal{D}$ is the shared experience replay buffer;
$\boldsymbol{\pi} := \left\{\pi^i_{\psi_{sh}}\right\}$ represents the joint policy of all agents;
$\psi_{sh}$ and $\phi_{sh}$ represent the sharing actor parameter and critic parameter respectively;
$Q^{\boldsymbol{\pi},i}_{tg}:=c^i + \gamma  Q^{\boldsymbol{\pi},i}_{\phi_{sh}}\left(o',\pi^i_{\psi_{sh}}(o')\right)$;

The detailed calculation process of $\pi^i_{\psi_{sh}}(o^i)$ is as follows:
Firstly, the observation of agent $i$ is encoded by an observation encoder which is parameterized by $\theta_{enc}$.
Then the observation embedding is fed into a recurrent attention model, which is parameterized by $\theta_{ram}$, and output a two-valued variable, indicating whether or not agent $i$ becomes the founder of the communication group.
If the output value is $1$, then agent $i$ becomes the founder of a new communication group; otherwise, agent $i$ chooses an existing communication group to join according to a pre-defined strategy.
After that, all messages in the communication group are fed into an LSTM network which is parameterized by $\theta_{lstm}$, and the output corresponding to the agent $i$ is used as a fusion representation of the message sent by the remaining agents in the communication group.
Finally, this fusion representation is fed into the actor-network, which is parameterized by $\theta_{actor}$, to obtain the final policy.
It can be seen that $\psi_{sh}:=\{\theta_{enc}, \theta_{ram}, \theta_{lstm}, \theta_{actor}\}$ and $\psi_{in}^i:=\emptyset$, $\phi_{sh}$ is the critic sharing between all agents and $\phi_{in}^i:=\emptyset$.

The overall optimization problem specified from (\ref{eq:framework-v3}) can be formulated as
\begin{equation}\label{eq:framework-ATOC}
    \min_{\psi_{sh}, \phi_{sh}} \alpha_1 \sum_{i=1}^n \mathbb{E}\left[Q^{\boldsymbol{\pi},i}_{\phi_{sh}}(o^i,\pi^i_{\psi_{sh}}(o^i))\right] + \alpha_2 \sum_{i=1}^n \mathbb{E} \left[ \left( Q^{\boldsymbol{\pi},i}_{\phi_{sh}}(o^i,a^i) - Q^{\boldsymbol{\pi},i}_{tg} \right)^2 \right].
\end{equation}

Note that here we omit the parameters regularization term $\mathcal{R}(\mathbf{w}_{in}, \mathbf{w}_{sh})$ in (\ref{eq:framework-v3}).
In the practical implementation of the ATOC algorithm, the regularization of the parameters is generally implemented by $L_2$ regularization or gradient norm clipping.
All detailed elements about ATOC are summarized in Table \ref{table:ATOC},
\begin{table}[tb!]
\centering
    \resizebox{\textwidth}{!}{\begin{tabular}{|c|c|c|c|c|c|}
	\hline
    \multicolumn{2}{|c|}{$\mathbf{w}_{in}$} & \multicolumn{2}{c|}{$\mathbf{w}_{sh}$}& \multicolumn{2}{c|}{Functions: $J^i_{\text{actor}}$ and $J^i_{\text{critic}}$} \\
	\hline
	$\psi_{in}^i$ & $\phi_{in}^i$ & $\psi_{sh}^i$ & $\phi_{sh}^i$ & $J^i_{\text{actor}} \left(\psi_{sh}, \phi_{sh} \right)$ & $J^i_{\text{critic}} \left(\psi_{in}, \phi_{sh} \right)$ \\
	\hline
	$\varnothing$ & $\varnothing$ & $\{\theta_{enc}, \theta_{ram}, \theta_{lstm}, \theta_{actor}\}$ & \makecell{$\pi^i(o^i)$(MLP)} & $\mathbb{E}\left[Q^{\boldsymbol{\pi},i}_{\phi_{sh}}(o^i,\pi^i_{\psi_{sh}}(o^i))\right]$ & $\mathbb{E} \left[ \left( Q^{\boldsymbol{\pi},i}_{\phi_{sh}}(o^i,a^i) - Q^{\boldsymbol{\pi},i}_{tg} \right)^2 \right]$ \\
	\hline
	\end{tabular}}
    \caption{Algorithm elements about ATOC.}
	\label{table:ATOC}
\end{table}
and the algorithm framework of ATOC for formulation (\ref{eq:framework-v3}) is show as the following form
\begin{equation}
\left\{\begin{array}{rl}
    \phi_{sh}^{k+1} &= \phi_{sh}^{k} - \alpha_1 \frac{1}{M} \sum_{m=1}^{M} \sum_{i=1}^n \nabla_{\phi_{sh}} \left( Q^{\boldsymbol{\pi}^k,i}_{\phi_{sh}^k}(o^{i,m}, a^{i,m}) - Q^{\boldsymbol{\pi}^k,i}_{tg} \right)^2 \\
    & \approx \arg\min\limits_{\phi_{sh}} \alpha_1 \sum\limits_{i=1}^{n} \mathbb{E} \left[ \left( Q^{\boldsymbol{\pi},i}_{\phi_{sh}}(o^i,a^i) - Q^{\boldsymbol{\pi},i}_{tg} \right)^2 \right], \\
    
    \{ \theta_{enc}, \theta_{actor} \}^{k+1} &= \{ \theta_{enc}, \theta_{actor} \}^{k} - \alpha_2 \frac{1}{M} \sum_{m=1}^{M} \sum_{i=1}^{n} \\
    &\qquad\nabla_{\{ \theta_{enc}, \theta_{actor} \}} \left( Q^{\boldsymbol{\pi},i}_{\phi_{sh}^{k+1}}(o^{i,m},\pi^i_{\psi_{sh}^{k}}(o^{i,m})) \right) \\
    & \approx \arg\min_{\{ \theta_{enc}, \theta_{actor} \}} \alpha_2 \sum\limits_{i=1}^{n} \mathbb{E}\left[Q^{\boldsymbol{\pi},i}_{\phi_{sh}}(o^i,\pi^i_{\psi_{sh}}(o^i))\right], \\
    
    \theta_{lstm}^{k+1} &= \theta_{lstm}^{k} - \alpha_2 \frac{1}{M} \sum_{m=1}^{M} \sum_{i=1}^{n} \nabla_{\theta_{lstm}} \left( Q^{\boldsymbol{\pi},i}_{\phi_{sh}^{k+1}}(o^{i,m},\pi^i_{\psi_{sh}^{k}}(o^{i,m})) \right) \\
    & \approx \arg\min_{\theta_{lstm}} \alpha_2 \sum\limits_{i=1}^{n} \mathbb{E}\left[Q^{\boldsymbol{\pi},i}_{\phi_{sh}}(o^i,\pi^i_{\psi_{sh}}(o^i))\right], \nonumber
\end{array}\right.
\end{equation}which can be considered to employ block coordinate gradient descent on formulation (\ref{eq:framework-ATOC}).


\subsection{Multi-Agent Common Knowledge (MACKRL)}
For the cooperative POSG, in the absence of common knowledge, complex decentralized coordination has to rely on implicit communication, i.e., observing each other’s actions or their effects~\citep{heider1944experimental, rasouli2017agreeing}. 
However, implicit communication protocols for complex coordination problems are challenging to learn and, as they typically require multiple timesteps to execute, can limit the agility of control during execution~\citep{tian2018learning}. 
By contrast, coordination based on common knowledge is simultaneous and does not require learning communication protocols~\citep{halpern1990knowledge}.

MACKRL~\citep{de2019multi} is a novel stochastic policy actor-critic algorithm that can learn complex coordination policies end-to-end by exploiting common knowledge between groups of agents at the appropriate level. 
MACKRL uses a hierarchical policy tree in order to select the right level of coordination dynamically.
Specifically, MACKRL is learning a centralized actor and centralized critic.
However, due to the use of hierarchical policies, the centralized actor can be calculated efficiently.
In other words, just like ATOC, MACKRL sharing all actor parameters and all critic parameters for all agents, which means $\phi_{sh}=\emptyset$ and $\psi_{sh}=\emptyset$.

MACKRL is based on \textit{Central-V}~\citep{foerster2018counterfactual} and approximately solves the Eq.~\eqref{eq:framework-v3} by iteratively optimizing the same two subproblems as above algorithms, and the specific form of $J^i_{\text{actor}}$ and $J^i_{\text{critic}}$ of agent $i$ in MACKRL algorithm are as follows:
\begin{eqnarray}
J^i_{\text{actor}} \left(\psi_{sh}, \phi_{sh} \right) &=& \mathbb{E}\left[Q^{\boldsymbol{\pi},i}_{\phi_{sh}}(\boldsymbol{o},\boldsymbol{a})\right], \nonumber\\
J^i_{\text{critic}} \left(\psi_{sh}, \phi_{sh} \right) &=& \mathbb{E} \left[ \left( Q^{\boldsymbol{\pi},i}_{\phi_{sh}}(\boldsymbol{o}, \boldsymbol{a}) - Q^{\boldsymbol{\pi},i}_{tg} \right)^2 \right], \nonumber
\end{eqnarray}
where the expectations are taken on $s \sim d_{\Psi}, \boldsymbol{o} \sim \mathcal{E}, a \sim \boldsymbol{\pi}_{\Psi}$ and $\Psi:=\{ \psi_{sh} \}$;
$\boldsymbol{\pi} := \boldsymbol{\pi}_{\psi_{sh}}$ represents the joint policy of all agents;
$\psi_{sh}$ and $\phi_{sh}$ represent the sharing actor parameter and critic parameter respectively;
MACKRL also uses TD($\lambda$) algorithm to learn the shared critic and it will not be repeated here because it was mentioned in the previous introduction of the COMA algorithm.

The way that MACKRL effectively calculates joint policy is similar to hierarchical reinforcement learning. 
It divides joint policy into common-knowledge-based group-level policies and common-knowledge-based agent-level policies, so $\psi_{sh}$ can be decomposed into $\psi_{group}$ and $\psi_{agent}$.
MACKRL allows multi-agent policies to introduce common knowledge while training end-to-end efficiently naturally (see the original paper~\citep{de2019multi} for details).

The overall optimization problem specified from (\ref{eq:framework-v3}) can be formulated as
\begin{equation}\label{eq:framework-MACKRL}
    \min_{\psi_{sh}, \phi_{sh}} \alpha_1 \sum_{i=1}^n \mathbb{E}\left[Q^{\boldsymbol{\pi},i}_{\phi_{sh}}(\boldsymbol{o},\boldsymbol{a})\right] + \alpha_2 \sum_{i=1}^n \mathbb{E} \left[ \left( Q^{\boldsymbol{\pi},i}_{\phi_{sh}}(\boldsymbol{o}, \boldsymbol{a}) - Q^{\boldsymbol{\pi},i}_{tg} \right)^2 \right].
\end{equation}

Note that here we omit the parameters regularization term $\mathcal{R}(\mathbf{w}_{in}, \mathbf{w}_{sh})$ in (\ref{eq:framework-v3}).
In the practical implementation of the MACKRL algorithm, the regularization of the parameters is generally implemented by $L_2$ regularization  or gradient norm clipping.
All detailed elements about MACKRL are summarized in Table \ref{table:MACKRL},
\begin{table}[tb!]
\centering
    \resizebox{\textwidth}{!}{\begin{tabular}{|c|c|c|c|c|c|}
	\hline
    \multicolumn{2}{|c|}{$\mathbf{w}_{in}$} & \multicolumn{2}{c|}{$\mathbf{w}_{sh}$}& \multicolumn{2}{c|}{Functions: $J^i_{\text{actor}}$ and $J^i_{\text{critic}}$} \\
	\hline
	$\psi_{in}^i$ & $\phi_{in}^i$ & $\psi_{sh}^i$ & $\phi_{sh}^i$ & $J^i_{\text{actor}} \left( \psi_{sh}, \phi_{sh} \right)$ & $J^i_{\text{critic}} \left( \psi_{sh}, \phi_{sh} \right)$ \\
	\hline
	$\varnothing$ & $\varnothing$ & $\{\psi_{group}, \psi_{agent}\}{\hbox{(GRUs)}}$ & \makecell{The parameter of \\ $Q$ and $V$(MLP)} & $\mathbb{E}\left[Q^{\boldsymbol{\pi},i}_{\phi_{sh}}(\boldsymbol{o},\boldsymbol{a})\right]$ & $\mathbb{E} \left[ \left( Q^{\boldsymbol{\pi},i}_{\phi_{sh}}(\boldsymbol{o}, \boldsymbol{a}) - Q^{\boldsymbol{\pi},i}_{tg} \right)^2 \right]$ \\
	\hline
	\end{tabular}}
    \caption{Algorithm information about MACKRL.}
	\label{table:MACKRL}
\end{table}
and the algorithm framework of MACKRL for formulation (\ref{eq:framework-v3}) is show as the following form
\begin{equation}
\left\{\begin{array}{ll}
    \phi_{sh}^{k+1} &= \phi_{sh}^{k} - \alpha_1 \frac{1}{M} \sum_{m=1}^{M} \sum_{i=1}^n \nabla_{\phi_{sh}} \left( Q^{\boldsymbol{\pi}^k,i}_{\phi_{sh}^k}(\boldsymbol{o^m}, \boldsymbol{a^m}) - Q^{\boldsymbol{\pi}^k,i}_{tg} \right)^2 \\
    & \approx \arg\min\limits_{\phi_{sh}} \alpha_1 \sum\limits_{i=1}^{n} \mathbb{E} \left[ \left( Q^{\boldsymbol{\pi},i}_{\phi_{sh}}(\boldsymbol{o}, \boldsymbol{a}) - Q^{\boldsymbol{\pi},i}_{tg} \right)^2 \right], \\
    
    \psi_{sh}^{k+1} &= \psi_{sh}^{k} - \alpha_2 \frac{1}{M} \sum_{m=1}^{M} \sum_{i=1}^{n} \nabla_{\psi_{sh}} \left(Q^{\boldsymbol{\pi}^k,i}_{\phi_{sh}^{k+1}}(\boldsymbol{o}^m,\boldsymbol{a}^m)\right)  \\
    & \approx \arg \min\limits_{\psi_{sh}} \alpha_2 \sum\limits_{i=1}^{n} \mathbb{E}\left[Q^{\boldsymbol{\pi},i}_{\phi_{sh}}(\boldsymbol{o},\boldsymbol{a})\right], \nonumber
\end{array}\right.
\end{equation}which can be considered to employ block coordinate gradient descent on formulation (\ref{eq:framework-MACKRL}).

\subsection{Multi-Agent Deep Deterministic Policy Gradient (MADDPG)}
MADDPG algorithm adopting the framework of CTDE. 
Thus, it allows the policies to use extra information to ease training so long as it is not used at test time.
However, it is unnatural to do this with Q-learning, as the Q function generally cannot contain different information at training and test time.
Thus, MADDPG proposes a simple extension of actor-critic methods based on the DDPG algorithm, where the critic is augmented with extra information about other agents' policies.

Since the MADDPG algorithm is directly extended on the DDPG, the difference from the DDPG algorithm is that additional information of the other agents is introduced when train the centralized critic. 
For each agent, it is still essentially completing a single-agent task. 
Therefore, there is no shared parameter between agents, which means $\psi_{sh}=\varnothing$ and $\phi_{sh}=\varnothing$.
Because this paper focuses on the collaboration problem, minor modifications to MADDPG are made here.
That is, all agents share a centralized Q-value function, which means $\phi_{in}^i=\varnothing$.
The objective function for each agent $i$ in actor and critic phases can be reformulated as
\begin{eqnarray}
J^i_{\text{actor}} \left( \{ \psi_{in}^i\}, \phi_{sh}\right) &=& \mathbb{E}\left[Q^{\boldsymbol{\pi},i}_{\phi_{sh}}(\boldsymbol{o},\boldsymbol{a})\right], \nonumber\\
J^i_{\text{critic}} \left( \{ \psi_{in}^i\}, \phi_{sh} \right) &=& \mathbb{E} \left[ \left( Q^{\boldsymbol{\pi},i}_{\phi_{sh}}(\boldsymbol{o}, \boldsymbol{a}) - Q^{\boldsymbol{\pi},i}_{tg} \right)^2 \right], \nonumber
\end{eqnarray}
where the expectations are take on $\boldsymbol{o}, \boldsymbol{a}, \boldsymbol{c}, \boldsymbol{o'} \sim \mathcal{D}$ because the DDPG algorithm are off-policy algorithm, and $\mathcal{D}$ is the shared experience replay buffer;
$\boldsymbol{\pi} := \left\{\pi^i_{\psi_{in}^i}\right\}$ represents the policy of each agent;
$\psi_{in}^i$ and $\phi_{sh}$ represent the independent actor parameter and shared critic parameter respectively;
$Q^{\boldsymbol{\pi},i}_{tg}:=c^i + \gamma  Q^{\boldsymbol{\pi},i}_{\phi_{sh}}\left(\boldsymbol{o'},\boldsymbol{\pi}(\boldsymbol{o'})\right)$.

The overall optimization problem specified from (\ref{eq:framework-v3}) can be formulated as
\begin{equation}\label{eq:framework-MADDPG}
    \min_{\psi_{sh}, \phi_{sh}} \alpha_1 \sum_{i=1}^n \mathbb{E}\left[Q^{\boldsymbol{\pi},i}_{\phi_{sh}}(\boldsymbol{o},\boldsymbol{a})\right] + \alpha_2 \sum_{i=1}^n \mathbb{E} \left[ \left( Q^{\boldsymbol{\pi},i}_{\phi_{sh}}(\boldsymbol{o}, \boldsymbol{a}) - Q^{\boldsymbol{\pi},i}_{tg} \right)^2 \right].
\end{equation}

All detailed information about MADDPG is summarized in Table \ref{table:MADDPG},
\begin{table}[tb!]
\centering
    \scalebox{0.9}{\begin{tabular}{|c|c|c|c|c|c|}
	\hline
    \multicolumn{2}{|c|}{$\mathbf{w}_{in}$} & \multicolumn{2}{c|}{$\mathbf{w}_{sh}$}& \multicolumn{2}{c|}{Functions: $J^i_{\text{actor}}$ and $J^i_{\text{critic}}$} \\
	\hline
	$\psi_{in}^i$ & $\phi_{in}^i$ & $\psi_{sh}^i$ & $\phi_{sh}^i$ & $J^i_{\text{actor}} \left(\{ \psi_{in}^i\}, \phi_{sh} \right)$ & $J^i_{\text{critic}} \left(\{ \psi_{in}^i\}, \phi_{sh} \right)$ \\
	\hline
	$\pi^i(o^i)$ & $\varnothing$ & $\varnothing$ & $Q(\boldsymbol{o},\boldsymbol{a})$ & $\mathbb{E}\left[Q^{\boldsymbol{\pi},i}_{\phi_{sh}}(\boldsymbol{o},\boldsymbol{a})\right]$ & $\mathbb{E} \left[ \left( Q^{\boldsymbol{\pi},i}_{\phi_{sh}}(\boldsymbol{o}, \boldsymbol{a}) - Q^{\boldsymbol{\pi},i}_{tg} \right)^2 \right]$ \\
	\hline
	\end{tabular}}
    \caption{Algorithm information about MADDPG.}
	\label{table:MADDPG}
\end{table}
and the algorithm framework of MADDPG for formulation (\ref{eq:framework-v3}) is show as the following form
\begin{equation}
\left\{\begin{array}{lll}
    \phi_{sh}^{k+1} &= \phi_{sh}^{k} - \alpha_1 \frac{1}{M} \sum_{m=1}^{M} \sum_{i=1}^{N} \nabla_{\phi_{sh}} \left( Q^{\boldsymbol{\pi}^k,i}_{\phi_{sh}^{k}}(\boldsymbol{o}^m, \boldsymbol{a}^m) - Q^{\boldsymbol{\pi}^k,i}_{tg} \right)^2 \\
    & \approx \arg\min\limits_{\phi_{sh}} \alpha_1 \sum\limits_{i=1}^{n} \mathbb{E} \left[ \left( Q^{\boldsymbol{\pi},i}_{\phi_{sh}}(\boldsymbol{o}, \boldsymbol{a}) - Q^{\boldsymbol{\pi},i}_{tg} \right)^2 \right], \\
    
    \{\psi_{in}^{i}\}^{k+1} &= \{\psi_{in}^{i}\}^{k} - \alpha_2 \frac{1}{M} \sum_{m=1}^{M} \sum_{n=1}^{N} \nabla_{\{\psi_{in}^i\}} \left( Q^{\boldsymbol{\pi}^k,i}_{\phi_{sh}^{k+1}}(\boldsymbol{o}^m, \boldsymbol{a}^m) - Q^{\boldsymbol{\pi}^k,i}_{tg} \right)^2 \\
    & \approx \arg\min\limits_{\{\psi_{in}^i\}} \alpha_2 \sum\limits_{i=1}^{n} \mathbb{E} \left[ \left( Q^{\boldsymbol{\pi},i}_{\phi_{sh}}(\boldsymbol{o}, \boldsymbol{a}) - Q^{\boldsymbol{\pi},i}_{tg} \right)^2 \right], \nonumber
\end{array}\right.
\end{equation}which can be considered to employ block coordinate gradient descent for formulation (\ref{eq:framework-MADDPG}).

\subsection{Counterfactual Multi-Agent Policy Gradients (COMA)}

In fully cooperative POSG, joint actions typically generate only the global cost (i.e., the same cost function sharing between agents), making it difficult for each agent to deduce its contribution to the team’s success.
In some cases, it is possible to design an individual cost function for each agent. 
However, these costs are not generally available in a collaborative setting and usually fail to encourage individual agents to sacrifice to obtain better global performance.
This will substantially impede multi-agent learning in challenging tasks, even with a relatively small number of agents \citep{foerster2018counterfactual}.
This crucial challenge is called \textit{multi-agent credit assignment} problem \citep{chang2004all}.
COMA \citep{foerster2018counterfactual} solves this problem by learning a centralized critic with a \textit{counterfactual baseline} which is inspired by \textit{difference rewards} \citep{wolpert2002optimal}.

In the COMA algorithm, for the reason that a centralized critic is employed, so that all the parameters in critic are shared, which means $\phi_{in}^i = \varnothing$ for all agents.
However there is no shared parameter in actors, which means $\psi_{sh} = \varnothing$.
Further, the objective function for each agent $i$ in actor and critic phases can be reformulated as
\begin{eqnarray}
J^i_{\text{actor}} \left( \{ \psi_{in}^i\}, \phi_{sh} \right) &=& \mathbb{E}\left[Q^{\boldsymbol{\pi},i}_{\phi_{sh}}(\boldsymbol{o},\boldsymbol{a})  - \mathcal{B}(\boldsymbol{o}, \boldsymbol{a^{\setminus i}})\right], \nonumber\\
J^i_{\text{critic}} \left( \{ \psi_{in}^i\}, \phi_{sh} \right) &=& \mathbb{E} \left[ \left( Q^{\boldsymbol{\pi},i}_{\phi_{sh}}(\boldsymbol{o}, \boldsymbol{a}) - Q^{\boldsymbol{\pi},i}_{tg} \right)^2 \right], \nonumber
\end{eqnarray}where
where the expectations are taken on $s \sim d_{\Psi}, \boldsymbol{o} \sim \mathcal{E}, a \sim \boldsymbol{\pi}_{\Psi}$ and $\Psi:=\{ \psi_{in}^i \}$;
$\boldsymbol{\pi} := \left\{\pi^i_{\psi^i_{in}}\right\}$ represents the joint policy of all agents;
$\phi_{sh}$ represents the sharing parameter of the centralized critic;
$\psi_{in}^i$ represents the independent actor parameter of each agent;
COMA introduces TD($\lambda$) \citep{sutton2018reinforcement} for critic learning, thus $Q^{\boldsymbol{\pi},i}_{tg}$ is also denoted as $G_{t}^{\lambda}:=(1-\lambda)\sum_{n=1}^{\infty}\lambda^{n-1}G_{t}^{(n)}$;
when $n=1$, we have
$
G_{t}^{1}:=\mathbf{E}_{s'\sim\mathcal{P}, \boldsymbol{o}^{'}\sim\cal{E}}\left(\mathcal{C}_{s, \boldsymbol{a}}^{s'}+ \gamma V^{\boldsymbol{\pi},i}_{\phi_{sh}}\left(\boldsymbol{o}'\right) \right),
$
and $V$ represents the approximated \textit{state-value function} \citep{sutton2018reinforcement};
$\mathcal{B}(\boldsymbol{o}, \boldsymbol{a^{\setminus i}})$ denotes the \textit{multi-agent counterfactual baseline} that is used to solve the credit assignment problem,
$
\mathcal{B}(\boldsymbol{o}, \boldsymbol{a^{\setminus i}})=\mathbb{E}_{a^{i} \sim \pi^{i}\left(o^{i}\right)}\left[Q^{\boldsymbol{\pi},i}_{\phi_{sh}}(\boldsymbol{o}, (a^i, \boldsymbol{a^{\setminus i}}))\right].
$
The overall optimization problem specified from (\ref{eq:framework-v3}) can be formulated as
\begin{equation}\label{eq:framework-COMA}
    \min_{\{ \psi_{in}^i \}, \phi_{sh}} \alpha_1 \sum_{i=1}^n \mathbb{E}\left[Q^{\boldsymbol{\pi},i}_{\phi_{sh}}(\boldsymbol{o},\boldsymbol{a})  - \mathcal{B}(\boldsymbol{o}, \boldsymbol{a^{\setminus i}})\right] + \alpha_2 \sum_{i=1}^n \mathbb{E} \left[ \left( Q^{\boldsymbol{\pi},i}_{\phi_{sh}}(\boldsymbol{o}, \boldsymbol{a}) - Q^{\boldsymbol{\pi},i}_{tg} \right)^2 \right].
\end{equation}
Note that here we omit the parameters regularization term $\mathcal{R}(\mathbf{w}_{in}, \mathbf{w}_{sh})$ in (\ref{eq:framework-v3}).
In the practical implementation of the COMA algorithm, the regularization of the parameters is generally implemented by $L_2$ regularition or gradient norm clipping.
All detailed elements about COMA are summarized in Table \ref{table:COMA}.
\begin{table}[tb!]
\centering
    \resizebox{0.98\textwidth}{!}{\begin{tabular}{|c|c|c|c|c|c|}
	\hline
    \multicolumn{2}{|c|}{$\mathbf{w}_{in}$} & \multicolumn{2}{c|}{$\mathbf{w}_{sh}$}& \multicolumn{2}{c|}{Functions: $J^i_{\text{actor}}$ and $J^i_{\text{critic}}$} \\
	\hline
	$\psi_{in}^i$ & $\phi_{in}^i$ & $\psi_{sh}$ & $\phi_{sh}$ & $J^i_{\text{actor}} \left( \{ \psi_{in}^i\}, \phi_{sh} \right)$ & $J^i_{\text{critic}} \left( \{ \psi_{in}^i\}, \phi_{sh} \right)$ \\
	\hline
	\makecell{$\pi^i(o^i)$(GRUs\footnote{Gated recurrent units.})} & $\varnothing$ & $\varnothing$ & \makecell{$Q(\boldsymbol{o}, \boldsymbol{a})$ and $V(\boldsymbol{o})$(MLP\footnote{Multiple layer perceptron.})} & $\mathbb{E}\left[Q^{\boldsymbol{\pi},i}_{\phi_{sh}}(\boldsymbol{o},\boldsymbol{a})  - \mathcal{B}(\boldsymbol{o}, \boldsymbol{a^{\setminus i}})\right]$ & $\mathbb{E} \left[ \left( Q^{\boldsymbol{\pi},i}_{\phi_{sh}}(\boldsymbol{o}, \boldsymbol{a}) - Q^{\boldsymbol{\pi},i}_{tg} \right)^2 \right]$ \\
	\hline
	\end{tabular}}
    \caption{Algorithm elements about COMA.}
	\label{table:COMA}
\end{table}
Further the algorithm framework of COMA algorithm is shown follows, i.e.,
\begin{equation}
\left\{\begin{array}{ll}
    \phi_{sh}^{k+1} &= \phi_{sh}^{k} - \alpha_1 \frac{1}{M} \sum_{m=1}^{M} \sum_{i=1}^n \nabla_{\phi_{sh}} \left( Q^{\boldsymbol{\pi}^k,i}_{\phi_{sh}^k}(\boldsymbol{o^m}, \boldsymbol{a^m}) - Q^{\boldsymbol{\pi}^k,i}_{tg} \right)^2 \\
    & \approx \arg\min\limits_{\phi_{sh}} \alpha_1 \sum\limits_{i=1}^{n} \mathbb{E} \left[ \left( Q^{\boldsymbol{\pi},i}_{\phi_{sh}}(\boldsymbol{o}, \boldsymbol{a}) - Q^{\boldsymbol{\pi},i}_{tg} \right)^2 \right], \\
    \{\psi_{in}^{i, k+1}\} &= \{\psi_{in}^{i, k}\} - \alpha_2 \frac{1}{M} \sum_{m=1}^{M} \sum_{i=1}^{n} \nabla_{\{\psi_{in}^i\}} \left(Q^{\boldsymbol{\pi}^k,i}_{\phi_{sh}^{k+1}}(\boldsymbol{o}^m,\boldsymbol{a}^m)  - \mathcal{B}(\boldsymbol{o}^m, \boldsymbol{a^{\setminus i,m}})\right)  \\
    & \approx \arg \min\limits_{\{\psi_{in}^i\}} \alpha_2 \sum\limits_{i=1}^{n} \mathbb{E}\left[Q^{\boldsymbol{\pi},i}_{\phi_{sh}}(\boldsymbol{o},\boldsymbol{a})  - \mathcal{B}(\boldsymbol{o}, \boldsymbol{a^{\setminus i}})\right], \nonumber
\end{array}\right.
\end{equation}which can be considered to employ block coordinate gradient descent on formulation (\ref{eq:framework-COMA}).

\subsection{Multi-Actor-Attention-Critic (MAAC)}
A large-scale cooperative multi-agent system is complicated and unnecessary for an agent to surveillance all agents' states and behavior.
Meanwhile, the decision of each agent may only be affected by these strongly related agents, not all agents.
Considering too much other agents' information will make proper signals inevitably submerged in the background noise.
Therefore, the multiple attention actor-critic (MAAC) algorithm \citep{iqbal2019actor} introduces the effective attention mechanism to avoid the instability problem of estimating other agents' policy in MADDPG \citep{lowe2017multi}.
MAAC follows the learning procedure of centralized training with decentralized execution.
Based on the popular Soft-Actor-Critic (SAC) algorithm \citep{haarnoja2018soft}, MAAC considers an additional attention layer to avoid directly using other agents' policies, and the policy is determined by maximizing a trade-off between expected return and the entropy regularization.

The MAAC algorithm introduces a shared attention module in the modeling of each agent's critic, so that $\phi_{in}^i$ and $\phi_{sh}$ both \textit{non-empty} parameter sets.
Moreover, same as COMA, there is no shared parameter in actors, which means $\psi_{sh} =\emptyset$.
The objective function for each agent $i$ in actor and critic phases can be reformulated as
\begin{eqnarray}
J^i_{\text{actor}} \left( \{ \psi_{in}^i\}, \phi_{in}^i, \phi_{sh} \right) &=& \mathbb{E}\left[Q^{\boldsymbol{\pi},i}_{\phi_{in}^i, \phi_{sh}}(\boldsymbol{o},\boldsymbol{a}) + \alpha \mathcal{H}(\cdot|\pi^i_{\psi_{in}^i}(o^i)) - \mathcal{B}(\boldsymbol{o}, \boldsymbol{a^{\setminus i}})\right], \nonumber\\
J^i_{\text{critic}} \left( \{ \psi_{in}^i\}, \phi_{in}^i, \phi_{sh} \right) &=& \mathbb{E} \left[ \left( Q^{\boldsymbol{\pi},i}_{\phi_{in}^i, \phi_{sh}}(\boldsymbol{o}, \boldsymbol{a}) - Q^{\boldsymbol{\pi},i}_{tg} \right)^2 \right], \nonumber
\end{eqnarray}
where the expectations are take on $\boldsymbol{o}, \boldsymbol{a}, \boldsymbol{c}, \boldsymbol{o'} \sim \mathcal{D}$ because the SAC algorithm are off-policy algorithm, and $\mathcal{D}$ is the shared experience replay buffer;
$\boldsymbol{\pi} := \left\{\pi^i_{\psi^i_{in}}\right\}$ represents the joint policy of all agents;
$\phi_{sh}$ represents the sharing attention parameter of the centralized critic and $\phi_{in}^i$ represents the rest independent critic parameter of each agent;
$\psi_{in}^i$ represents the independent actor parameter of each agent;
$\mathcal{H}(\cdot|\pi^i_{\psi_{in}^i}(o^i))$ denotes the entropy of the policy at state $o^i$;
$\mathcal{B}(\boldsymbol{o}, \boldsymbol{a^{\setminus i}})$ is the multi-agent counterfactual baseline that is same as COMA;
$Q^{\boldsymbol{\pi},i}_{tg}$ is defined as
$$
c_{i}+\gamma \mathbb{E}_{\boldsymbol{a'} \sim \boldsymbol{\pi}\left(\boldsymbol{o'}\right)} \left[Q^{\boldsymbol{\pi},i}_{\phi_{in}^i, \phi_{sh}}(\boldsymbol{o'}, \boldsymbol{a'})-\alpha \log \left(\pi^i_{\psi_{in}^i}\left(a_{i}^{\prime} | o_{i}^{\prime}\right)\right)\right].
$$
Finally, the detailed form of critic with attention mechanism is denoted as
$$
Q^{\boldsymbol{\pi},i}_{\phi_{in}^i, \phi_{sh}}(\boldsymbol{o}, \boldsymbol{a}) = \zeta^i \left(\digamma^i (o^i, a^i), \sum_{h=1}^{H}\sum_{j\neq i}\alpha_h^{j} \Upsilon(V_{h} \digamma^j (o^j, a^j))\right),
$$
where $\phi_{in}^i:=\{ \zeta^i, \digamma^i \}$ and $\zeta^i, \digamma^i$ are two-layer multi-layer perceptron (MLP) and one-layer MLP encoding function respectively. 
$\Upsilon$ denotes a specific nonlinear activation function, and
$$\phi_{sh}:=\left\{ V_h, W_h^{key}, W_h^{que} \right\}_{h=1}^{H},$$
where $V_h$ represents attention module parameters and $\alpha_{h}^{j\neq i} \propto\exp((\digamma^{j\neq i})^{T}(W_h^{key})^{\rm T}W_h^{que}\digamma^{i})$ represents the attention factor.
$H$ denotes the number of attention heads.

The overall optimization problem specified from (\ref{eq:framework-v3}) can be formulated as
\begin{equation}\label{eq:framework-MAAC}
    \begin{aligned}
        \min_{ \{ \psi_{in}^i\}, \phi_{in}^i, \phi_{sh}} &\alpha_1 \sum_{i=1}^n \mathbb{E}\left[Q^{\boldsymbol{\pi},i}_{\phi_{in}^i, \phi_{sh}}(\boldsymbol{o},\boldsymbol{a}) + \alpha \mathcal{H}(\cdot|\pi^i_{\psi_{in}^i}(o^i)) - \mathcal{B}(\boldsymbol{o}, \boldsymbol{a^{\setminus i}})\right] + \\
        &\qquad\alpha_2 \sum_{i=1}^n \mathbb{E} \left[ \left( Q^{\boldsymbol{\pi},i}_{\phi_{in}^i, \phi_{sh}}(\boldsymbol{o}, \boldsymbol{a}) - Q^{\boldsymbol{\pi},i}_{tg} \right)^2 \right]. 
    \end{aligned}
\end{equation}

Note that here we omit the parameters regularization term $\mathcal{R}(\mathbf{w}_{in}, \mathbf{w}_{sh})$ in (\ref{eq:framework-v3}).
In the practical implementation of the MAAC algorithm, the regularization of the parameters is generally implemented by $L_2$ regularization  or gradient norm clipping.
All detailed elements about MAAC are summarized in Table \ref{table:MAAC}.

\begin{table}[tb!]
\centering
    \resizebox{0.98\textwidth}{!}{\begin{tabular}{|c|c|c|c|c|c|}
	\hline
    \multicolumn{2}{|c|}{$\mathbf{w}_{in}$} & \multicolumn{2}{c|}{$\mathbf{w}_{sh}$}& \multicolumn{2}{c|}{Functions: $J^i_{\text{actor}}$ and $J^i_{\text{critic}}$} \\
	\hline
	$\psi_{in}^i$ & $\phi_{in}^i$ & $\psi_{sh}^i$ & $\phi_{sh}^i$ & $J^i_{\text{actor}} \left( \{ \psi_{in}^i\}, \phi_{in}^i, \phi_{sh} \right)$ & $J^i_{\text{critic}} \left( \{ \psi_{in}^i\}, \phi_{in}^i, \phi_{sh} \right)$ \\
	\hline
	\makecell{$\pi^i(o^i)$(MLP)} & $\zeta^i, \digamma^i$ & $\varnothing$ & $\left\{ V_h, W_h^{key}, W_h^{que} \right\}_{h=1}^{H}$ & $\mathbb{E}\left[Q^{\boldsymbol{\pi},i}_{\phi_{in}^i, \phi_{sh}}(\boldsymbol{o},\boldsymbol{a}) + \alpha \mathcal{H}(\cdot|\pi^i_{\psi_{in}^i}(o^i)) - \mathcal{B}(\boldsymbol{o}, \boldsymbol{a^{\setminus i}})\right]$ & $\mathbb{E} \left[ \left( Q^{\boldsymbol{\pi},i}_{\phi_{in}^i, \phi_{sh}}(\boldsymbol{o}, \boldsymbol{a}) - Q^{\boldsymbol{\pi},i}_{tg} \right)^2 \right]$ \\
	\hline
	\end{tabular}}
    \caption{Algorithm elements about MAAC.}
	\label{table:MAAC}
\end{table}
Futher the algorithm framework of MAAC for formulation (\ref{eq:framework-MAAC}) is shown as follows,
\begin{small}
\begin{equation}
\left\{\begin{array}{lll}
    \!\!\!\!&\!\!\!\!\!\!\!\!\left\{ V_h, W_h^{key}, W_h^{que} \right\}^{k+1} = \left\{ V_h, W_h^{key}, W_h^{que} \right\}^{k} - \alpha_1 \frac{1}{M} \sum_{m=1}^{M} \sum_{i=1}^n \\
    &\qquad\nabla_{\left\{ V_h, W_h^{key}, W_h^{que} \right\}} \left( Q^{\boldsymbol{\pi}^k,i}_{\phi_{in}^{i,k}, \phi_{sh}^k}(\boldsymbol{o^m}, \boldsymbol{a^m}) - Q^{\boldsymbol{\pi}^k,i}_{tg} \right)^2 \\
    &\qquad\qquad\qquad\qquad \approx \arg\min_{{\left\{ V_h, W_h^{key}, W_h^{que} \right\}}} \alpha_1 \sum\limits_{i=1}^{n} \mathbb{E} \left[ \left( Q^{\boldsymbol{\pi},i}_{\phi_{in}^{i}, \phi_{sh}}(\boldsymbol{o}, \boldsymbol{a}) - Q^{\boldsymbol{\pi},i}_{tg} \right)^2 \right], \vspace{10pt} \\
    \!\!\!\!&\!\!\!\!\!\!\!\!\left\{ \zeta^i, \digamma^i \right\}^{k+1} = \left\{ \zeta^i, \digamma^i \right\}^{k} - \alpha_1 \frac{1}{M} \sum_{m=1}^{M} \nabla_{\left\{ \zeta^i, \digamma^i \right\}} \left( Q^{\boldsymbol{\pi}^k,i}_{\phi_{in}^{i,k}, \phi_{sh}^k}(\boldsymbol{o^m}, \boldsymbol{a^m}) - Q^{\boldsymbol{\pi}^k,i}_{tg} \right)^2 \\
    &\qquad\qquad \approx \arg\min_{{\left\{ \zeta^i, \digamma^i \right\}}} \alpha_1 \sum\limits_{i=1}^{n} \mathbb{E} \left[ \left( Q^{\boldsymbol{\pi},i}_{\phi_{in}^{i}, \phi_{sh}}(\boldsymbol{o}, \boldsymbol{a}) - Q^{\boldsymbol{\pi},i}_{tg} \right)^2 \right], \vspace{10pt} \\
    \!\!\!\!&\!\!\!\!\!\!\!\! \{\psi_{in}^{i, k+1}\} = \{\psi_{in}^{i, k}\} - \alpha_2 \frac{1}{M} \sum_{m=1}^{M} \sum_{i=1}^{n} \\
    &\qquad\nabla_{\{\psi_{in}^i\}} \left(Q^{\boldsymbol{\pi}^k,i}_{\phi_{in}^{i,k+1}, \phi_{sh}^{k+1}}(\boldsymbol{o}^m,\boldsymbol{a}^m) + \alpha \mathcal{H}(\cdot|\pi^i_{\psi_{in}^{i,k}}(o^{i,m})) - \mathcal{B}(\boldsymbol{o}^m, \boldsymbol{a^{\setminus i,m}})\right)  \\
    &\qquad\ \ \approx \arg \min\limits_{\{\psi_{in}^i\}} \alpha_2 \sum\limits_{i=1}^{n} \mathbb{E}\left[Q^{\boldsymbol{\pi},i}_{\phi_{in}^i, \phi_{sh}}(\boldsymbol{o},\boldsymbol{a}) + \alpha \mathcal{H}(\cdot|\pi^i_{\psi_{in}^i}(o^i)) - \mathcal{B}(\boldsymbol{o}, \boldsymbol{a^{\setminus i}})\right], \nonumber
\end{array}\right.
\end{equation}
\end{small} which can be considered to employ block coordinate gradient descent on formulation (\ref{eq:framework-MAAC}).

{\color{black}{
\section{Limitations}\label{sec:limit-appendix}

We ensured the Markov property of the policy through the enrichment of the observation space, avoiding the use of non-Markovian and history-dependent policy classes. 
Although this approach facilitates theoretical analysis and engineering implementation, the resulting \textit{vast belief space} and \textit{infinite hierarchy of belief} render the Markov policy class intractable in more complex problems, further constraining its generality.
We shed light on the limitations that affect this policy class concerning crucial factors such as the error of the approximated information state and the infinite hierarchy belief representation capabilities. 
Careful consideration of these factors is essential to enhance the tractability and generality of the Markovian policy within the enriched observation space.

\subsection{Error of the Approximated Information State}

Lifting POMDPs to an MDP over the belief states would require estimating the belief states based on the whole history. 
This estimation is prone to the curse of dimensionality due to the ample space of belief states. 
To analyze the theoretical impact of this curse on algorithm performance, we adopt the perspective of \textit{information state} based on existing work~\citep{subramanian2022approximate}.

It is pertinent to mention that the concepts and theories related to the information state in POSG~\citep{mao2020information} are an extension of those in POMDP~\citep{subramanian2022approximate}. 
A comprehensive description of these concepts would require a significant amount of space, and it is not the primary focus of this paper. 
Instead, we will concentrate on a specific agent in a POSG and gradually deduce the limitations of Markovian policies from a single-agent perspective. 
This approach allows us to omit the discussion of \textit{common} and \textit{private information} and their relevant content~\citep{tavafoghi2021unified,mao2020information} while retaining the generalizability of our analysis.
Concretely, we have the following definition (for a single agent in POSG):

\begin{definition}[\citealt{subramanian2022approximate}]
    Let $\left\{\mathrm{Z}_t\right\}_{t=1}^T$ be a pre-specified collection of Banach spaces,  $\mathrm{H}_t:=\{\boldsymbol{O}_{1:t-1},\boldsymbol{A}_{1:t-1},\}$ be the past (joint) observations and actions.
    A collection $\left\{\sigma_t: \mathrm{H}_t \rightarrow \mathrm{Z}_t\right\}_{t=1}^T$ of history compression functions is called an \textit{information state generator} and $\left\{Z_t\right\}_{t=1}^T$ is called an \textit{information state}, if the process $\left\{Z_t\right\}_{t=1}^T$, where $Z_t=\sigma_t\left(H_t\right)$, satisfies the following properties:
    
    \noindent\textbf{(P1) Sufficient for performance evaluation}, i.e., for any timestep $t$, any realization $h_t$ of $H_t$ and any choice $\boldsymbol{a}_t$ of $\boldsymbol{A}_t$, we have
    $$
    \mathbb{E}\left[R_t \mid H_t=h_t, \boldsymbol{A}_t=\boldsymbol{a}_t\right]=\mathbb{E}\left[R_t \mid Z_t=\sigma_t\left(h_t\right), \boldsymbol{A}_t=\boldsymbol{a}_t\right] .
    $$
    \textbf{(P2) Sufficient to predict itself}, i.e., for any time $t$, any realization $h_t$ of $H_t$ and any choice $\boldsymbol{a}_t$ of $\boldsymbol{A}_t$, we have that for any Borel subset $\mathrm{B}$ of $\mathrm{Z}_{t+1}$,
    $$
    \mathbb{P}\left(Z_{t+1} \in \mathrm{B} \mid H_t=h_t, \boldsymbol{A}_t=\boldsymbol{a}_t\right)=\mathbb{P}\left(Z_{t+1} \in \mathrm{B} \mid Z_t=\sigma_t\left(h_t\right), \boldsymbol{A}_t=\boldsymbol{a}_t\right) .
    $$
\end{definition}

It is evident that the history $\mathrm{H}_t$ is a trivial information state for any partially observed model~\citet{subramanian2022approximate}. 
In F2A2, we enrich the observation space based on all histories to create an information state space. 
However, for computational efficiency, F2A2 learns a Markovian policy in the approximated information state space.
To accomplish this, we first perform a finite memory compression (FMC) of the entire history by maintaining a fixed window of the local history as an approximate information state. 
We then use LSTM, a type of recurrent neural network, to encode the truncated/approximated information state.

Nevertheless, truncating and embedding the information state leads to errors. 
As per theoretical results from~\citet{subramanian2022approximate} and~\citet{mao2020information}, the distance between the optimal Markovian policy learned in the approximated information state space and the actual optimal Markovian policy is bounded by $(T-t+1)(\varepsilon+L_V \delta)$, where $L_V$ represents the upper bound of the Lipschitz constant of the policy and value function.
The values of $(\varepsilon, \delta)$ are defined as follows:

\begin{definition}
    Let $\left\{\hat{Z}_t\right\}_{t=1}^T$ be a pre-specified collection of Banach spaces, $\mathfrak{F}$ be a probability metrics with a $\zeta$-structure~\citep{zolotarev1983probability}, and $\left(\varepsilon, \delta\right)$ be pre-specified positive real numbers. A collection $\left\{\hat{\sigma}_t: \mathrm{H}_t \rightarrow\right.$ $\left.\hat{Z}_t\right\}_{t=1}^T$ of history compression functions, along with approximate update kernels $\left\{\hat{P}_t: \hat{Z}_t \times\right.$ $\left.\boldsymbol{A} \rightarrow \Delta\left(\hat{Z}_{t+1}\right)\right\}_{t=1}^T$ and reward approximation functions $\left\{\hat{r}_t: \hat{Z}_t \times \boldsymbol{A} \rightarrow \mathbb{R}\right\}_{t=1}^T$, is called an $\left(\varepsilon, \delta\right)$ approximate information state generator and $\left\{Z_t\right\}_{t=1}^T$ is called an \textit{approximate information state} if the process $\left\{\hat{Z}_t\right\}_{t=1}^T$, where $\hat{Z}_t=\hat{\sigma}_t\left(H_t\right)$, satisfies the following properties:
    
    \noindent\textbf{(AP1) Sufficient for approximate performance evaluation}, i.e., for any time t, any realization $h_t$ of $H_t$ and any choice $\boldsymbol{a}_t$ of $\boldsymbol{A}_t$, we have
    $$
    \left|\mathbb{E}\left[R_t \mid H_t=h_t, \boldsymbol{A}_t=\boldsymbol{a}_t\right]-\hat{r}_t\left(\hat{\sigma}_t\left(h_t\right), \boldsymbol{a}_t\right)\right| \leq \varepsilon
    $$
    \textbf{(AP2) Sufficient to predict itself approximately}, i.e., for any time t, any realization $h_t$ of $H_t$, any choice $\boldsymbol{a}_t$ of $\boldsymbol{A}_t$, and for any Borel subset $\mathrm{B}$ of $\hat{Z}_{t+1}$, define $\mu_t(\mathrm{~B}):=$ $\mathbb{P}\left(\hat{Z}_{t+1} \in B \mid H_t=h_t, \boldsymbol{A}_t=\boldsymbol{a}_t\right)$ and $\nu_t(\mathrm{~B}):=\hat{P}_t\left(B \mid \hat{\sigma}_t\left(h_t\right), \boldsymbol{a}_t\right) ;$ then
    $$
    d_{\mathfrak{F}}\left(\mu_t, \nu_t\right) \leq \delta.
    $$
\end{definition}

The theoretical upper bound for the $(\varepsilon, \delta)$ values remains uncertain due to the need for quantifying the expressiveness of the LSTM.
As the state space, action space, and task horizon increase, the error caused by using FMC to estimate the information state will likely amplify, resulting in a more significant gap between the policy converged by F2A2 and the optimal policy. 
Therefore, our future research will focus on more efficient methods of estimating the information state based on existing work~\citet{mao2020information}. 

Furthermore, the \textit{infinite hierarchy of beliefs} in the POSG provides an alternative perspective that can be used to analyze the upper bound for the $(\varepsilon, \delta)$ values and the limitations of the Markovian policy, as discussed in the next section.

\subsection{Infinite Hierarchy Belief Representation Capabilities}

In the preceding section, our analysis begins with the information state. 
It aims to theoretically examine how the curse of dimensionality arises from the enriched observation space in POSG and its impact on the Markovian policy.
Bayesian RL frameworks represent another approach to addressing partially observed problems that are discussed in the literature~\citep{Zamir2008BayesianGG,Ross2007BayesAdaptiveP,Katt2018BayesianRL,Foerster2018BayesianAD}. 
These frameworks maintain a posterior distribution over environment models, where at each step, a model is sampled from the posterior, and its corresponding optimal policy is learned and executed.

In practice, it is reasonable for each agent in a POSG to compute a belief state that captures their uncertainty about the environment's state, following the logic of POMDPs.
In this context, such a belief state is called a "zeroth-order belief," as discussed in~\citet{moreno2021neural}. 
When viewed from a single agent's perspective, assuming that other agents' policies are fixed, the POSG can be regarded as a POMDP. 
The presence of other agents transforms these agents into part of the environment, causing the unknown state of the world to encompass not only the environment state but also the knowledge state of other agents.
Consequently, agent $i$ must create a belief about the other agent's beliefs, known as "first-order" beliefs. 
This recursive process can continue, allowing agents to form "second-order" beliefs about other agents' "first-order" beliefs, and so on. 
Ultimately, each agent maintains an "infinite hierarchy" of beliefs.

It is apparent that F2A2 must possess the ability to express the infinite hierarchy belief if it uses finite memory compression and LSTM models to generate small $(\varepsilon, \delta)$ values. 
However, representing an infinite hierarchy of beliefs is theoretically unfeasible. 
Fortunately, agents do not need to consider an infinite hierarchy of beliefs to exhibit robust decision-making and generalization abilities. 
This fact is exemplified by human beings, with the Keynes Beauty Contest~\citep{keynes1937general} being a classic illustration. 
In a simplified version of this experiment, players are asked to select a number between $0$ and $100$, with the player whose guess is closest to half of the average declared the winner.
According to the theory of Nash equilibrium (NE), selecting $0$ is the only rational choice for each player in the Keynes Beauty Contest. 
This reasoning is as follows: "If all players make random guesses, then the average of those guesses would be $50$ (level-$0$). 
Therefore, I should guess at most $1/2 * 50=25$ (level- $ 1 $). 
If other players think similarly, I should not guess more than $1/2 * 25=13$ (level-$2$)," and so on. 
In this manner, the level of beliefs can continue to develop until all players choose $0$, which is the only NE in this game.

However, experimental evidence reveals that most human players select numbers between $13$ and $25$~\citep{coricelli2009neural}, contradicting this theoretical result. 
This discrepancy is because not all human players exhibit perfect rationality; they are bounded by the levels of recursion they prefer to reason with (i.e., bounded rationality). 
Furthermore, in constrained scenarios, \citet{wen2021modelling} demonstrated that endowing agents with higher-level reasoning capabilities did not significantly enhance performance.

To this end, we can comprehend the limitations of Markovian policies based on the previous analysis by examining whether finite memory compression (FMC) and LSTM can represent a \ul{finite} hierarchy of beliefs. 
In this regard, \citet{moreno2021neural} introduces a scalable approach to approximate hierarchical belief structures using recursive deep generative models and leverages these belief models to obtain representations useful for complex tasks. 
For belief approximation at each level, they employ a Markovian policy based on finite memory compression and two independent RNNs. 
This implies that multi-layer RNNs may be capable of representing a finite hierarchy of beliefs. 
However, in F2A2, only a single-layer LSTM is used in the Markovian policy, which likely imposes significant limitations on representing multi-layer beliefs and leads to large $(\varepsilon, \delta)$ values. 
Therefore, in future work, we intend to conduct more comprehensive research on designing a more effective network structure to estimate multi-layer beliefs.
}}

\vskip 0.2in
\bibliography{main}

\end{document}